\documentclass{article}

\usepackage[final]{nips_2017}
\usepackage{amsmath}
\usepackage{amssymb}
\usepackage{amsthm}
\pdfoutput=1

\usepackage[utf8]{inputenc} 
\usepackage[T1]{fontenc}    
\usepackage{hyperref}       
\usepackage{url}            
\usepackage{booktabs}       
\usepackage{amsfonts}       
\usepackage{nicefrac}       
\usepackage{microtype}      
\usepackage{xcolor} 	
\usepackage{upgreek}  
\usepackage{thmtools}
\usepackage{thm-restate} 
\usepackage{graphicx}
\usepackage{float}

\usepackage{caption} 
\usepackage[export]{adjustbox} 

\renewcommand{\Re}{{\rm Re}}
\renewcommand{\Im}{{\rm Im }}
\newcommand{\C}{{\sf Col}}
\newcommand{\N}{{\sf Null}}
\renewcommand{\vec}[1]{\boldsymbol{\mathbf{#1}}}

\newcommand{\LQ}{{\sc{LQ}}}
\newcommand{\der}[2]{\frac{\partial {#2}}{\partial {#1}}}
\newcommand{\flatder}[2]{\frac{\partial}{\partial {#1}}  {#2} }
\newcommand{\transpose}[1]{\left( #1\right)^T }
\newcommand{\RNum}[1]{\uppercase\expandafter{\romannumeral #1\relax}}
\newcommand{\Null}{{\sf Null}}
\newcommand{\Col}{{\sf Col}}
\allowdisplaybreaks

\begingroup
    \makeatletter
    \@for\theoremstyle:=definition,remark,plain\do{%
        \expandafter\g@addto@macro\csname th@\theoremstyle\endcsname{%
            \addtolength\thm@preskip\parskip
            }%
        }
\endgroup

\newtheorem{theorem}{Theorem}[section]

\newtheorem{lemma}[theorem]{Lemma}
\newtheorem{proposition}{Proposition}[section]
\newtheorem{fact}{Fact}[section]
\theoremstyle{definition}
\newtheorem{definition}{Definition}[section]

\newtheorem{assumption}{Assumption}

\newtheorem{property}{Property}

\title{Gradient descent GAN optimization is locally stable}

\author{
  Vaishnavh Nagarajan \\
  Computer Science Department\\
  Carnegie-Mellon University\\
  Pittsburgh, PA 15213 \\
  \texttt{vaishnavh@cs.cmu.edu}\\
\And
  J. Zico Kolter\\
  Computer Science Department\\
  Carnegie-Mellon University\\
  Pittsburgh, PA 15213 \\
  \texttt{zkolter@cs.cmu.edu}
}

\begin{document}
\sloppy

\maketitle

\begin{abstract}
  Despite the growing prominence of generative adversarial
  networks (GANs), optimization in GANs is still a poorly understood topic.  In this paper, we analyze
  the ``gradient descent'' form of GAN optimization i.e., the natural
  setting where we simultaneously take small gradient steps in both generator
  and discriminator parameters.  We show that even though GAN optimization
  does \emph{not} 
  correspond to a convex-concave game (even for simple parameterizations), under
  proper conditions, equilibrium points of this  optimization 
  procedure are still \emph{locally asymptotically stable} for the traditional
  GAN formulation. On the other hand, we show that the recently proposed
  Wasserstein GAN   can have non-convergent limit cycles near equilibrium.
  Motivated by this stability analysis, we 
  propose an additional regularization term for gradient descent GAN updates,
  which \emph{is} able to guarantee local stability for both the WGAN and the
  traditional GAN, and also shows practical promise in speeding up
  convergence and addressing mode collapse. 
\end{abstract}

\section{Introduction}

Since their introduction a few years ago, Generative Adversarial Networks (GANs)
\citep{goodfellow2014generative}  have gained prominence as one of the most
widely used methods for training deep generative models.  GANs have been
successfully deployed for tasks  such as photo super-resolution, object
generation, video prediction, language modeling, vocal synthesis, and
semi-supervised learning, amongst many others
\citep{ledig2016photorealistic,wu2016learning,mathieu2015deep,nguyen2016plugnplay,denton2015deep,im2016generating}.

At the core of the GAN methodology is the idea of jointly training two networks: a
generator network, meant to produce samples from some distribution (that ideally
will mimic examples from the data distribution), and a discriminator network, which
attempts to differentiate between samples from the data distribution and the ones
produced by the generator.  This problem is typically written as a min-max
optimization problem of the following form:
\begin{align}
  \min_G \max_D \;\; \left(  \mathbb{E}_{x\sim p_{\mathrm{\rm data}}}[\log D(x)] +
  \mathbb{E}_{z\sim p_{\mathrm{latent}}}[\log (1 - D(G(z))] \right).
\end{align}
For the purposes of this paper, we will shortly consider a more general
form of the optimization problem, which also includes the recent Wasserstein GAN
(WGAN) \citep{arjovsky2017wasserstein} formulation. 

Despite their prominence, the actual task of optimizing GANs remains a
challenging problem, both from a theoretical and a practical standpoint.
Although the original GAN paper included some analysis on the convergence
properties of the approach \citep{goodfellow2014generative}, it
assumed that updates occurred in pure 
function space, allowed arbitrarily powerful generator and discriminator
networks, and modeled the resulting optimization objective as a
convex-concave game, therefore yielding well-defined global convergence
properties.  Furthermore, this analysis assumed that the discriminator network
is fully optimized between 
generator updates, an assumption that does not mirror the practice of GAN
optimization.  Indeed, in practice, there exist a
number of well-documented failure modes for GANs such as mode collapse or vanishing
gradient problems.

\subparagraph{Our contributions.} In this paper, we consider the ``gradient descent'' formulation of GAN
optimization, the setting where both the generator and the discriminator are
updated simultaneously via simple (stochastic) gradient updates; that is, there
are no inner and outer optimization loops, and neither the generator nor the
discriminator are assumed to be optimized to convergence.  Despite the fact
that, as we show, this does \emph{not} correspond to a convex-concave
optimization problem (even for simple linear generator and discriminator
representations), we show that:

\begin{center}
\parbox{0.95\linewidth}{
{ Under suitable conditions on the representational powers of the discriminator and the generator, the resulting GAN dynamical
system \emph{is} locally exponentially stable}.
}
\end{center}

That is, for some region around an equilibrium
point of the updates, the gradient updates will converge to this equilibrium point at an exponential rate.
Interestingly, 
 our conditions can be satisfied by the traditional GAN but \emph{not} by the WGAN, and we indeed
show that WGANs can have non-convergent limit cycles in the gradient descent case.

Our theoretical analysis also suggests a natural method for regularizing GAN
updates by adding an additional regularization term on the norm of the
discriminator gradient.  We show that the addition of this term leads to 
locally exponentially stable equilibria for all classes of GANs, including WGANs.  The
additional penalty is highly related to (but also 
notably different from) recent proposals for practical GAN optimization,
such as the unrolled GAN \citep{metz2016unrolled} and the improved Wasserstein GAN
training \citep{gulrajani2017improved}.  In practice, the approach is simple to
implement, and preliminary 
experiments show that it helps avert mode collapse and leads to
faster convergence.

\section{Background and related work}

\paragraph{GAN optimization and theory.}  Although the theoretical analysis of
GANs has been far outpaced by their 
practical application, there have been some notable results in recent years, in
addition to the aforementioned work in the original GAN paper.
For the most part, this work is entirely complementary to our own, and studies a
very different set of questions.
\citet{arjovsky2016towards} provide important insights into \emph{instability} that
arises when the supports of the generated distribution and the true distribution
are disjoint.  In contrast, in this paper we delve into an equally important
question of whether the updates are stable even \emph{when} the generator
is in fact very close to the true distribution (and we answer in the
affirmative).   \citet{arora2017generalization}, on the other 
hand, explore questions relating to the sample complexity and expressivity of
the GAN architecture and their relation to the existence of an equilibrium
point. However, it is still unknown as to whether, given that an equilibrium
exists, the GAN update procedure will converge locally.

From a more practical standpoint, there have been a number of papers that address
the topic of optimization in GANs.  Several methods have been proposed that
introduce new objectives or architectures for improving the (practical and
theoretical) stability of GAN optimization
\citep{arjovsky2017wasserstein,poole2016improved}.  A wide variety of
optimization heuristics and architectures have also been proposed 
to address challenges such as mode collapse
\citep{salimans2016improved,metz2016unrolled,che2016mode,radford2015unsupervised}.
Our own proposed regularization term falls under this same category, and
hopefully provides some context for understanding some of these methods.
Specifically, our regularization term (motivated by stability analysis) captures
a degree of ``foresight'' of the generator in the optimization procedure, similar to
the unrolled GANs procedure \citep{metz2016unrolled}. Indeed, we show that
our gradient penalty is closely related to $1$-unrolled GANs, but also provides
more flexibility in leveraging this foresight.   Finally, gradient-based
regularization has been explored for GANs, with one of the most recent works
being that of \citet{gulrajani2017improved}, though their penalty is on the 
discriminator rather than the generator as in our case. 

Finally, 
there are several works that have simultaneously addressed
similar issues as this paper.  Of particular similarity to the methodology we
propose here are 
the works by \citet{roth2017stabilizing} and \citet{mescheder2017numerics}.  The
first of these two present a stabilizing regularizer that is based on a gradient norm, where the gradient is calculated with respect 
to the datapoints.
Our regularizer on the other hand is based on the norm of a gradient calculated with respect to the parameters. Our approach has some strong similarities with that of the second work noted above;
however, the
authors there do not establish or disprove stability, and instead note the
presence of zero eigenvalues (which we will treat in some depth) as a motivation
for their alternative optimization method.  Thus, we feel the works as a
whole are quite complementary, and signify the growing interest in GAN
optimization issues.

\paragraph{Stochastic approximation algorithms and analysis of nonlinear
  systems.}
The technical tools we use to analyze the GAN optimization dynamics in this
paper come from the fields of stochastic approximation algorithm and the
analysis of nonlinear differential equations -- notably the ``ODE method'' for
analyzing convergence properties of dynamical systems \citep{borkar2000ode, kushner2003stochastic}.
Consider a general stochastic process
driven by the updates 
  $\vec{\theta}_{t+1} =\vec{\theta}_{t} + \alpha_t (h(\vec{\theta}_t) + \epsilon_t)$
for vector $\vec{\theta}_t \in \mathbb{R}^n$, step size $\alpha_t > 0$, function $h :
\mathbb{R}^n \rightarrow \mathbb{R}^n $ and  a martingale difference
sequence $\epsilon_t$.\footnote{Stochastic
gradient descent  on an objective $f(\theta)$ can be expressed in this
framework as $h(\vec{\theta}) = \nabla_{\vec{\theta}}  f(\vec{\theta})$.} Under fairly general
conditions, namely: 1) bounded second moments of $\epsilon_t$, 2) Lipschitz
continuity of $h$, and 3) summable but not square-summable step sizes, the
stochastic approximation algorithm converges to an equilibrium point of the
(deterministic) ordinary differential equation $\dot{\vec{\theta}}(t) =
h(\vec{\theta}(t))$. 

Thus, to understand stability of the stochastic approximation algorithm, it
suffices to understand the stability and convergence of the deterministic
differential equation.  Though such analysis is typically used to show global
asymptotic convergence of the stochastic approximation algorithm to an
equilibrium point (assuming the related ODE also is globally
asymptotically stable), it can also be used to analyze the \emph{local}
asymptotic stability properties of the stochastic approximation algorithm around 
equilibrium points.\footnote{Note that the local analysis does \emph{not} show that
  the stochastic approximation algorithm will necessarily converge to an
  equilibrium point, but still provides a valuable characterization of how the
  algorithm will behave around these points.}  This is the technique we follow
throughout this entire work, though for brevity we will focus entirely on the
analysis of the continuous time ordinary differential equation, and appeal to
these standard results to imply similar properties regarding the discrete
updates.

Given the above consideration, our focus will be on proving
stability of the dynamical system around equilbrium points, i.e. points
$\vec{\theta}^\star$ for which $h(\vec{\theta}^\star) = 0$.\footnote{Note that
  this is a slightly different usage of the term equilibrium as typically used
  in the GAN literature, where it refers to a Nash equilibrium of the min max
  optimization problem.  These two definitions (assuming we mean just a local
  Nash equilibrium) are equivalent for the ODE corresponding to the min-max
  game, but we use the dynamical systems meaning throughout this paper, that is,
any point where the gradient update is zero}.  Specifically, we appeal to the
well known \emph{linearization theorem} \citep[Sec 4.3]{khalil1996noninear}, which states that if the
Jacobian of the dynamical system 
  $\vec{J}  = \left . {\partial h(\theta)}/{\partial \theta} \right |_{\theta = \theta^\star}$
evaluated at an equilibrium point is {Hurwitz} (has all strictly negative eigenvalues, $\Re(\lambda_i(\vec{J}
)) < 0, \; \forall i=1,\dots,n$), then the ODE will converge to $\theta^\star$
for some non-empty region around $\theta^\star$, at an exponential rate.  This means that the system is 
locally asymptotically stable, or more precisely, locally exponentially stable (see Definition~\ref{def:stability} in Appendix~\ref{app:lyapunov}).

Thus, an important contribution of this paper is a proof of this seemingly simple fact:
under some conditions, \emph{the Jacobian of the
dynamical system given by the GAN update is a Hurwitz matrix at an
equilibrium} (or, if there are zero-eigenvalues, if they correspond to a
subspace of equilibria, the system is still asymptotically stable).  While this is a
trivial property to show for
convex-concave games,  the fact that the GAN is \emph{not} convex-concave  leads
to a 
substantially more challenging analysis.

In addition to this, we provide an analysis that is based on Lyapunov's stability theorem (described in Appendix~\ref{app:lyapunov}).  The crux of the idea is that to prove convergence it is sufficient to identify a non-negative ``energy'' function for the linearized system which always decreases with time (specifically, the energy function will be a distance from the equilibrium, or from the subspace of equilibria).  Most importantly, this analysis provides insights into the dynamics that lead to GAN convergence.

\section{GAN optimization dynamics}

This section comprises the main results of this paper, showing that under proper
conditions the gradient descent updates for GANs (that is, updating both the
generator and discriminator locally and simultaneously), is locally exponentially stable
around ``good'' equilibrium points (where ``good'' will be defined shortly).
This requires that the GAN loss be strictly concave, which is not the case
for WGANs, and we indeed show that the updates for WGANs can cycle indefinitely.
This leads us to propose a simple regularization term that \emph{is} able to
guarantee exponential stability for \emph{any} concave GAN loss, including the WGAN,
rather than requiring strict concavity.

\subsection{The generalized GAN setting}
For the remainder of the paper, we consider a slightly more general formulation
of the GAN optimization problem than the one presented earlier, given by the following min/max problem:
\begin{equation}
\label{eq:generic_gan}
  \min_G \max_D \;\; V(G,D) = \left(  \mathbb{E}_{x\sim p_{\mathrm{\rm data}}}[f(D(x))] +
  \mathbb{E}_{z\sim p_{\mathrm{latent}}}[f(-D(G(z)))] \right)
\end{equation}
where $G: \mathcal{Z} \rightarrow \mathcal{X}$ is the generator network, which maps
from the latent space $\mathcal{Z}$ to the input space $\mathcal{X}$; $D :
\mathcal{X} \rightarrow \mathbb{R}$ is the discriminator network, which maps
from the input space to a classification of the example as real or synthetic;
and $f : \mathbb{R}\rightarrow \mathbb{R}$ is a concave function.  We can
recover the traditional GAN formulation \citep{goodfellow2014generative} by
taking $f$ to be the (negated) 
logistic loss $f(x) = -\log (1+\exp(-x))$; note that this convention slightly differs from the
standard formulation in that in this case the discriminator outputs the real-valued
``logits'' and the loss function would implicitly scale this to a probability. 
We can recover the Wasserstein GAN by simply taking $f(x) = x$. 

Assuming the generator and discriminator networks to be parameterized by some
set of parameters, $\vec{\theta}_D$ and $\vec{\theta}_G$ respectively, we analyze the simple
stochastic gradient descent approach to solving this optimization problem. That is,
we take simultaneous gradient steps in both $\vec{\theta_D}$ and
$\vec{\theta_G}$, which in our ``ODE method'' analysis leads to the following
differential equation:
\begin{align}
\label{eq:undamped_updates}
    \vec{\dot{\theta}_D}  = \nabla_{ \vec{\theta_D}} V(\vec{\theta_G}, \vec{\theta_D}), \;\;  \vec{\dot{\theta}_G}  := \nabla_{ \vec{\theta_G}} V(\vec{\theta_G}, \vec{\theta_D}).
\end{align}

\paragraph{A note on alternative updates.}
Rather than updating both the generator and discriminator according to the
min-max problem above,  \citet{goodfellow2014generative} also proposed a
modified update for just the generator that minimizes a different objective, 
  $V'(G,D) =  -\mathbb{E}_{z\sim p_{\mathrm{latent}}}[f(D(G(z)))]$
(the negative sign is pulled out from inside $f$).  In fact, all the
analyses we consider in this paper apply equally to this case (or any convex
combination of both updates), as the ODE of the update equations have the same
Jacobians at equilibrium.

\subsection{Why is proving stability hard for GANs?}

Before presenting our main results, we first highlight why understanding the local 
stability of GANs is non-trivial, even when the generator and discriminator have
simple forms.  As stated above, GAN optimization consists of a min-max game, and
gradient descent algorithms will converge if the game is convex-concave -- the
objective must be convex in the term being minimized and concave in the term being
maximized.  Indeed, this was a crucial assumption in the convergence proof in the
original GAN paper.  However, for virtually any
parameterization of the real GAN generator and discriminator, even if both
representations are \emph{linear}, the GAN objective will not be a convex-concave
game:

\begin{proposition}
The GAN objective in Equation~\ref{eq:generic_gan} can be a concave-concave objective i.e., concave
with respect to both the discriminator and generator parameters, 
for
a large part of the discriminator space, including regions arbitrarily close to
the equilibrium.
\end{proposition}

To see why, consider a simple GAN over 1 dimensional data
and latent space with linear generator and discriminator, i.e. $D(x) = \theta_D
x + \theta_D'$ and $G(z) = \theta_Gz + \theta_G'$.  Then the GAN objective is:
\begin{align*}
  V(G,D) = \mathbb{E}_{x\sim p_{\mathrm{\rm data}}}[f(\theta_D x + \theta_D')] +
  \mathbb{E}_{z\sim p_{\mathrm{latent}}}[f(-\theta_D (\theta_G z + \theta_G') - \theta_D')].
\end{align*}
Because $f$ is concave, by inspection we can see that $V$ is concave in
$\theta_D$ and $\theta_D'$; but it is \emph{also} concave (not convex) in
$\theta_G$ and $\theta_G'$, for the same reason.  Thus, the optimization
involves \emph{concave} minimization, which in general is a difficult problem. 
To prove that this is not a peculiarity of the above linear discriminator
system, in Appendix~\ref{app:convex-concave}, we show similar
observations for a more general parametrization, and also for the case where 
 $f''(x) = 0$ (which
 happens in the case of WGANs).

Thus, a major question remains as to whether or not GAN optimization is stable
at all (most concave maximization is not).
Indeed, there are several well-known properties of GAN optimization that may
make it seem as though gradient descent optimization may \emph{not} work in
theory.  For instance, it is well-known that at the optimal location $p_g =
p_{\mathrm{\rm data}}$, the optimal discriminator will output zero on all examples,
which in turn means that \emph{any} generator distribution will be optimal for
this generator.  This would seem to imply that the system can not be stable
around such an equilibrium.

However, as we will show, gradient descent GAN optimization 
\emph{is} locally asymptotically stable, even for natural parameterizations of
generator-discriminator pairs (which still make up concave-concave optimization
problems). 
Furthermore, at equilibrium, although the
zero-discriminator property means that the generator is not stable
``independently'', the joint dynamical 
system of generator and discriminator \emph{is} locally asymptotically stable
around certain equilibrium points.

\subsection{Local stability of general GAN systems}
\label{sec:general-stability}
This section contains our first technical result, establishing that GANs are locally stable under proper local conditions.  Although the
proofs are 
deferred to the appendix, the elements that we do emphasize here are the
conditions that we identified for local stability to hold.  Indeed, because the
proof rests on these 
conditions (some of which are fairly strong), we want to highlight them as much
as possible, as they themselves also convey valuable intuition as to what is
required for GAN convergence.

To formalize our conditions, we denote the support of
a distribution with probability density function (p.d.f) $p$ by ${\rm supp}(p)$ and the p.d.f of the generator $\vec{\theta_G}$ by $p_{\vec{\theta_G}}$. Let $B_{\epsilon}(\cdot)$ denote the
Euclidean $L_2$-ball of radius of $\epsilon$. Let $\lambda_{\max}(\cdot)$ and $\lambda_{\min}^{(+)}(\cdot)$ denote the largest and the
smallest non-zero eigenvalues of a non-zero positive semidefinite matrix. 
Let $\Col(\cdot)$ and $\Null(\cdot)$  denote the column space and null space of a matrix respectively.  Finally, we define two key matrices that will be integral to our analyses:
\begin{align*}
\vec{K}_{DD} \triangleq & \left. \mathbb{E}_{p_{\rm data}} [\nabla_{\vec{\theta_D}}
    D_{\vec{\theta_D}}(x) \nabla_{\vec{\theta_D}}^T D_{\vec{\theta_D}}(x)] \right\vert_{\vec{\theta^\star_D}} , \; \;
    \vec{K}_{DG} \triangleq & \left. \int_{\mathcal{X}} \nabla_{\vec{\theta_D}} D_{\vec{\theta_D}}(x)  \nabla^T_{\vec{\theta_G}} p_{\theta_G}(x)  dx\right\vert_{( \vec{\theta^\star_D}, \vec{\theta^\star_G})} 
\end{align*}

Here, the matrices are evaluated at an equilibrium point $(\vec{\theta_D^\star}, \vec{\theta_G^\star})$ which we will characterize shortly.  The significance of these terms is that, as we will see, $\vec{K}_{DD}$ is proportional to
the Hessian of the GAN objective with respect to the discriminator parameters at equilibrium, and $\vec{K}_{DG}$ is proportional to
the off-diagonal term in this Hessian, corresponding to the discriminator and
generator parameters. 
These matrices also occur in similar positions in the Jacobian of the system at equilibrium.

We now discuss conditions under which we can guarantee exponential stability. All our conditions are imposed on both $(\vec{\theta_D^\star}, \vec{\theta_G^\star})$ and all equilibria in a small neighborhood around it, though we do not state this explicitly in every assumption.  First, we define the ``good'' equilibria we care about as those that
 correspond to a generator which matches the true
distribution and a discriminator that is identically zero on the support of this
distribution.  As described next, implicitly, this 
also assumes that the discriminator and generator representations are powerful
enough to guarantee  
that there are no ``bad'' equilibria in a local neighborhood of this
equilibrium.

\begin{assumption}
\label{as:global-gen} 
$p_{\vec{\theta^{\star}_G}} = p_{\rm data}$ and 
$D_{\vec{\theta^\star_D}}(x) = 0$, $\forall \;
x \in {\rm supp}(p_{\rm data})$.  
\end{assumption}

The assumption that the generator matches the true distribution is a rather
strong assumption, as it limits us to the 
``realizable'' case, where the generator is capable of creating the underlying
data distribution. Furthermore, this means	the discriminator is (locally)
powerful enough that for any other generator distribution it is not at
equilibrium (i.e., discriminator updates are non-zero).  Since we do not
typically expect this to be the case, we also 
provide an alternative non-realizable assumption below that is also sufficient for our
results i.e., the system is still stable.  
 In both 
 the realizable and non-realizable cases the
requirement of an all-zero discriminator remains. This implicitly requires even 
the generator representation be (locally) rich enough so that when the discriminator is not identically zero, the generator is not at equilibrium (i.e., generator updates are non-zero).
Finally, note that these conditions do not disallow bad equilibria outside of this neighborhood, which may potentially even be unstable.  

\textbf{Assumption} \textbf{~\ref{as:global-gen}.} (\textbf{Non-realizable})
The  discriminator is  \emph{linear} in its
parameters $\vec{\theta_D}$ and furthermore, for any equilibrium point
 $(\vec{\theta^\star_D}, \vec{\theta^\star_G})$,
$D_{\vec{\theta^\star_D}}(x) = 0$, 
$\forall \;
x \in {\rm supp}(p_{\rm data}) \cup {\rm supp}(p_{\vec{\theta^\star_G}})$.

This alternative assumption is largely a weakening of Assumption~\ref{as:global-gen}, as the condition
on the discriminator remains, but there is no requirement that the generator give
rise to the true distribution.  However, the requirement that the discriminator
be linear in the parameters (\emph{not} in its input), is an additional 
restriction that seems unavoidable in this case for technical reasons.  Further,
note that the fact that $D_{\vec{\theta_D^\star}}(x) = 0$ and that the generator/discriminator are
both at equilibrium, still means that although it may be that
$p_{\vec{\theta^\star_G}} \neq p_{\mathrm{data}}$, these distributions are
(locally) indistinguishable as far as the discriminator is concerned.  Indeed, this
is a nice characterization of ``good'' equilibria, that the discriminator cannot
differentiate between the real and generated samples.

Our goal next is to identify strong curvature conditions that can be imposed on the objective $V$ (or a function related to the objective), though only locally { at equilibrium}. First, we will require that the objective is strongly concave in the discriminator parameter space at equilibrium (note that it is concave by default).  However, on the other hand, we cannot ask the objective to be strongly convex in the generator parameter space as we saw that the objective is not convex-concave even in the nicest scenario, even arbitrarily close to equilbrium. Instead, we  identify another convex function, namely {\em the magnitude of the update on the equilibrium discriminator} i.e., $ \| \left. \nabla_{\vec{\theta_D}}  V(\vec{\theta}_D, \vec{\theta}_G)  \right\vert_{\vec{\theta}_D=\vec{\theta}_D^\star}\|^2$, and require that to be strongly convex in the generator space at equilibrium.  Since these strong curvature assumptions will allow only systems with a locally unique equilibrium, we will state them in a relaxed form that accommodates a local subspace of equilibria. Furthermore, we will state these assumptions in two parts, first as a condition on $f$, second as a condition on the parameter space. 
 
First, the condition on $f$ is straightforward, 
 making it necessary that 
 the loss $f$ be
 concave at $0$; as we will show, when this condition is not met, 
  there need
 not be local asymptotic convergence.
\begin{assumption}
\label{as:loss}
  The function $f$ satisfies $f''(0) < 0$, and $f'(0) \neq 0$
\end{assumption}

Next, to state conditions on the parameter space while also allowing systems with multiple equilibria locally,
we first define the following property for a function, say $g$, at a specific point in its domain: along any direction either the second derivative of $g$ must be non-zero or {\em all} derivatives must be zero. For example,  at the origin,
$g(x,y) = x^2 + x^2 y^2$ is flat along $y$, and along any other direction at an angle $\alpha \neq 0$ with the $y$ axis, the
second derivative is $2 \sin^2 \alpha$.  For the GAN system, we will require this property, formalized in Property~\ref{prop:convex}, for  two convex functions whose Hessians are proportional to
  $\vec{K}_{DD}$ and $\vec{K}_{DG}^T \vec{K}_{DG}$.  We provide 
 more intuition for these functions below.

\begin{property}
\label{prop:convex} 
$g: \Theta \to \mathbb{R}$ satisfies Property~\ref{prop:convex} at  $\vec{\theta^\star} \in \Theta$  if for  any $\vec{\theta} \in \Null(\left. \nabla^2_{\vec{\theta}} g(\vec{\theta}) \right\vert_{\vec{\theta}^\star}  )$, the function is locally constant along $\vec{\theta}$ at $\vec{\theta^\star}$ i.e., $\exists \epsilon > 0$ such that for all $\epsilon' \in (-\epsilon, \epsilon)$, $g(\vec{\theta^\star}) = g(\vec{\theta^\star} + \epsilon' \vec{\theta})$.
\end{property}

\begin{assumption}
\label{as:convexity}
At an equilibrium $(\vec{\theta^\star_D}, \vec{\theta^\star_G})$, the functions
$\mathbb{E}_{p_{\rm data}}[D^2_{\vec{\theta_D}} (x)]$  and $\left. \left\| \mathbb{E}_{p_{\rm data}}[ \nabla_{\vec{\theta_D}} D_{\vec{\theta_D}} (x)  ]   -  \mathbb{E}_{p_{\vec{\theta_G}}}[ \nabla_{\vec{\theta_D}} D_{\vec{\theta_D}} (x)  ]   \right\|^2 \right\vert_{\vec{\theta_D} = \vec{\theta_D^\star}}$ must satisfy Property~\ref{prop:convex} in the discriminator and generator space respectively.
\end{assumption}

Here is an intuitive explanation of what these two non-negative functions represent and how they relate to the objective. 
The first function is a function of $\vec{\theta_D}$ which measures  how far
 $\vec{\theta_D}$ is from an all-zero state, and the second is a function of $\vec{\theta_G}$ which measures
 how far 
$\vec{\theta_G}$ is from the true distribution; at equilibrium these
functions are zero. We will see later that  given $f''(0) < 0$, the curvature of the first function at
$\vec{\theta^\star_D}$ is representative of the curvature of $V(\vec{\theta_D},
\vec{\theta_G^\star})$ in the discriminator space; similarly, given $f'(0) \neq 0$ the curvature of the second function at $\vec{\theta_G^\star}$
is representative of the curvature of  {\em the magnitude of the discriminator
  update on $\vec{\theta}_D^\star$} in the generator space.  
  The intuition behind why this particular relation holds is that, when $\vec{\theta_G}$ moves away from
the true distribution, while the second function in Assumption~\ref{as:convexity} increases, $\vec{\theta_D^\star}$ also becomes more suboptimal for that generator; as a result,  the magnitude of
update on $\vec{\theta_D^\star}$ increases too.  Note that we show in Lemma~\ref{lem:eqspace}, that the Hessian of the two functions in Assumption~\ref{as:convexity}  in the discriminator and the generator space respectively, are proportional to $\vec{K}_{DD}$ and $\vec{K}_{DG}^T \vec{K}_{DG}$. 
%

The above relations involving the two functions and the GAN objective, together with Assumption~\ref{as:convexity}, basically allow us to consider systems with reasonable strong curvature properties, while also allowing many equilibria in a local neighborhood in a specific sense. In particular, if the curvature of the first function is flat along a direction $\vec{u}$
 (which also means that $\vec{K}_{DD}\vec{u} = 0$)  we can
 perturb $\vec{\theta_D^\star}$ slightly along $\vec{u}$ and still have an
 `equilibrium discriminator' as defined in Assumption~\ref{as:global-gen} i.e.,
 $\forall x \in {\rm supp}(p_{\vec{\theta^\star_G}}) $, $D_{\vec{\theta_D}}(x) =
 0$. Similarly, for any direction $\vec{v}$ along which
the curvature of the second function is flat (i.e., $\vec{K}_{DG} \vec{v} = 0$), we can perturb
$\vec{\theta_G^\star}$ slightly along that direction such that $\vec{\theta_G}$
remains an `equilibrium generator' as defined in Assumption~\ref{as:global-gen}
i.e., ${p_{\theta_G}} = {p_{\rm data}}$. We prove this formally in  Lemma~\ref{lem:eqspace}. Perturbations along any other directions do not yield equilibria because then, either $\vec{\theta}_D$ is no longer in an all-zero state or $\vec{\theta}_G$ does not match the true distribution. Thus, we consider a setup where the rank deficiencies of  $\vec{K}_{DD}$, $\vec{K}_{DG}^T\vec{K}_{DG}$ if any, correspond to equivalent equilibria 
 (which typically exist for neural networks, though in practice they may not correspond to `linear' perturbations as modeled here).



Our final assumption is on the supports of the true and generated distributions: we require that all the generators in a sufficiently
small neighborhood of the equilibrium have distributions with the same support
as the true distribution. Following this, we briefly discuss a relaxation of this assumption.
\begin{assumption}
\label{as:same-support}
$\exists \epsilon_G > 0$ such that $\forall \vec{\theta_G} \in B_{\epsilon_G}(\vec{\theta^\star_G})$, ${\rm
  supp}(p_{\vec{\theta_G}}) = {\rm supp}(p_{\rm data}) $.
\end{assumption}

This 
may typically hold if the support covers the whole space $\mathcal{X}$; but when the true distribution has support in some smaller disjoint parts of the space $\mathcal{X}$, nearby generators may correspond to slightly displaced versions of this distribution with a different support.  For the latter scenario, we show in Appendix~\ref{app:realizable-relaxed} that local exponential stability holds under a  certain smoothness condition on the discriminator. Specifically, we  require that $D_{\vec{\theta}_D^\star}(\cdot)$ be zero not only on the support of $\vec{\theta}_G^\star$ but also on the support of small perturbations of $\vec{\theta}_G^\star$ as otherwise the generator will not be at equilibrium. (Additionally, we also require this property from the discriminators that lie within a small perturbation of $\vec{\theta_D^\star}$ in the null space of $\vec{K}_{DD}$ so that they correspond to equilibrium discriminators.) We note that while this relaxed assumption accounts for a larger class of examples, it is still strong in that it also restricts us from certain simple systems. Due to space constraints, we state and discuss the implications of this assumption in greater detail in Appendix~\ref{app:realizable-relaxed}.




 We now state our main result.

\begin{restatable}{theorem}{generalstability}
\label{thm:general-stability}
The dynamical system defined by the GAN objective in
Equation~\ref{eq:generic_gan} and the updates in
Equation~\ref{eq:undamped_updates} is locally exponentially stable with respect to 
an equilibrium point  
$(\vec{\theta^\star_D},\vec{\theta^\star_G})$ when the Assumptions~\ref{as:global-gen},~\ref{as:loss},~\ref{as:convexity},~\ref{as:same-support} hold for $(\vec{\theta^\star_D},\vec{\theta^\star_G})$ and other equilibria in a small neighborhood around it. Furthermore,  the rate of convergence is governed only by the eigenvalues 
 $\lambda$ of the Jacobian $\vec{J}$ of the system at equilibrium with a strict negative real part upper bounded as:  
\begin{itemize}
\itemsep-0.2em 
\item  If $\Im(\lambda) = 0$, then $\Re(\lambda)\leq 
 \frac{2 f''(0) f'^2(0) \lambda_{\min}^{(+)}(\vec{K}_{DD})  \lambda_{\min}^{(+)}(\vec{K}_{DG}^T \vec{K}_{DG})}
 {4f''^2(0)\lambda_{\min}^{(+)}(\vec{K}_{DD}) \lambda_{\max}(\vec{K_{DD}}) +  f'(0)^2\lambda_{\min}^{(+)}(\vec{K}_{DG}^T\vec{K}_{DG})} $
\item If $\Im(\lambda) \neq 0$, then $\Re(\lambda) \leq f''(0)
  \lambda_{\min}^{(+)}(\vec{K}_{DD}) $
\end{itemize}
\end{restatable}

  The vast majority of our proofs are deferred to the appendix, but we briefly
  describe the intuition here. It is straightforward to show that the Jacobian $\vec{J}$ of the
  system at equilibrium can be written as: 
\[
\vec{J}= 
\begin{bmatrix}
\vec{J}_{DD}& \vec{J}_{DG} \\
-\vec{J}_{DG}^T & \vec{J}_{GG} \\
\end{bmatrix} = 
\begin{bmatrix}
2f''(0) \vec{K}_{DD} &f'(0)
\vec{K}_{DG} \\ 
-f'(0) \vec{K}_{DG} ^T & 0 \\
\end{bmatrix} 
\]
Recall that we wish to show this is Hurwitz.  First note that $\vec{J}_{DD}$ (the
Hessian of the objective with respect to the discriminator) is negative
semi-definite if and only if $f''(0) < 0$.
Next, a crucial observation is that $\vec{J}_{GG}=0$ i.e, the Hessian term w.r.t. the generator vanishes because for the all-zero discriminator, 
all generators result in the same objective value. Fortunately, 
this means {\em at equilibrium} we do not have non-convexity in $\vec{\theta}_G$
precluding local stability. Then, we make use of the crucial
Lemma~\ref{lem:undamped-bound} we prove in the appendix, showing that any matrix
of  
the form $\begin{bmatrix} -\vec{Q} & \vec{P}; & -\vec{P}^T & 0\end{bmatrix}$ is Hurwitz provided
that $-\vec{Q}$ is strictly negative definite and $\vec{P}$ has full column rank.

 However, this property holds only when $\vec{K}_{DD}$ is positive definite and $\vec{K}_{DG}$ is full column rank. Now, if  $\vec{K}_{DD}$ or $\vec{K}_{DG}$ do not have this property,
recall that the rank deficiency is due to a subspace of equilibria around $(\vec{\theta^\star_D},\vec{\theta^\star_G})$. Consequently, we can analyze the stability of the system projected to an subspace orthogonal to these equilibria (Theorem~\ref{thm:multiple-equilibria}).  Additionally, we also prove stability using Lyapunov's stability (Theorem~\ref{thm:lyapunov}) by showing that the squared $L_2$ distance to the subspace of equilibria always either decreases or only instantaneously remains constant.

\subparagraph{Additional results.} In order to illustrate our assumptions in Theorem~\ref{thm:general-stability}, in Appendix~\ref{app:lqgan} we consider a simple GAN that learns a multi-dimensional Gaussian using a quadratic discriminator and a linear generator.  In a similar set up, in Appendix~\ref{app:wgan-unstable}, we consider the case where $f(x) = x$ i.e., the Wasserstein GAN
and so $f''(x) = 0$, and we show that  the system can perennially cycle around an equilibrium point without converging.  A simple
two-dimensional example is visualized in Section \ref{sec:results}. Thus,  {\em gradient descent WGAN optimization is not necessarily asymptotically stable.}


\subsection{Stabilizing optimization via gradient-based regularization}
Motivated by the considerations above, in this section we propose a
regularization penalty for the generator update, which uses a term based upon the
gradient of the discriminator.  Crucially, the regularization term does
\emph{not} change the parameter values at the equilibrium point, and at the same time
enhances the local stability of the optimization procedure, both in theory and
practice. 
Although these update equations do require that we differentiate with respect to
a function of another gradient term, such ``double backprop'' terms (see e.g.,
\cite{drucker1992improving}) are easily computed by modern automatic
differentiation tools.   Specifically, we propose the regularized update
\begin{equation}
  \begin{split}
    \vec{\theta_G} & := \vec{\theta_G} - \alpha \nabla_{\vec{\theta_G}}  \left({V}(D_{\vec{\theta_D}}, G_{\vec{\theta_G}}) + \eta \|\nabla_{\vec{\theta_D}}{V}(D_{\vec{\theta_D}}, G_{\vec{\theta_G}})\|^2 \right) \label{eq:damped_updates}
  \end{split}
\end{equation}

\subparagraph{Local Stability}  The intuition of this regularizer is perhaps
most easily understood by considering how it changes the Jacobian at equilibrium
(though there are other means of motivating the update as well, discussed
further in  Appendix~\ref{app:intuition}).  In the Jacobian of the new
update, although there are now non-antisymmetric diagonal blocks, the block diagonal
terms are now negative definite: 

\begin{align*}
\begin{bmatrix}
\vec{J}_{DD} & \vec{J}_{DG} \\
-\vec{J}_{DG}^T(\vec{I} +2 \eta \vec{J}_{DD}) &  - 2\eta \vec{J}_{DG}^T \vec{J}_{DG}
\end{bmatrix}
\end{align*}

As we show below in Theorem~\ref{thm:regularized} (proved in Appendix~\ref{app:damped-updates}), as long
as we choose $\eta$ small enough so that $I + 2 \eta \vec{J}_{DD} \succeq 0$, this
guarantees the updates are locally asymptotically stable for any concave $f$. 
In addition to stability properties, this regularization term also addresses a well known
failure state in GANs called {\em mode collapse}, by lending more ``foresight'' to the generator. 
The way our updates provide this foresight is very similar to the unrolled updates proposed in \cite{metz2016unrolled}, although, our
regularization is much simpler and provides more flexibility to leverage the foresight. In practice, we see that our method can be as powerful as the more complex and slower 10-unrolled GANs. We discuss this and other intuitive ways of motivating our regularizer in Appendix~\ref{app:damped-updates}.
 \begin{restatable}{theorem}{regularized}
\label{thm:regularized}
The dynamical system defined by the GAN objective in Equation~\ref{eq:generic_gan} and the updates in 
 Equation~\ref{eq:damped_updates}, 
is locally exponentially stable at the equilibrium, under the same conditions as in Theorem~\ref{thm:general-stability}, if $\eta <
\frac{1}{2\lambda_{\max}(-\vec{J}_{DD})} $. 
Further, under appropriate conditions similar to these, the WGAN system is locally exponentially stable at the equilibrium for any $\eta$. The rate of convergence for the WGAN is governed only by the eigenvalues $\lambda$ of the Jacobian at equilibrium with a strict negative real part upper bounded as:
\begin{itemize}
\item If $\Im(\lambda) = 0$, then $\Re(\lambda) \leq  - \frac{2 f'^2(0) \eta \lambda_{\min}^{(+)}(\vec{K}_{DG}^T \vec{K}_{DG})}
{4  f'^2(0)\eta^2 \lambda_{\max}(\vec{K}_{DG}^T \vec{K}_{DG}) +  1} $
\item If $\Im(\lambda) \neq 0$, then $\Re(\lambda) \leq -  \eta f'^2(0) {\lambda_{\min}^{(+)}(\vec{K}_{DG}^T \vec{K}_{DG})} $
\end{itemize}
\end{restatable}

\section{Experimental results}
\label{sec:results}
We very briefly present experimental results that
demonstrate that our regularization term also has substantial
practical promise.\footnote{We provide an implementation of this technique at \url{https://github.com/locuslab/gradient_regularized_gan}}
 In Figure~\ref{fig:toy}, we compare our gradient regularization to
$10$-unrolled GANs on the same architecture and dataset (a mixture of eight
gaussians) as in \citet{metz2016unrolled}.
Our system quickly spreads out all the points
instead of first exploring only a few modes and then redistributing its mass
over all the modes gradually. Note that the conventional GAN updates are known to enter mode collapse
for this setup. We see similar results (see
Figure~\ref{fig:mnist} here, and 
Figure~\ref{fig:mnist-full} in the Appendix for a more detailed figure)
 in the case of a stacked MNIST dataset using a DCGAN
\citep{radford2015unsupervised} i.e., three random digits from MNIST are stacked together so as to
create a distribution over 1000 modes.  
Finally,  Figure~\ref{fig:streamline}, presents streamline plots for 
a  2D system where both the true and the latent distribution is uniform over
$[-1,1]$ and the discriminator is $D(x) = w_2 x^2$ while the generator is $G(z)
= az$. Observe that while the WGAN system goes in orbits
as expected, the original GAN system converges. With our
updates, both these systems converge quickly to the true equilibrium.


\begin{figure}[!htb]
    \centering
    \begin{minipage}{.2\textwidth}
        \centering
        \adjincludegraphics[width=1\textwidth,trim={0 {0.1\height} 0 {0.1\height}},clip,valign=t]{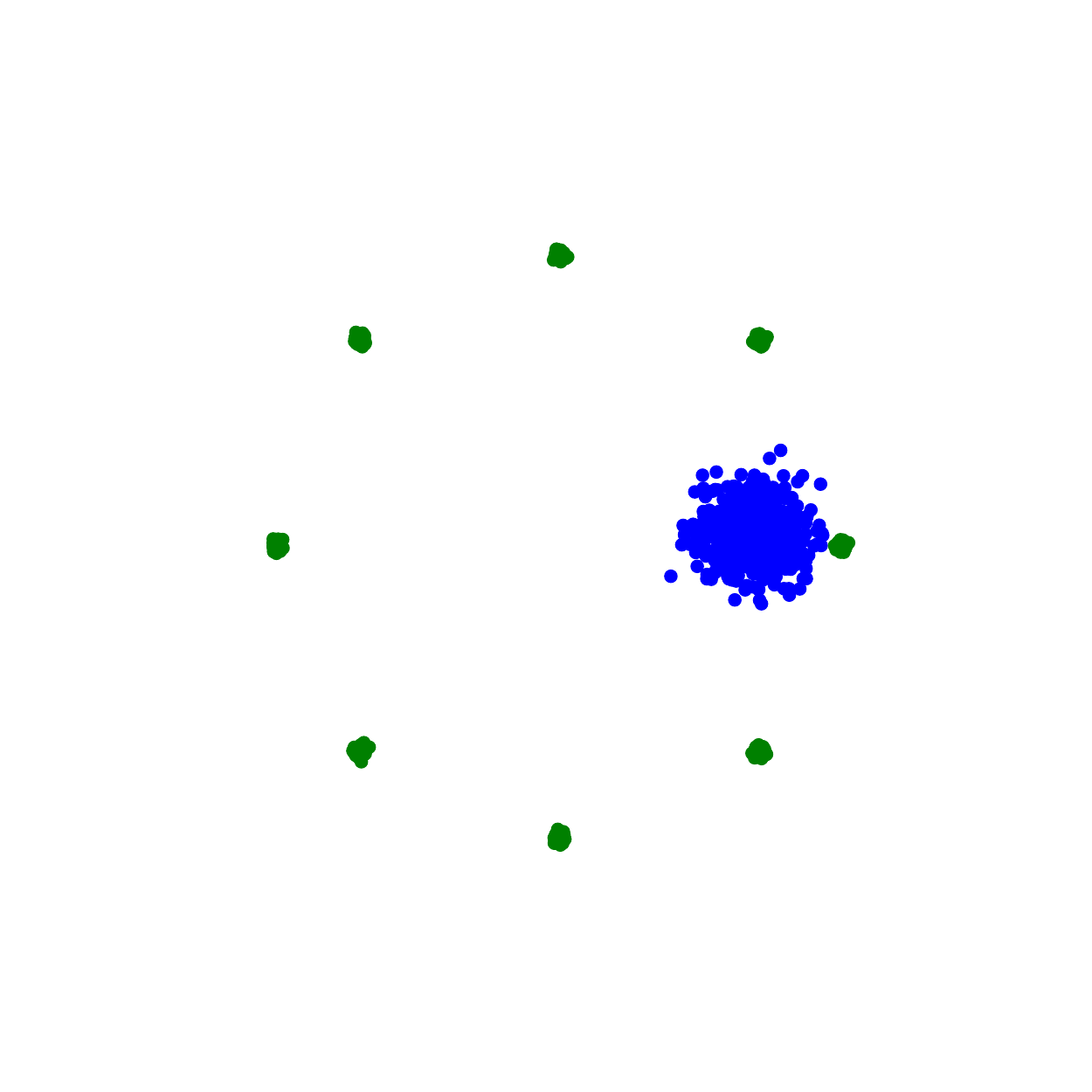} \\
         \adjincludegraphics[width=1\textwidth,trim={0 {0.1\height} 0 {0.1\height}},clip,valign=t]{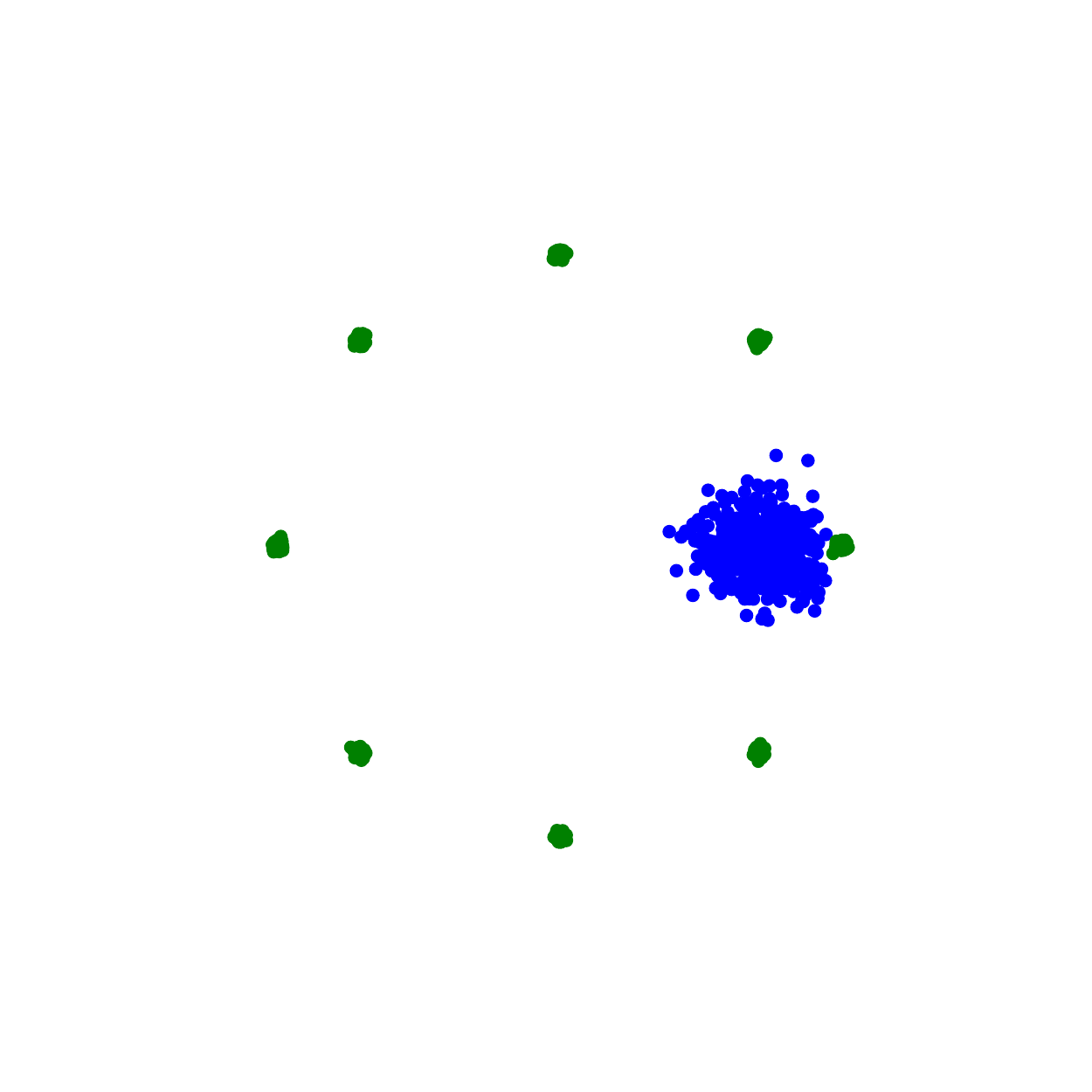}
         \caption*{Iteration 0}
    \end{minipage}%
        \begin{minipage}{.2\textwidth}
        \centering
        \adjincludegraphics[width=1\textwidth,trim={0 {0.1\height} 0 {0.1\height}},clip,valign=t]{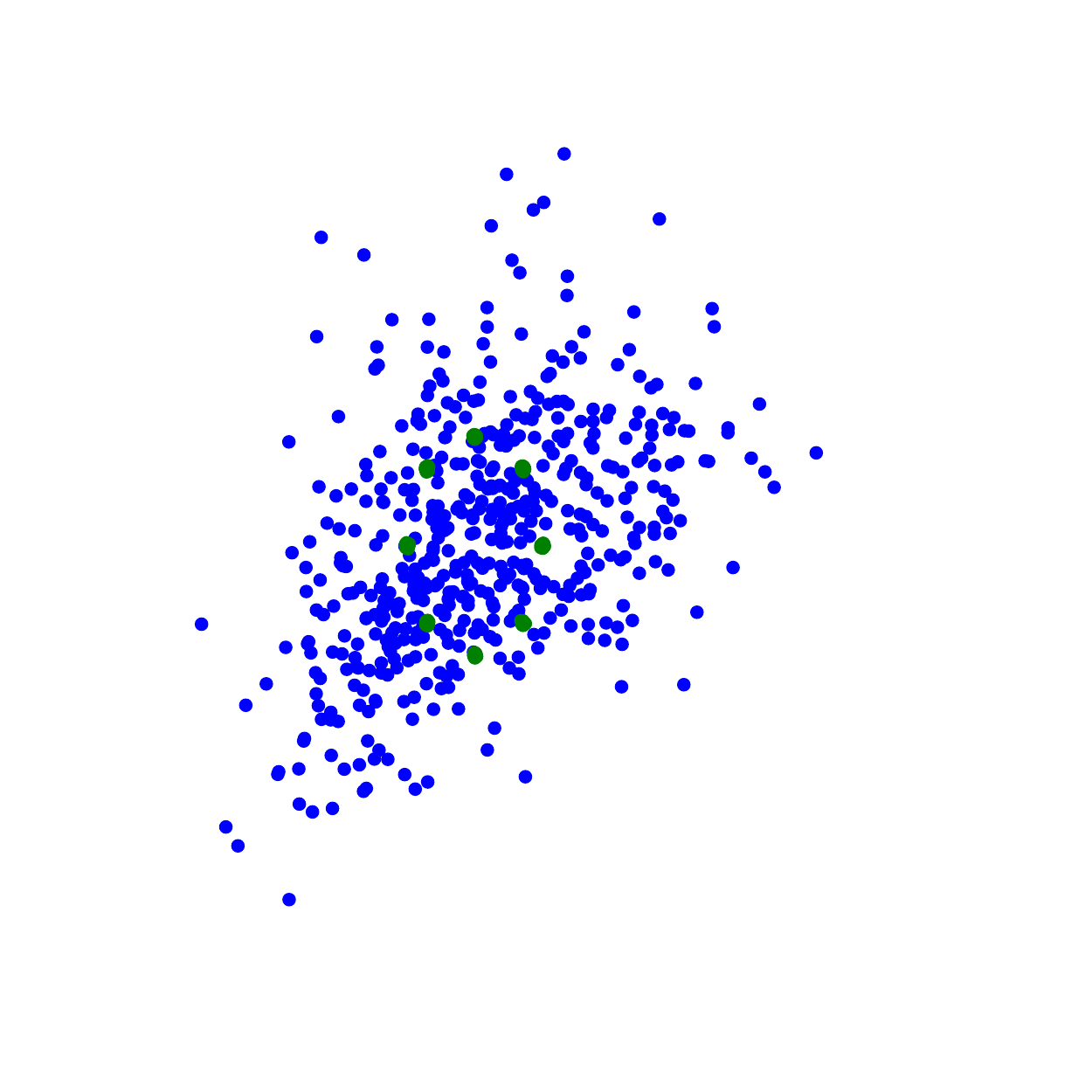} \\
         \adjincludegraphics[width=1\textwidth,trim={0 {0.1\height} 0 {0.1\height}},clip,valign=t]{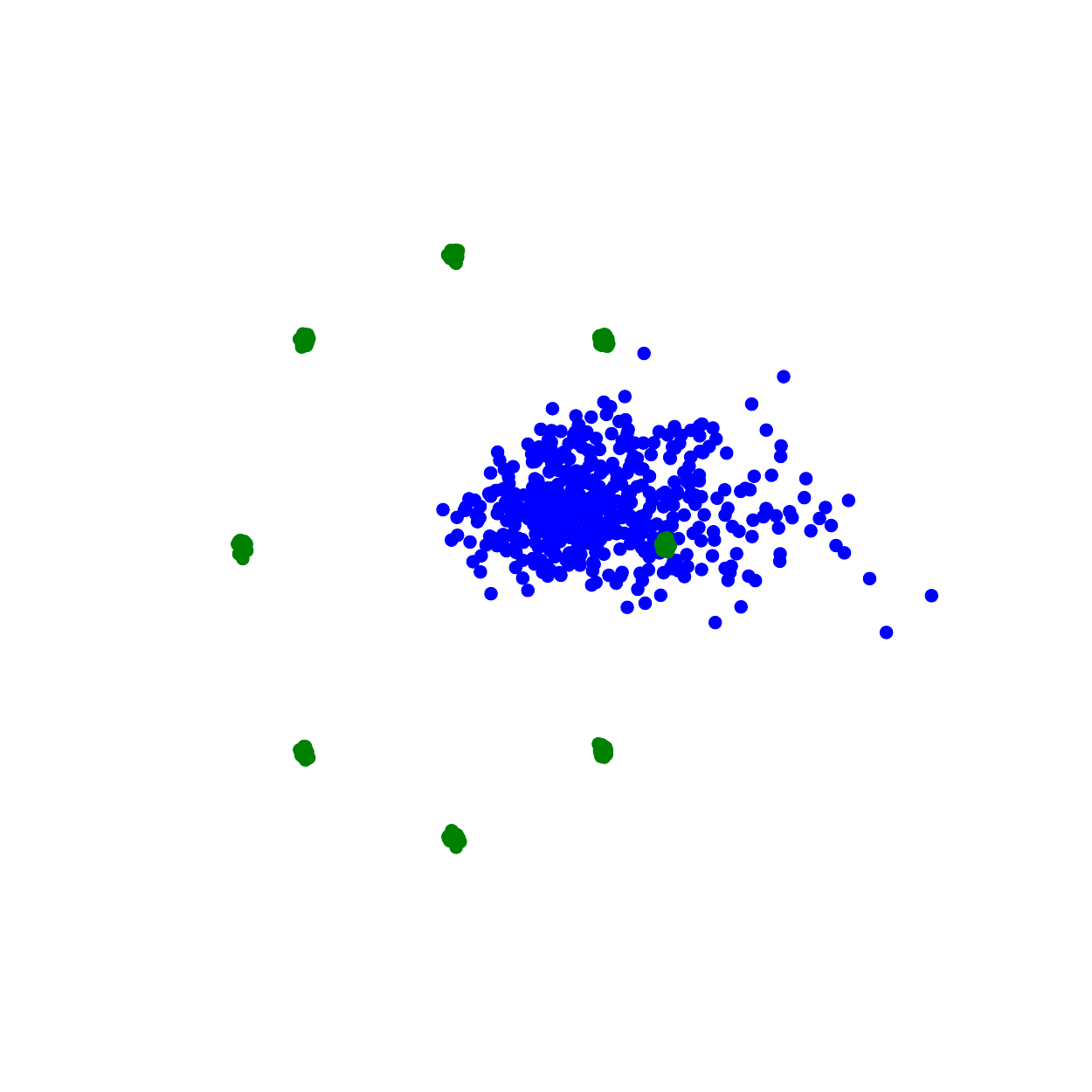}
         \caption*{Iteration 3000}
    \end{minipage}%
        \begin{minipage}{.2\textwidth}
        \centering
        \adjincludegraphics[width=1\textwidth,trim={0 {0.1\height} 0 {0.1\height}},clip,valign=t]{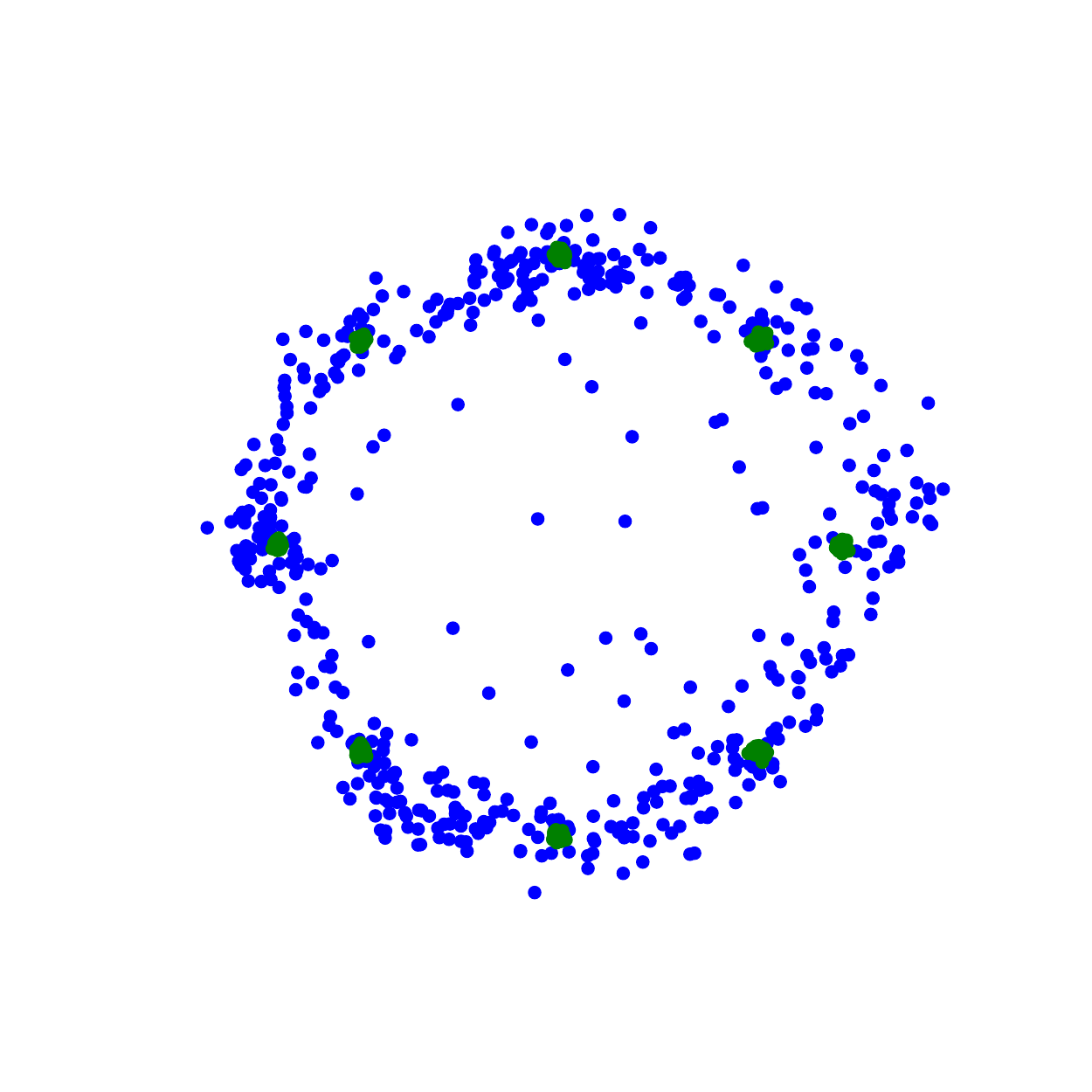} \\
         \adjincludegraphics[width=1\textwidth,trim={0 {0.1\height} 0 {0.1\height}},clip,valign=t]{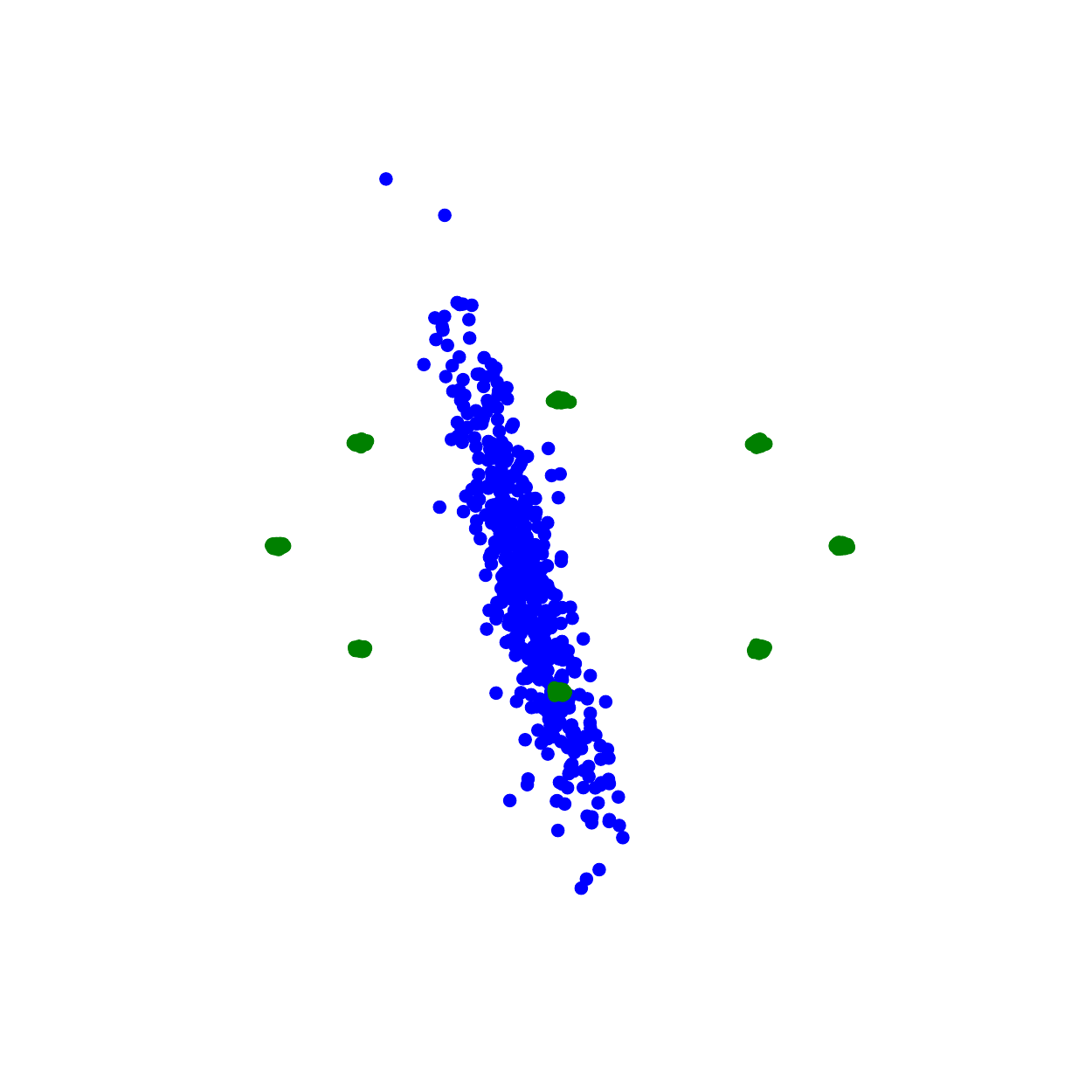}
         \caption*{Iteration 8000}
    \end{minipage}%
        \begin{minipage}{.2\textwidth}
        \centering
        \adjincludegraphics[width=1\textwidth,trim={0 {0.1\height} 0 {0.1\height}},clip,valign=t]{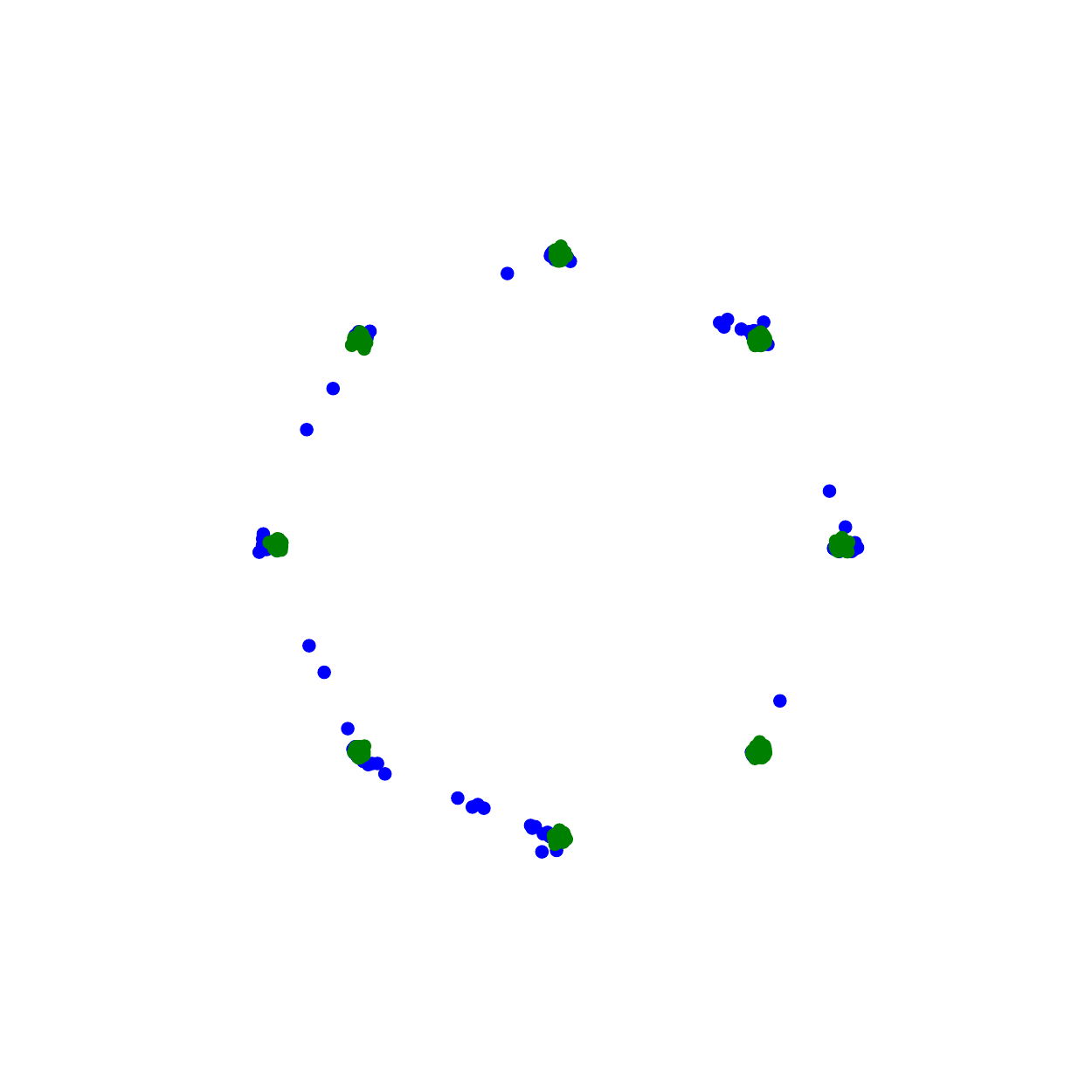} \\
         \adjincludegraphics[width=1\textwidth,trim={0 {0.1\height} 0 {0.1\height}},clip,valign=t]{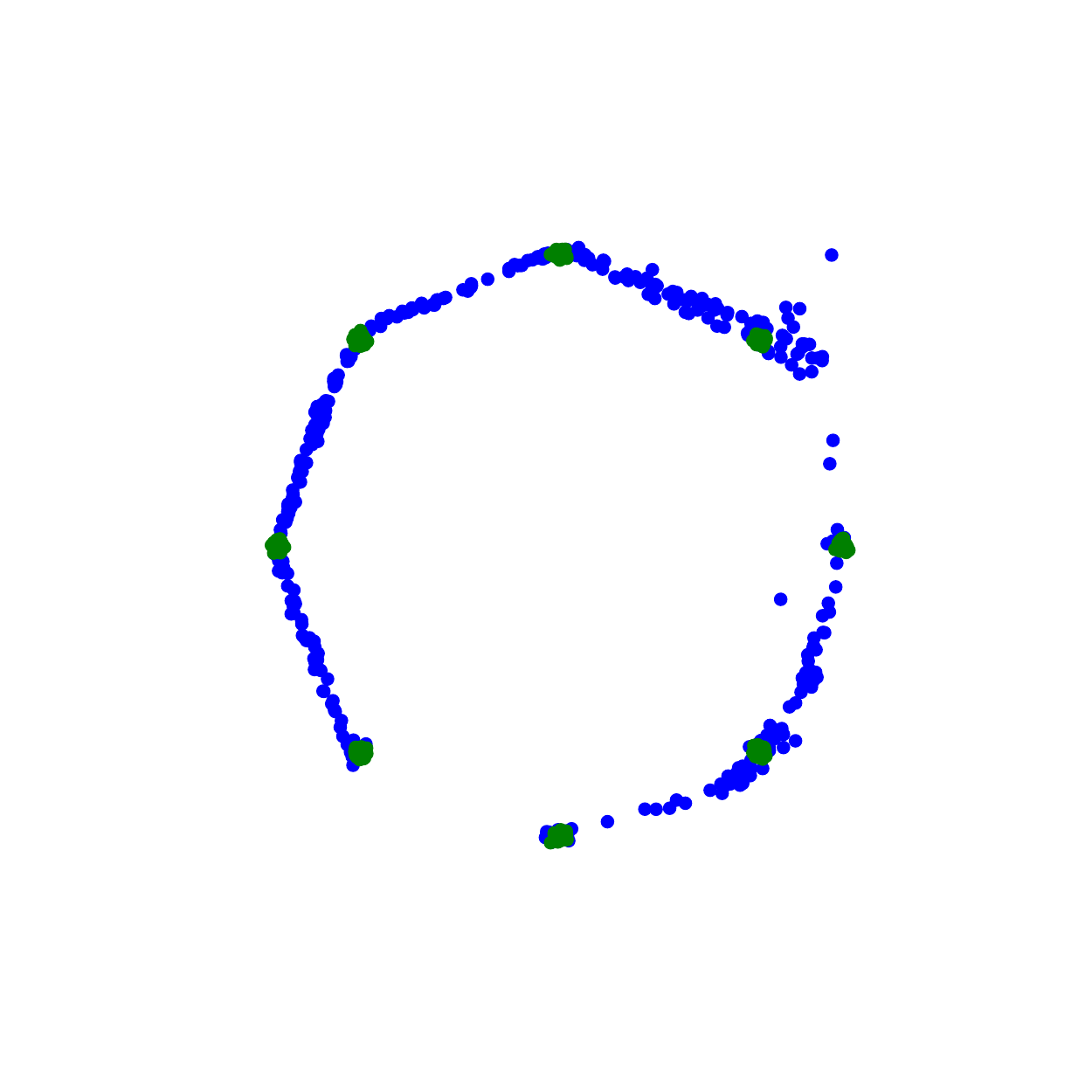}
         \caption*{Iteration 50000}
    \end{minipage}%
            \begin{minipage}{.2\textwidth}
        \centering
        \adjincludegraphics[width=1\textwidth,trim={0 {0.1\height} 0 {0.1\height}},clip,valign=t]{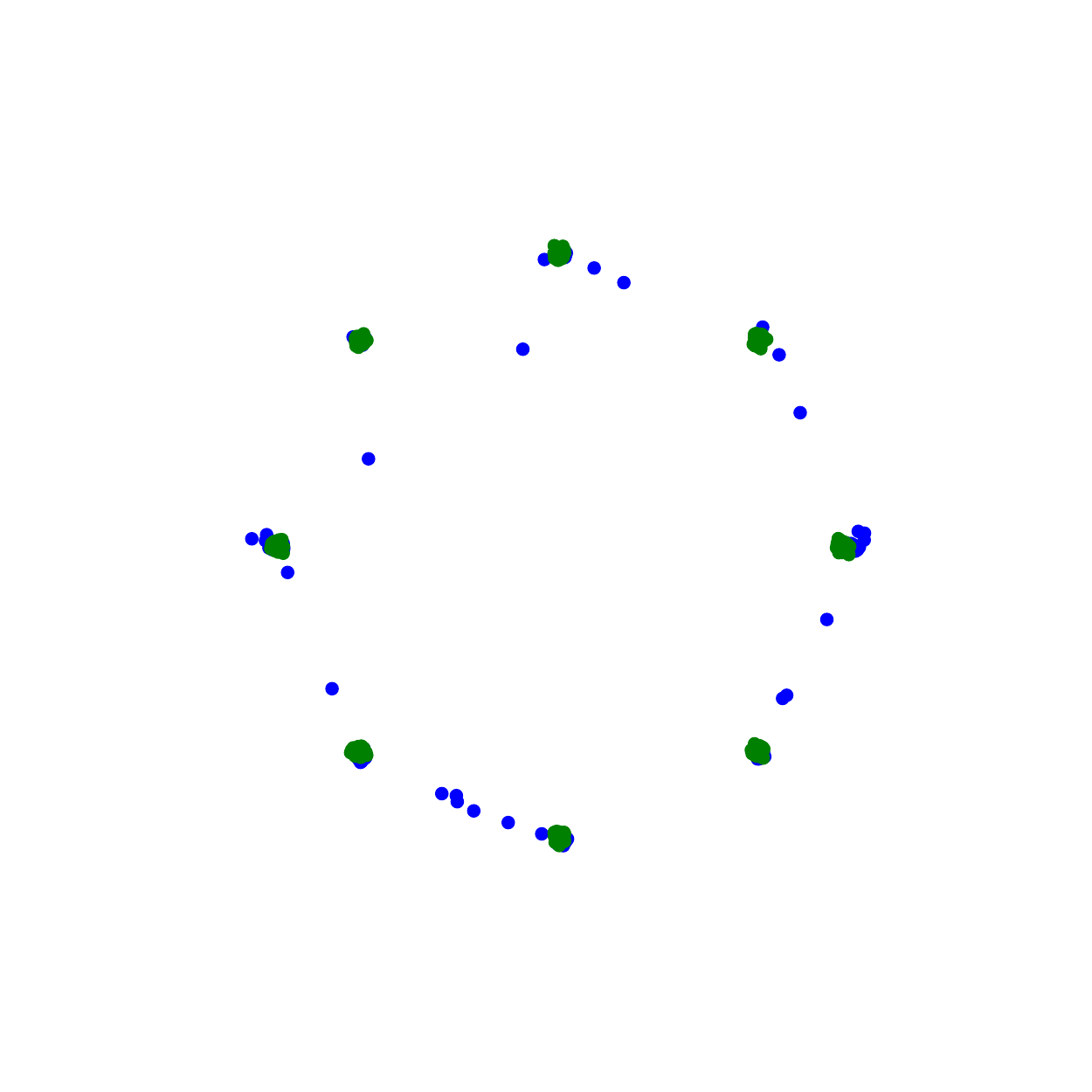} \\
         \adjincludegraphics[width=1\textwidth,trim={0 {0.1\height} 0 {0.1\height}},clip,valign=t]{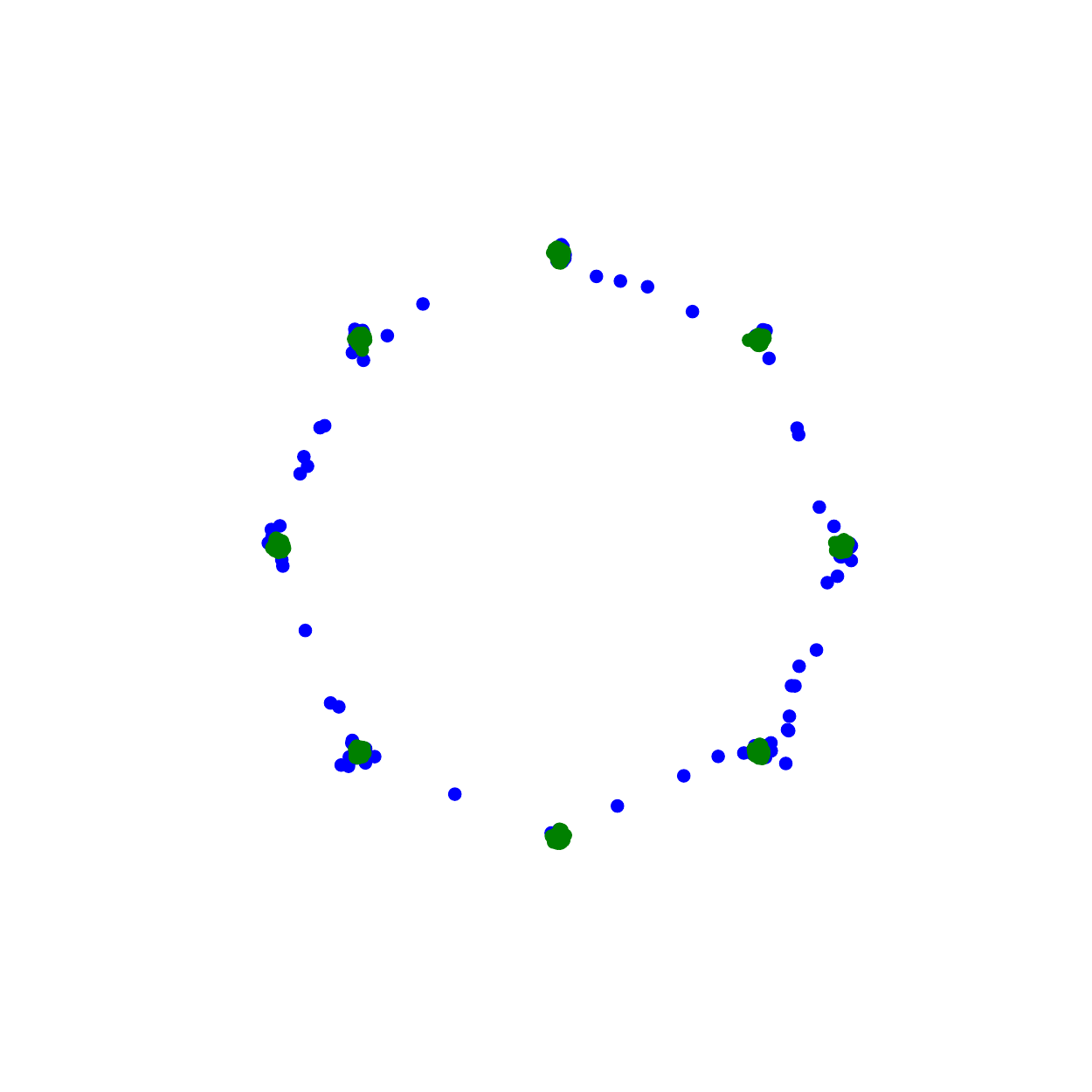}
         \caption*{Iteration 70000}
    \end{minipage}%
    \caption{Gradient regularized GAN, $\eta = 0.5$ (top row) vs. 10-unrolled with $\eta=10^{-4}$  (bottom row)}
    \label{fig:toy}
\end{figure}

\begin{figure}[!h]
    \centering
    \begin{minipage}[t]{0.5\textwidth}
    \centering
        \begin{minipage}{.3\textwidth}
        \centering
        \adjincludegraphics[width=0.95\textwidth,trim={0 {0.75\height} 0 0},clip,,valign=t]{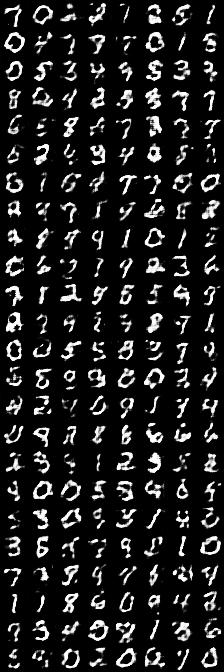} %
    \end{minipage}%
            \begin{minipage}{.3\textwidth}
        \centering
        \adjincludegraphics[width=0.95\textwidth,trim={0 {0.75\height} 0 0},clip,valign=t]{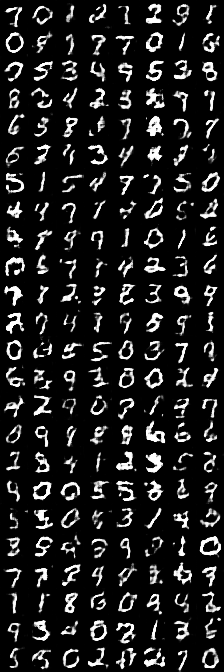} %
    \end{minipage}%
                \begin{minipage}{.3\textwidth}
        \centering
        \adjincludegraphics[width=0.95\textwidth,trim={0 {0.75\height} 0 0},clip,valign=t]{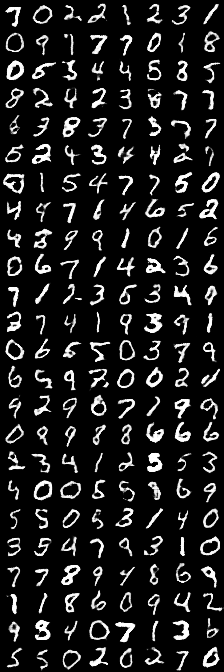} %
    \end{minipage}%
    \end{minipage}%
        \begin{minipage}[t]{0.5\textwidth}
        \centering
        \begin{minipage}{.3\textwidth}
        \centering
        \adjincludegraphics[width=0.95\textwidth,trim={0 {0.75\height} 0 0},clip,valign=t]{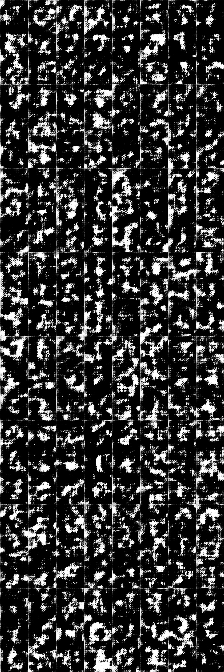} %
    \end{minipage}%
            \begin{minipage}{.3\textwidth}
        \centering
        \adjincludegraphics[width=0.95\textwidth,trim={0 {0.75\height} 0 0},clip,valign=t]{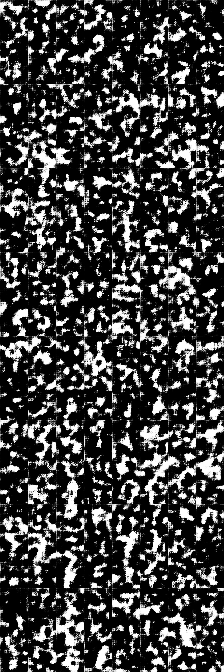} %
    \end{minipage}%
                \begin{minipage}{.3\textwidth}
        \centering
        \adjincludegraphics[width=0.95\textwidth,trim={0 {0.75\height} 0 0},clip,valign=t]{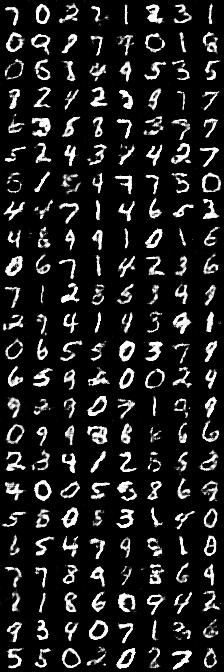} %
    \end{minipage}%
    \end{minipage}
    \caption{Gradient regularized (left) and traditional (right) DCGAN
      architectures on stacked MNIST examples, after 1,4 and 20 epochs.}
        \label{fig:mnist}
\end{figure}

\begin{figure}[!htb]
    \centering
    \begin{minipage}{.25\textwidth}
        \centering
        \includegraphics[width=1\textwidth]{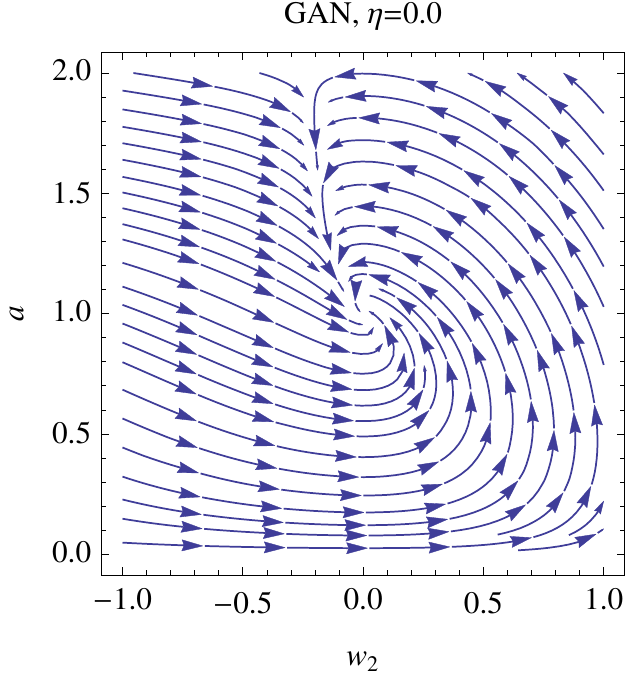} \\
         \includegraphics[width=1\textwidth]{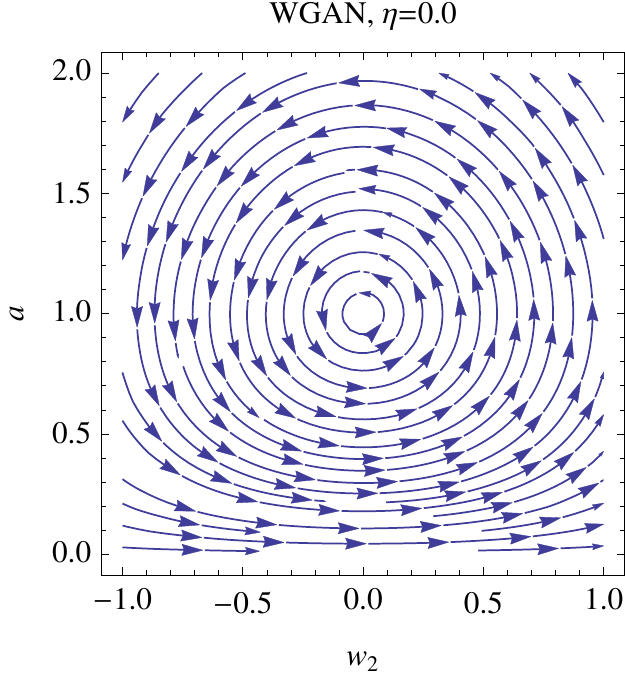}
    \end{minipage}%
        \begin{minipage}{.25\textwidth}
        \centering
        \includegraphics[width=1\textwidth]{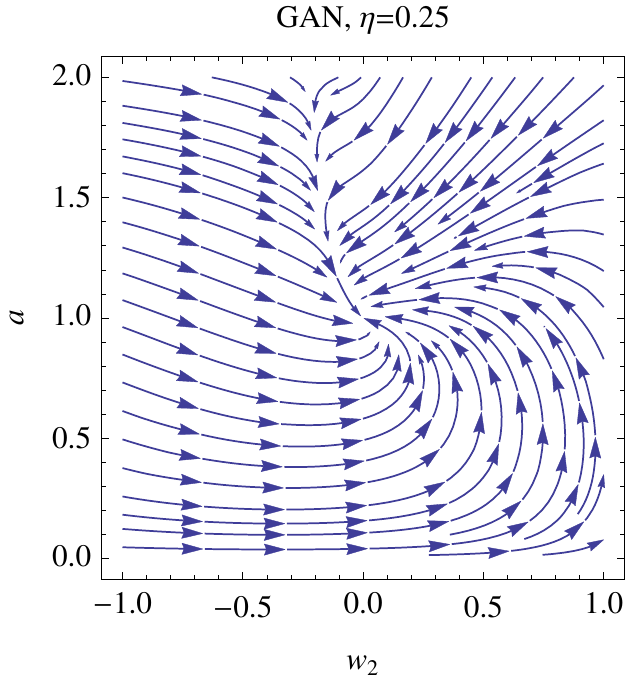} \\
         \includegraphics[width=1\textwidth]{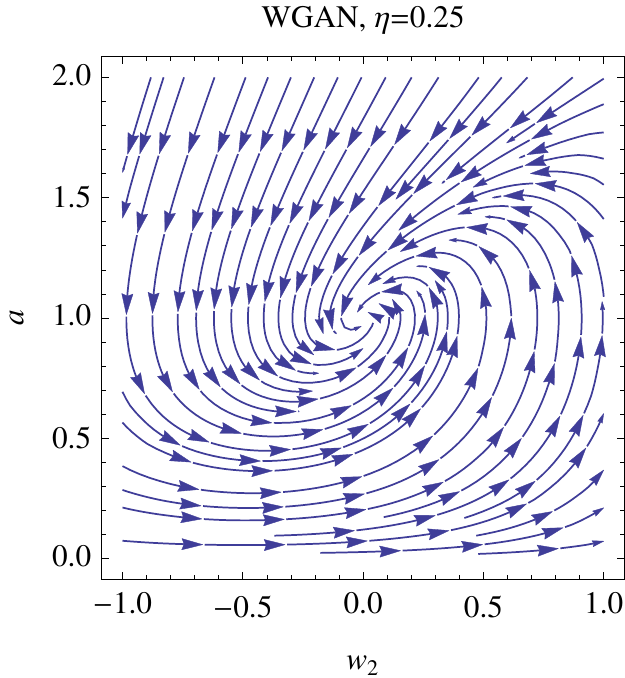}
    \end{minipage}%
        \begin{minipage}{.25\textwidth}
        \centering
        \includegraphics[width=1\textwidth]{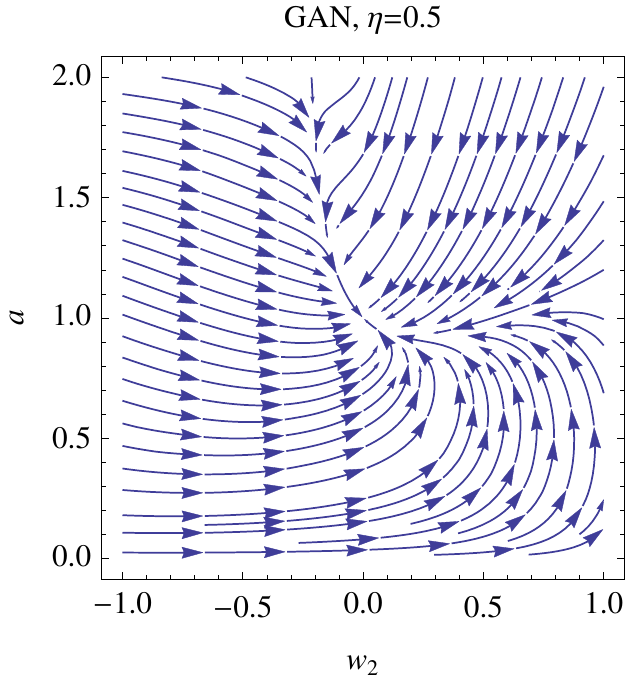} \\
         \includegraphics[width=1\textwidth]{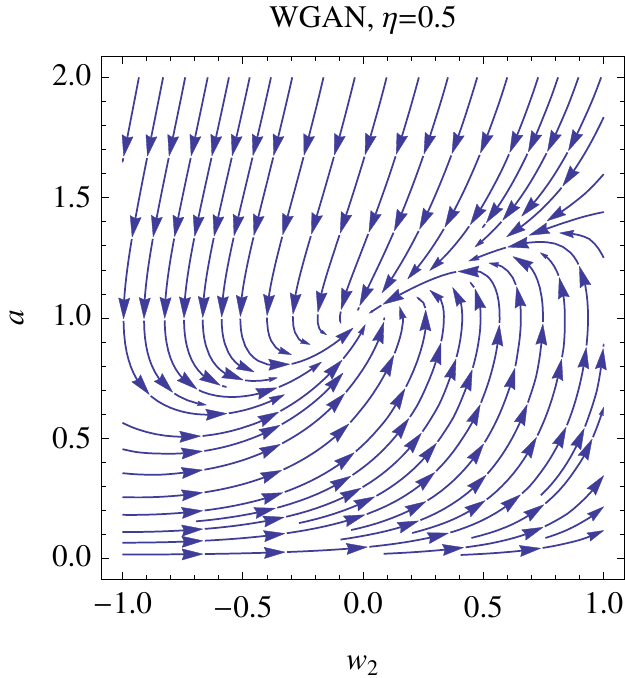}
    \end{minipage}%
            \begin{minipage}{.25\textwidth}
        \centering
        \includegraphics[width=1\textwidth]{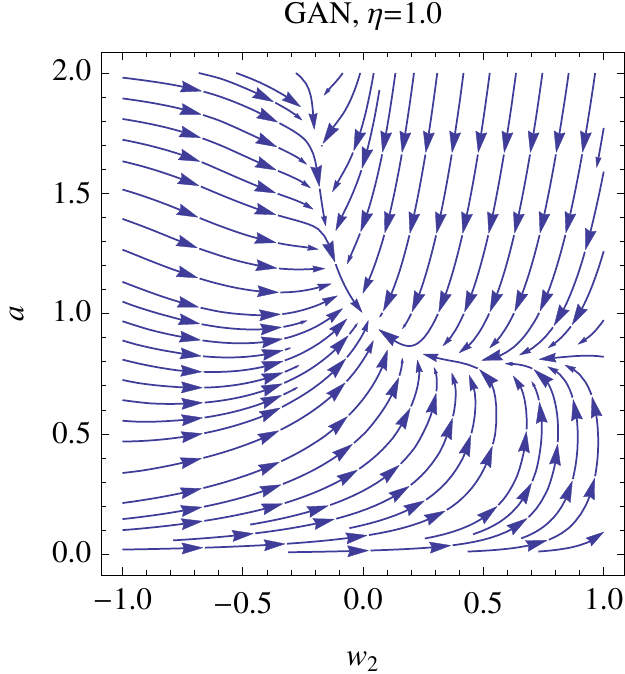} \\
         \includegraphics[width=1\textwidth]{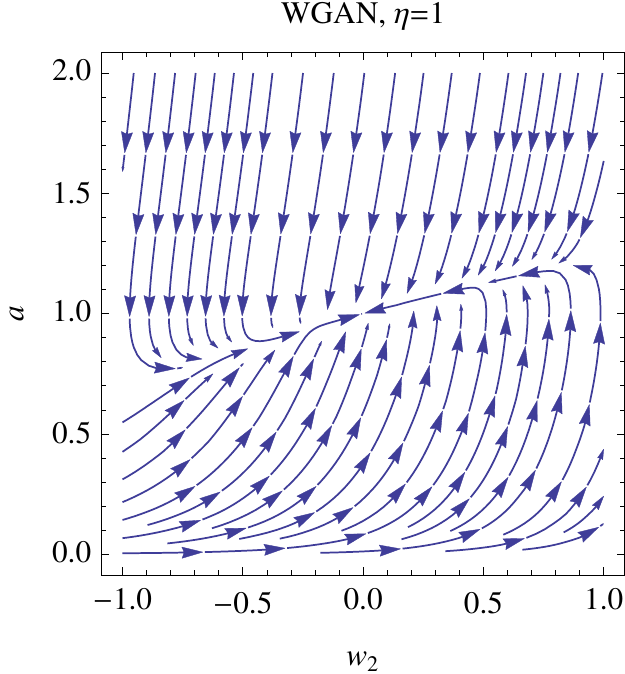}
    \end{minipage}%
          \caption{Streamline plots around the equilibrium $(0,1)$ for the conventional GAN (top) and the WGAN (bottom) for $\eta=0$ (vanilla updates) and $\eta =0.25,0.5,1$ (left to right).}
    \label{fig:streamline}
\end{figure}

\section{Conclusion}
In this paper, we presented a theoretical analysis of the local asymptotic
stability of GAN optimization under proper conditions.  We further showed that
the recently proposed WGAN is \emph{not} asymptotically stable under the same
conditions, but we introduced a gradient-based regularizer which stabilizes both
traditional GANs and the WGANs, and can improve convergence speed in practice. 

The results here provide substantial insight into the nature of GAN
optimization, perhaps even offering some clues as to why these methods have
worked so well \emph{despite} not being convex-concave. 
However, we also emphasize that there are substantial limitations to the
analysis, and directions for future work.  Perhaps most notably, the analysis
here only provides an understanding of what happens locally, close to an
equilibrium point.  For non-convex architectures this may be all that is
possible, but it seems plausible that much stronger \emph{global} convergence
results could hold for simple settings like the linear quadratic GAN (indeed, as
the streamline plots show, we observe this in practice for simple domains).
Second, the analysis here does not show the equilibrium points necessarily
exist, but only illustrates convergence if there do exist points that satisfy
certain criteria: the existence question has been addressed by previous work
\citep{arora2017generalization}, but much more analysis remains to be done here.
GANs are rapidly becoming a cornerstone of deep learning methods, and the
theoretical and practical understanding of these methods will prove crucial in
moving the field forward.

\subparagraph{Acknowledgements.} We thank Lars Mescheder for pointing out a missing condition in the relaxed version of Assumption~\ref{as:same-support} (see Appendix~\ref{app:realizable-relaxed}) in earlier versions of this manuscript.

\bibliographystyle{plainnat}
\bibliography{references}

\newpage
\appendix
    \section*{Appendix}

\begin{figure}[!htb]
    \centering
        \begin{minipage}{.19\textwidth}
        \centering
        \includegraphics[width=0.7\textwidth]{Images/mnist-refined/1} \\
         \includegraphics[width=0.7\textwidth]{Images/mnist-vanilla/1}
         \caption*{Epoch 1}
    \end{minipage}%
        \begin{minipage}{.19\textwidth}
        \centering
        \includegraphics[width=0.7\textwidth]{Images/mnist-refined/2} \\
         \includegraphics[width=0.7\textwidth]{Images/mnist-vanilla/2}
         \caption*{Epoch 2}
    \end{minipage}%
        \begin{minipage}{.19\textwidth}
        \centering
        \includegraphics[width=0.7\textwidth]{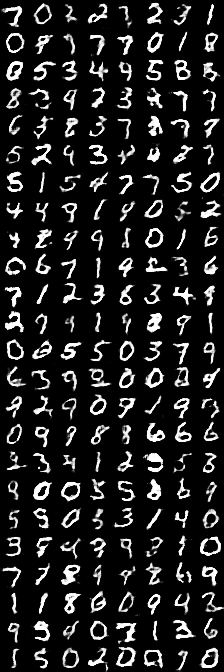} \\
         \includegraphics[width=0.7\textwidth]{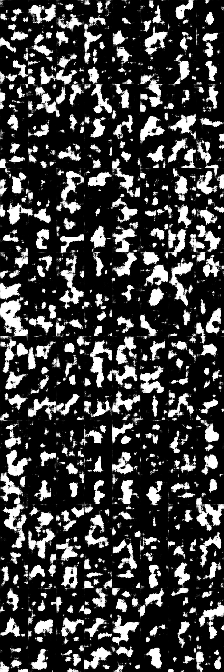}
         \caption*{Epoch 4}
    \end{minipage}%
            \begin{minipage}{.19\textwidth}
        \centering
        \includegraphics[width=0.7\textwidth]{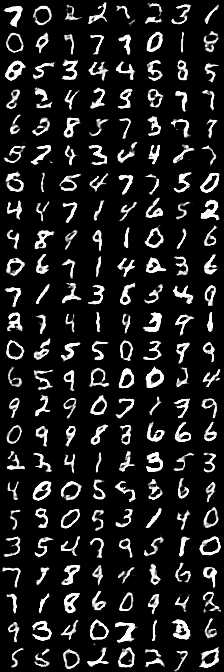} \\
         \includegraphics[width=0.7\textwidth]{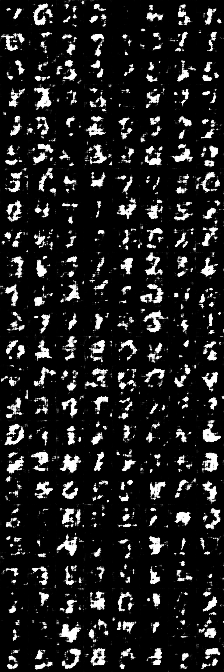}
         \caption*{Epoch 8}
    \end{minipage}%
                \begin{minipage}{.19\textwidth}
        \centering
        \includegraphics[width=0.7\textwidth]{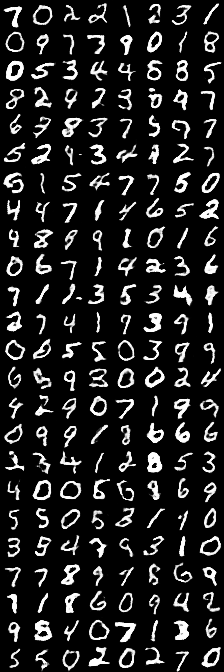} \\
         \includegraphics[width=0.7\textwidth]{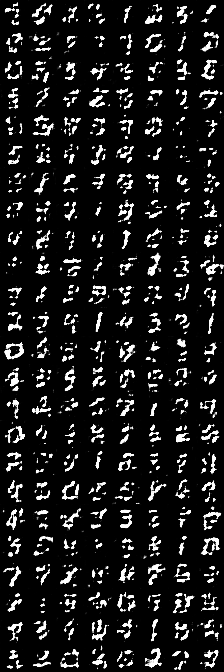}
         \caption*{Epoch 16}
    \end{minipage}%
                    \begin{minipage}{.19\textwidth}
        \centering
        \includegraphics[width=0.7\textwidth]{Images/mnist-refined/20} \\
         \includegraphics[width=0.7\textwidth]{Images/mnist-vanilla/20}
         \caption*{Epoch 20}
    \end{minipage}%
    \caption{Gradient regularized GAN with $\eta = 5 \times 10^{-6}$ vs. traditional GAN}
    \label{fig:mnist-full}
\end{figure}

\section{Preliminaries}
\label{app:lyapunov}

In this section, we present preliminaries from non-linear systems theory \citep{khalil1996noninear}. In particular, we formally define local stability of dynamic systems, and then present an important theorem that helps us study stability of non-linear systems. Finally, we present a modification of this result that will be crucial in proving stability of GANs under our assumptions.

Consider a system consisting of variables $\vec{\theta} \in \mathbb{R}^n$ whose time derivative is defined by $h(\vec{\theta})$ i.e.,
\begin{align}
\label{eq:basicsystem}
\vec{\dot{\theta}} = h(\theta).
\end{align}

Without loss of generality let the origin be an equilibrium point of this sytem. That is, 
$h(\vec{0}) = \vec{0}$. Let $\vec{\theta}(t)$ denote the state of the system 
at some time $t$. Then, we have the following definition of local stability:

\begin{definition}[\textbf{Stability}]
\label{def:stability} (Definition 4.1 from \citet{khalil1996noninear})
The origin of the system in Equation \ref{eq:basicsystem} is 
\begin{itemize}
	\item stable if for each $\epsilon > 0$, there is $\delta=\delta(\epsilon) > 0$ such that
	\[
\| \vec{\theta}(0) \| < \delta \implies \| \vec{\theta}(t)\| < \epsilon , \, \forall t \geq 0.
	\]
	\item unstable if not stable.
	\item asymptotically stable if it is stable and $\delta > 0$ can be chosen such that 
	\[
\|\vec{\theta}(0) \| < \delta \implies \lim_{t \to \infty} \theta(t) = 0
	\]
	\item exponentially stable if it is asymptotically stable and $\delta, k, \lambda > 0$ can be chosen such that 
	\[
\|\vec{\theta}(0) \| < \delta \implies  \|\vec{\theta}(t) \| \leq k \| \vec{\theta}(0)  \| \exp(-\lambda t)
	\]
	\end{itemize}
\end{definition}

The system is stable if for any chosen ball around the equilibrium (of radius $\epsilon$), one can initialize the system anywhere within a sufficiently small ball around the equilibrium (of radius $\delta(\epsilon)$) such that the system always stays within the $\epsilon$ ball. Note that such a system may either converge to equilibrium or orbit around equilibrium perennially within the $\epsilon$ ball. In contrast, a system is unstable if there are initializations that are arbitrarily close to the equilibrium which can escape the $\epsilon$-ball. Finally, asymptotic stability is a stronger notion of stability, which implies that there is a region around the equilibrium such that any initialization within that region will converge to the equilibrium (in the limit $t \to \infty$). For example, as we saw, GANs are always stable; however, WGANs are stable but not asymptotically stable. 

\subparagraph{Extension to multiple equilibria.} Note that since a GAN system might have multiple arbitrarily close equilibria, or a subspace of equilibria, we will define asymptotic stability to imply convergence to any of the equilibria in the neighborhood of a considered equilibrium. That is, $\lim_{t \to \infty} \theta(t) = \vec{\theta^{\star}}$ where $\vec{\theta^{\star}}$ is either the considered equilibrium point at the origin or any other equilibrium point that is within some small neighborhood around origin.

We now present Lyapunov's stability theorem which is used to prove locally asymptotic stability of a given system. The basic idea is that a system is asymptotically stable if we can find a scalar ``energy'' function $V(\vec{\theta})$ (also called a Lyapunov function) that i) is {\em positive definite} which means, $V(\vec{\theta})$ positive everywhere and zero at the equilibrium ii) its time derivative $\dot{V}(\vec{\theta})$ is strictly negative around the equilibrium. 
\begin{theorem}[\textbf{Lyapunov function}](Theorem 4.1 from \citet{khalil1996noninear})
\label{thm:lyapunov}
Let $B_{\epsilon}(\vec{0})$ be a small region around the origin of the system in Equation~\ref{eq:basicsystem}. Let $V: B_{\epsilon}(\vec{0}) \to \mathbb{R}$ be a continuously differentiable  function such that
\begin{itemize}
\item it is {\em positive definite} i.e., $V(\vec{0}) = 0$ and  $V(\vec{\theta}) > 0$ for $\vec{\theta} \in B_{\epsilon}(0) - \{ \vec{0} \}$
\item $\dot{V}(\vec{\theta}) \leq 0$ for $\vec{\theta} \in B_{\epsilon}(0) - \{ 0\} $
\end{itemize}
Then, the origin is stable. Moreover, if 
\[
\dot{V}(\vec{\theta}) < 0, \, \forall \vec{\theta} \in B_{\epsilon}(0) - \{ \vec{0}\} 
\]
then the origin is asymptotically stable.
\end{theorem}

We next present an important tool that simplifies the study of stability of non-linear systems. The result is that one can ``linearize'' any non-linear system  near an equilibrium and analyze the stability of the linearized system to comment on the local stability of the original system.

\begin{theorem}[\textbf{Linearization}] 
\label{thm:jacobian} (Theorem 4.5 from \citet{khalil1996noninear})
Let $\vec{J}$ be the Jacobian of the system in Equation~\ref{eq:basicsystem} at its 
 origin i.e.,
 \[
\vec{J} = \left.\der{\vec{\theta}}{h(\vec{\theta})}\right\vert_{\vec{\theta} = \vec{0}}.
 \]

 Then,
 \begin{itemize}
\item The origin is locally exponentially stable if $\vec{J}$ is Hurwitz i.e., $\Re(\lambda) < 0$ for all eigenvalues $\lambda$ of $\vec{J}$.
\item The origin is unstable if $\Re(\lambda) > 0$ for all eigenvalues $\lambda$ of $\vec{J}$.
 \end{itemize}
\end{theorem}

The key idea in the proof for this result is that the system can be written as $h(\vec{\theta}) = \vec{J} \vec{\theta} + g_1(\vec{\theta})$, where $g_1(\vec{\theta})$, the remainder of the linear approximation is bounded as $\| g_1(\vec{\theta})\| \leq O(\|\vec{\theta}  \|^2)$ sufficiently close to equilibrium.  Now, it turns out that when $\vec{J}$ is Hurwitz, one can find a quadratic Lyapunov function for the original system whose rate of decrease is also quadratic in $\vec{\theta}$. Since, $\| g_1(\vec{\theta})\|$ is only a quadratic remainder term, one can show that the remainder term only adds a cubic term to the change in the Lyapunov function. This is however smaller than a quadratic change near the equilibrium, and therefore the quadratic Lyapunov function for the linearized system works as a Lyapunov function for the original system too.

In all our analyses, we will linearize our system and show that the Jacobian is Hurwitz. However, it is often useful to identify the quadratic Lyapunov function for the (linearized) system.
Unfortunately, for some of the Jacobians we will encounter, it is hard to come up with a
quadratic Lyapunov function that always strictly decreases. Instead, we will identify a function that either strictly decreases or sometimes remains constant but only instantenously. While Lyapunov's stability theorem does not help us conclude anything about stability for this case, the following corollary of LaSalle's theorem (we do not state the theorem here) is sufficient to prove asymptotic stability in this case.

\begin{theorem}[\textbf{Corollary of LaSalle's invariance principle}, Corollary 4.1 from \citet{khalil1996noninear}]
\label{thm:strong-lyapunov}
Let $B_{\epsilon}(\vec{0})$ be a small region around an equilibrium $0$ of the system in Equation~\ref{eq:basicsystem}. Let $V: B_{\epsilon}(\vec{0}) \to \mathbb{R}$ be a continuously differentiable function such that
\begin{itemize}
\item  $V(\vec{\theta}) = 0$ if and only if $\vec{\dot{\theta}} = 0$ and $V(\vec{\theta}) > 0$  for $\vec{\theta} \in B_{\epsilon}(0) - \{ \vec{0} \}$ such that $\vec{\dot{\theta}} \neq 0$ .
\item $\dot{V}(\vec{\theta}) \leq 0$ for $\vec{\theta} \in B_{\epsilon}(\vec{0}) - \{ 0\} $
\item Let $S = \{\vec{\theta} \in B_{\epsilon}(\vec{0}) \, | \, \dot{V}(\vec{\theta})  = 0 \}$. There is no trajectory that identically stays in $S$ except for the trajectories at equilibrium points.
\end{itemize}
then the system is locally asymptotically stable with respect to $\vec{0}$ and other equilibria in its neighborhood. 
\end{theorem}

Finally, we prove an extension of the linearization theorem that helps us deal with analyzing the stability of a special kind of non-linear systems, specifically those with multiple equilibria in a local neighborhood of a considered equilibrium. The theorem, though inuitively follows from the original linearization theorem itself, is not a standard theorem in non-linear systems, to the best of our knowledge. 

Formally, we consider a case where the system consists of two sets of parameters $\vec{\theta}$ and $\vec{\gamma}$ such that from the equilibrium, any small perturbation along $\vec{\gamma}$ preserves the equilibrium. We show that it is enough to show that the Jacobian with respect to $\vec{\theta}$ is Hurwitz to prove stability. 

\begin{theorem}
\label{thm:multiple-equilibria}
Consider a non-linear system of parameters $(\vec{\theta}, \vec{\gamma})$,
\begin{align}
\label{eq:theta-gamma}
\vec{\dot{\theta}} = h_1(\vec{\theta},\vec{\gamma}), \vec{\dot{\gamma}} = h_2(\vec{\theta},\vec{\gamma})
\end{align}
with an equilibrium point at the origin. Let there exist $\epsilon$ such that for any $\vec{\gamma} \in B_{\epsilon}(\vec{0})$,  $(\vec{0},\vec{\gamma})$ is an equilibrium point. Then, if 
\begin{align}
\label{eq:theta}
\vec{J} = \left. \der{\vec{\theta}}{h_1(\vec{\theta}, \vec{\gamma})} \right\vert_{(\vec{0},\vec{0})}
\end{align}

is a Hurwitz matrix,
the non-linear system in Equation~\ref{eq:theta-gamma} is exponentially stable.
\end{theorem}

\begin{proof}
The proof for this statement is quite similar to the proof of the original theorem for linearization. The high level idea is that if $\vec{J}$ is exponentially stable, then there exists a quadratic Lyapunov function that is always decreasing for the system $\dot{\vec{\theta}} = \vec{J} \vec{\theta} $. Then, we show that the same quadratic function works for the original non-linear system too in a small neighborhood around equilibrium for which
the  non-linear remainder terms are sufficiently small. In particular, we show that $\vec{\theta}$ converges to zero, and $\vec{\gamma}$ converges to a value less than $\epsilon$.

A subtle point however, is that this quadratic function would decrease only when it is within a particular neighborhood of $\vec{\theta}$ around origin, and also a particular neighborhood of  $\vec{\gamma}$ around origin. However, within this neighborhood, say $\mathcal{S}$, we can only guarantee that $\vec{\theta}$ exponentially approaches the origin; $\vec{\gamma}$ might move away from the $\epsilon$-neighborhood around origin, and if it does, the system may exit $\mathcal{S}$ and the system may not even converge! We carefully overcome this, by first identifying $\mathcal{S}$, and then identifying a smaller space within $\mathcal{S}$ where $\vec{\gamma}$ does not vary too much over the course of convergence, so that the system stays within $\mathcal{S}$ forever -- until convergence.

To identify $\mathcal{S}$, let,
\[
h_1(\theta, \gamma) = \vec{J} \vec{\theta}  + g_1(\vec{\theta}, \vec{\gamma}).
\]

The first crucial step is to show that for any constant $c > 0$, for a sufficiently small neighborhood around the equilibrium, we will have $\| g_1(\vec{\theta}, \vec{\gamma}) \| \leq c\|\vec{\theta}\| $. To show this, consider the Taylor series expansion for the remainder $g_1(\vec{\theta}, \vec{\gamma}) = h_1(\vec{\theta}, \vec{\gamma}) -\vec{J} \vec{\theta}$ around equilibrium. Clearly, the expansion would not have a constant term because $h_1(\vec{0},\vec{0})=0$. It would not have a linear term in $\vec{\theta}$ because that is accounted for already. Finally, it will not have any term that is purely a function of $\vec{\gamma}$, because $h_1(\vec{0},\vec{\gamma})=0$ in a small neighborhood around equilibrium (since $(\vec{0},\vec{\gamma})$ are all equilibria). Therefore, we can write:

\[
g_1(\vec{\theta}, \vec{\gamma}) = \vec{\theta} g_2(\vec{\gamma}) + g_3(\vec{\theta}, \vec{\gamma})
\]

where $g_2(\vec{\gamma})$ only consists of linear or higher degree terms in $\vec{\gamma}$ and $g_3(\vec{\theta}, \vec{\gamma})$ consists only of terms that are quadratic or higher degree terms in $\vec{\theta}$ (and any arbitrary degree of $\vec{\gamma}$). Therefore, we have that:
\[
\lim_{\vec{\gamma} \to \vec{0}} g_2(\vec{\gamma}) = 0  , \,\,\, \lim_{\vec{\theta} \to 0} \frac{g_3(\vec{\theta}, \vec{\gamma})}{\| \vec{\theta} \|} = 0
\]

Then, for an arbitrarily chosen small constant $c$, for a sufficiently close neighborhood around the equilibrium, we can say that $\|g_2(\vec{\gamma})\| \leq c/2$ and $\|g_3(\vec{\theta}, \vec{\gamma}) \|\leq c \| \vec{\theta} \| /2 $. Thus,
\[
\|g_1(\vec{\theta}, \vec{\gamma})\|\leq \|\vec{\theta}\| \| g_2(\vec{\gamma}) \|+\| g_3(\vec{\theta}, \vec{\gamma}) \| \leq c\| \vec{\theta}  \|
\]

We will use this property soon for a cleverly chosen value of $c$. 
Now, by Theorem 4.6 in \citet{khalil1996noninear}, we have that for any positive definite symmetric matrix $\vec{Q}$, there exists a positive definite matrix $\vec{P}$ such that $\vec{J}^T \vec{P} + \vec{J}\vec{P}  = -\vec{Q}$. Then, if we choose $V(\vec{\theta}) = \vec{\theta}^T \vec{P} \vec{\theta}$ as the quadratic Lyapunov function for the linearized system $\vec{\dot{\theta}} = \vec{J} \vec{\theta}$, the rate of its decrease is given by $\dot{V}(\vec{\theta})=-\vec{\theta}^T \vec{Q} \vec{\theta}$ which is negative at all points except at $\vec{\theta} = 0$.

Now, if we use the same Lyapunov function for the whole system as $V(\vec{\theta}, \vec{\gamma}) = \vec{\theta}^T \vec{P} \vec{\theta}$, the rate of its decrease near the origin would be $\dot{V}(\vec{\theta}, \vec{\gamma}) = -\vec{\theta}^T \vec{Q} \vec{\theta} + 2\vec{\theta}^T \vec{P} g_1(\vec{\theta}, \vec{\gamma})$. If we choose a sufficiently small neighborhood such that for $c=\frac{1}{4 \|\vec{P} \|_F} \lambda_{\min}(\vec{Q})$, $\|g_1(\vec{\theta}, \vec{\gamma})\| \leq c \| \vec{\theta}\|$, then we have that,

\[
\dot{V}(\vec{\theta}, \vec{\gamma}) \leq - \lambda_{\min}(\vec{Q}) \| \vec{\theta}\|^2 + \frac{2}{4 \|\vec{P} \|_F} \lambda_{\min}(\vec{Q}) \| \vec{P}\|_F \| \vec{\theta} \|^2 = - \frac{1}{2}\lambda_{\min}(\vec{Q}) \| \vec{\theta}\|^2  < 0
\] 

Now, as long as we ensure that the trajectory of the system remains in the neighborhood around origin for which $|g_1(\vec{\theta}, \vec{\gamma})\| \leq \frac{1}{4 \|\vec{P} \|_F} \lambda_{\min}(\vec{Q}) \| \vec{\theta}\|$ and $\| \vec{\gamma} \| < \epsilon$, this system would then exponentially converge to one of the equilibria near origin.  Let us call this neighborhood $\mathcal{S}$ i.e.,  within this neighborhood of $\vec{\gamma}$ and $\vec{\theta}$, the Lyapunov function strictly decreases for the non-linear system.

This brings us to the second crucial part of this proof, which is to ensure that we always stay in $\mathcal{S}$. Let $\mathcal{S}$ contain a ball of radius $d$. We will show that for sufficiently close initializations which are within a ball of radius $d/2$, the displacement of $\vec{\gamma}$ is at most $d/2$. Since $\vec{\theta}$ only approaches origin, this means that the system never exited $\mathcal{S}$.

To bound how much $\gamma$ changes with time, let us consider the Taylor series expansion of $h_2(\vec{\theta}, \vec{\gamma})$. First of all, there is no constant term. Next, there is no term that is purely a function of $\vec{\gamma}$ because $h_2(\vec{0}, \vec{\gamma}) = 0$. Then, we can say that:

\[
h_2( \vec{\theta}, \vec{\gamma}) =  g_4(\vec{\theta},\vec{\gamma}) \vec{\theta}
\]

Since $g_4(\vec{0},\vec{0})$ is finite, in a small neighborhood around equilibrium, there exists a fixed constant $c'$ such that $\|g_4(\vec{\theta},\vec{\gamma}) \|_2 \leq c'$. Then, $h_2( \vec{\theta}, \vec{\gamma}) \leq c' \| \vec{\theta} \|$.

\begin{figure}[H]
\centering
\includegraphics[scale=0.75]{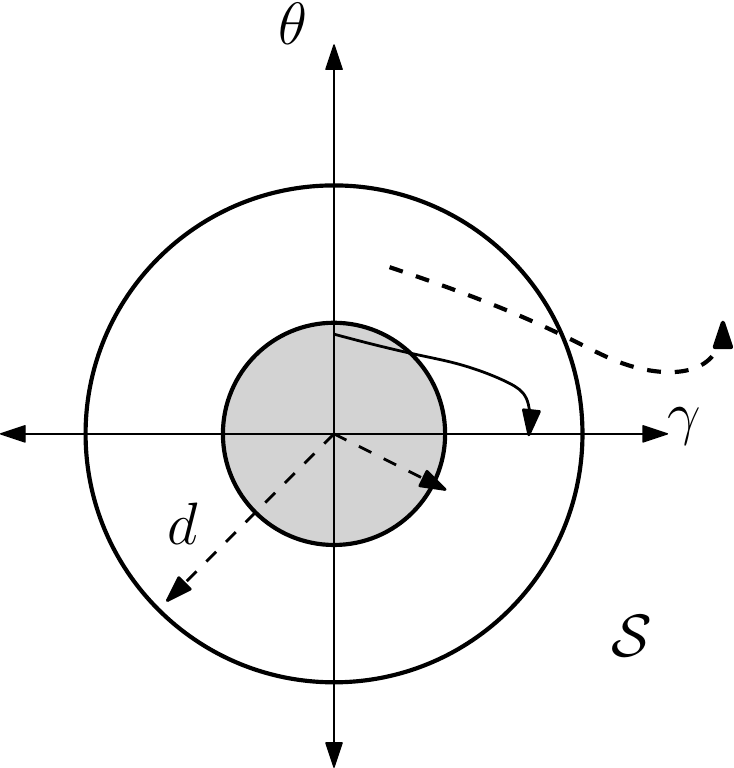}
\caption{\textbf{Illustration of Theorem~\ref{thm:multiple-equilibria}}: $\mathcal{S}$ is the neighborhood within which $\vec{\theta}$ converges exponentially to $\vec{0}$ to a point on the $\vec{\gamma}$-axis which corresponds to an equilibrium. However, all initializations within $\mathcal{S}$ may not preserve the trajectory within $\mathcal{S}$ due to a lack of guarantee on how $\vec{\gamma}$ behaves -- as illustrated by the dashed trajectory. We identify a smaller ball within $\mathcal{S}$ such that for any intitialization within that ball, $\vec{\gamma}$ is well-behaved and consquently ensures exponential convergence of $\vec{\theta}$.}
\label{fig:convergence}
\end{figure}

Now, if the trajectory indeed always remained in $\mathcal{S}$, we know that $\| \vec{\theta}(t) \| = \| \vec{\theta}(0)\|\exp \left(- c'' t\right)$ for some constant $c'' > 0$. Assume we initialize $\vec{\theta}(0)$ within  a radius of $\frac{c'' d}{2c'} $.
The rate at which $\vec{\gamma}$ changes at any point is,
\[
\| \dot{\vec{\gamma}}\| \leq c'  \| \vec{\theta}(0)\|\exp \left( -c'' t\right)
\]
Then, the maximum displacement in $\vec{\gamma}$ can be,
\[
\int\limits_{t=0}^{\infty} c'  \| \vec{\theta}(0)\|\exp \left( -c'' t\right) dt = c'\frac{\| \vec{\theta}(0)\| }{c''} \leq \frac{d}{2}
\]

Thus, the trajectory always lies in $\mathcal{S}$, which implies exponential convergence along $\vec{\theta}$ to a point where $\vec{\theta} =0$ and $\|\vec{\gamma}\| < \epsilon$. Thus the system exponentially converges to an equilibrium.

\end{proof}

\section{GANs are not concave-convex near equilibrium}
\label{app:convex-concave}
In this section, we consider a more general system than the one considered in the main paper to demonstrate that GANs are not concave-convex near equilibrium. In particular, consider the following discriminator and generator pair learning a distribution in 1-D:
\begin{align*}
D_{\vec{w}}(x) &= \sum_{i=0}^{d_D} w_{i} x^i \\
G_{\vec{a}}(z) &= \sum_{j=0}^{d_G} a_{j} z^j \\  
\end{align*}
where $d_D \geq 1$ and $d_G \geq 1$. 
Let the distribution to be learned be arbitrary. Let the latent distribution be the standard normal. Then, the gradient of the objective with respect to the generator parameters is:

\begin{align*}
\frac{\partial V(G,D)}{\partial a_j} = -\mathbb{E}_{z \sim \mathcal{N}(0,1)} \left[f'\left(-\sum_{i=0}^{d_D} w_{i} (G_{\vec{a}}(z))^i\right)  \cdot\left( \sum_{i=1}^{d_D} i w_{i} (G_{\vec{a}}(z))^{{i}-1}\right) \cdot z^{j}\right] 
\end{align*}

The second derivative is,
\begin{align*}
\frac{\partial^2 V(G,D)}{\partial a_j^2} &= - \mathbb{E}_{z \sim \mathcal{N}(0,1)} \left[f'\left(-\sum_{i=0}^{d_D} w_{i} (G_{\vec{a}}(z))^i\right)  \cdot\left( \sum_{i=2}^{d_D} i ({i}-1)w_{i} (G_{\vec{a}}(z))^{{i}-2}\right) \cdot z^{2j}\right]   \\
& + \mathbb{E}_{z \sim \mathcal{N}(0,1)} \left[f''\left(-\sum_{i=0}^{d_D} w_{i} (G_{\vec{a}}(z))^i\right)  \cdot \left(\left( \sum_{i=1}^{d_D} i w_{i} (G_{\vec{a}}(z))^{{i}-1}\right) \cdot z^{j} \right)^2\right] 
\end{align*}

Now, consider the case where $f''(x) < 0$. For points in the discriminator parameter space where  $w_1 \neq 0$ but $w_i = 0$ for all $i\neq 1$, the term above simplifies to the following when $j \neq 1$:

\[
 \mathbb{E}_{z \sim \mathcal{N}(0,1)} \left[f''\left(-w_{1} (G_{\vec{a}}(z))\right)  \cdot \left( w_1  z^{j} \right)^2\right] 
\] 
 which is clearly negative i.e., the objective is concave in most of the generator parameters, and this holds for parameters arbitrarily close to the all-zero discriminator parameter (as $w_1 \to 0$).
 
  On the other hand, consider the case where $f''(x) = 0$ for all $x\in\mathbb{R}$. Then,  if $d_D > 2$, we can consider $w_{2} \neq 0$ while $w_i =0$  for all $i \neq 2$. In this case, the second derivative simplifies to:
 
 \[
-\mathbb{E}_{z \sim \mathcal{N}(0,1)} \left[f'\left(- w_{2}(G^2_{\vec{a}}(z))\right) 2w_2  z^{2j}\right]  .
\]

If $f'(x) > 0$ for all $x$ (which is true in the case of WGANs), then in the region $w_2 >0$ the above term is negative i.e., the GAN objective is concave in terms of the generator parameters.

\section{Local exponential stability of GANs}
\label{app:general-stability}

In this section, we provide the full proof for our result about the local stability of GANs through the following lemmas. First, we derive the Jacobian at equilibrium.

\begin{lemma}
\label{lem:jacobian}
For the dynamical system defined by the GAN objective in
Equation~\ref{eq:generic_gan} and the updates in
Equation~\ref{eq:undamped_updates}, the Jacobian at 
an equilibrium point  
$(\vec{\theta^\star_D},\vec{\theta^\star_G})$, under the Assumptions~\ref{as:global-gen} and ~\ref{as:same-support} is:
\[
\vec{J}= 
\begin{bmatrix}
\vec{J}_{DD}& \vec{J}_{DG} \\
-\vec{J}_{DG}^T & \vec{J}_{GG} \\
\end{bmatrix} = 
\begin{bmatrix}
2f''(0) \vec{K}_{DD} &f'(0)
\vec{K}_{DG} \\ 
-f'(0) \vec{K}_{DG} ^T & 0 \\
\end{bmatrix} 
\]
where \[\vec{K}_{DD} \triangleq \left. \mathbb{E}_{p_{\rm data}} [(\nabla_{\vec{\theta_D}}
    D_{\vec{\theta_D}}(x)) (\nabla_{\vec{\theta_D}} D_{\vec{\theta_D}}(x))^T] \right\vert_{\vec{\theta^\star_D}} \succeq 0\]  and 
 \[\vec{K}_{DG} \triangleq\left. \int_{\mathcal{X}} \nabla_{\vec{\theta_D}} D_{\vec{\theta_D}}(x)  \nabla^T_{\vec{\theta_G}} p_{\vec{\theta_G}}(x)  dx\right\vert_{\vec{\theta_D} = \vec{\theta^\star_D}, \vec{\theta_G} = \vec{\theta^\star_G}} \] 
\end{lemma}

\begin{proof}
 To derive the Jacobian, we begin with a subtly different algebraic form of the GAN objective in Equation~\ref{eq:generic_gan} by replacing the term $  \mathbb{E}_{z\sim p_{\mathrm{latent}}}[f(-D_{\vec{\theta_D}}(G_{\vec{\theta_G}}(z)))]$ with $\mathbb{E}_{p_{\vec{\theta_G}}}[f(-D_{\vec{\theta_D}}(x))] = \int_{\mathcal{X}} p_{\vec{\theta_G}}(x) f(-D_{\vec{\theta_D}}(x))$. Effectively, we separate the discriminator and the generator's effects in this term. This is crucial because we will proceed with all of our analysis in this form.  Observe that the system then becomes,

 \begin{align*}
V(D_{\vec{\theta_D}},G_{\vec{\theta_G}}) & = \mathbb{E}_{p_{\rm data}}[f(D_{\vec{\theta_D}}(x))] + \mathbb{E}_{p_{\vec{\theta_G}}}[f(-D_{\vec{\theta_D}}(x))] \\
\vec{\dot{\theta}_D} & = \mathbb{E}_{p_{\rm data}}[f'(D_{\vec{\theta_D}}(x)) \nabla_{\vec{\theta_D}} D_{\vec{\theta_D}}(x)] - \mathbb{E}_{p_{\vec{\theta_G}}}[f'(-D_{\vec{\theta_D}}(x)) \nabla_{\vec{\theta_D}} D_{\vec{\theta_D}}(x)] \\
\vec{\dot{\theta}_G} &  =- \int_{\mathcal{X}} \nabla_{\vec{\theta_G}} p_{\vec{\theta_G}}(x) f(-D_{\vec{\theta_D}}(x)) dx
\end{align*}

Throughout this paper we will use the notation $\nabla^T (\cdot)$ to denote the row vector corresponding to the gradient that is being computed. Now, let $n_D$ be the number of discriminator parameters and $n_G$  the number of generator parameters. Then the first $n_D \times n_D$ block in $\vec{J}$, which we will denote by $\vec{J}_{DD}$ is:

\begin{align*}
\vec{J}_{DD} & \triangleq \left. \nabla^2_{\vec{\theta_D}} V(G_{\vec{\theta_G}}, D_{\vec{\theta_D}}) \right\vert_{(\vec{\theta^\star_D},\vec{\theta^\star_G})} =  \left. \frac{\partial  \vec{\dot{\theta}_D}}{\partial \vec{\theta_D}}  \right\vert_{\vec{\theta_D} = \vec{\theta^\star_D}, \vec{\theta_G} = \vec{\theta^\star_G}}  = \left. \frac{\partial \left. \vec{\dot{\theta}_D} \right\vert_{\vec{\theta_G} = \vec{\theta^\star_G}} }{\partial \vec{\theta_D}} \right\vert_{\vec{\theta_D} = \vec{\theta^\star_D}}   \\
& = \left. \frac{\partial \left( \mathbb{E}_{p_{\rm data}}[f'(D_{\vec{\theta_D}}(x)) \nabla_{\vec{\theta_D}} D_{\vec{\theta_D}}(x)] - \mathbb{E}_{p_{\rm data}}[f'(-D_{\vec{\theta_D}}(x)) \nabla_{\vec{\theta_D}} D_{\vec{\theta_D}}(x)]\right) }{\partial \vec{\theta_D}} \right\vert_{\vec{\theta_D} = \vec{\theta^\star_D}}     \\  
& =  \left. \left( \mathbb{E}_{p_{\rm data}}\left[f''(D_{\vec{\theta_D}}(x)) \nabla_{\vec{\theta_D}} D_{\vec{\theta_D}}(x)  \nabla^T_{\vec{\theta_D}} D_{\vec{\theta_D}}(x) \right] + \mathbb{E}_{p_{\rm data}}\left[f'(D_{\vec{\theta_D}}(x)) \nabla^2_{\vec{\theta_D}} D_{\vec{\theta_D}}(x)  \right]
 \right)   \right\vert_{\vec{\theta_D} = \vec{\theta^\star_D}} \\
& + \left. \left( \mathbb{E}_{p_{\rm data}}\left[f''(-D_{\vec{\theta_D}}(x)) \nabla_{\vec{\theta_D}} D_{\vec{\theta_D}}(x)  \nabla^T_{\vec{\theta_D}} D_{\vec{\theta_D}}(x) \right] - \mathbb{E}_{p_{\rm data}}\left[f'(-D_{\vec{\theta_D}}(x)) \nabla^2_{\vec{\theta_D}} D_{\vec{\theta_D}}(x)  \right] \right)   \right\vert_{\vec{\theta_D} = \vec{\theta^\star_D}} \\
&  =\left. \left( \mathbb{E}_{p_{\rm data}}\left[f''(0) \nabla_{\vec{\theta_D}} D_{\vec{\theta_D}}(x)  \nabla^T_{\vec{\theta_D}} D_{\vec{\theta_D}}(x) \right] + \mathbb{E}_{p_{\rm data}}\left[f'(0) \nabla^2_{\vec{\theta_D}} D_{\vec{\theta_D}}(x)  \right] \right)   \right\vert_{\vec{\theta_D} = \vec{\theta^\star_D}} \\
& + \left. \left( \mathbb{E}_{p_{\rm data}}\left[f''(0) \nabla_{\vec{\theta_D}} D_{\vec{\theta_D}}(x)  \nabla^T_{\vec{\theta_D}} D_{\vec{\theta_D}}(x) \right] - \mathbb{E}_{p_{\rm data}}\left[f'(0) \nabla^2_{\vec{\theta_D}} D_{\vec{\theta_D}}(x)  \right] \right)   \right\vert_{\vec{\theta_D} = \vec{\theta^\star_D}} \\
&  =2 f''(0) \left. \mathbb{E}_{p_{\rm data}}\left[ \nabla_{\vec{\theta_D}} D_{\vec{\theta_D}}(x)  \nabla^T_{\vec{\theta_D}} D_{\vec{\theta_D}}(x) \right]   \right\vert_{\vec{\theta_D} = \vec{\theta^\star_D}} \\
\end{align*}

The  subsequent $n_{D} \times n_G$ matrix, which we will denote by $\vec{J}_{DG}$ is:
\begin{align*}
\vec{J}_{DG} & \triangleq \left. \der{\vec{\theta_G}} {\nabla_{\vec{\theta_D}} V(G_{\vec{\theta_G}}, D_{\vec{\theta_D}})} \right\vert_{(\vec{\theta^\star_D},\vec{\theta^\star_G})} = \left. \frac{\partial  \vec{\dot{\theta}_D}}{\partial \vec{\theta_G}}  \right\vert_{\vec{\theta_D} = \vec{\theta^\star_D}, \vec{\theta_G} = \vec{\theta^\star_G}}  = \left. \frac{\partial \left. \vec{\dot{\theta}_D} \right\vert_{\vec{\theta_D} = \vec{\theta^\star_D}} }{\partial \vec{\theta_G}} \right\vert_{\vec{\theta_G} = \vec{\theta^\star_G}}   \\
&=  \left. \frac{\partial   }{\partial \vec{\theta_G}} \mathbb{E}_{p_{\vec{\theta_G}}} [f'(0) \nabla_{\vec{\theta_D}} D_{\vec{\theta_D}}(x) ]  \right\vert_{\vec{\theta_D} = \vec{\theta^\star_D}, \vec{\theta_G} = \vec{\theta^\star_G}}   \\
&= f'(0) \left. \int_{\mathcal{X}} \nabla_{\vec{\theta_D}} D_{\vec{\theta_D}}(x)  \nabla_{\vec{\theta_G}}^T p_{\vec{\theta_G}}(x)  dx\right\vert_{\vec{\theta_D} = \vec{\theta^\star_D}, \vec{\theta_G} = \vec{\theta^\star_G}} = f'(0)\vec{K}_{DG}  \\
\end{align*}

It is easy to see that the lower $n_G \times n_D$ matrix is $-\vec{J}_{DG}^T$:
\begin{align*}
 \left. \frac{\partial  \vec{\dot{\theta}_G}}{\partial \vec{\theta_D}}  \right\vert_{\vec{\theta_D}  = \vec{\theta^\star_D}, \vec{\theta_G} = \vec{\theta^\star_G}}  &= \left. \frac{\partial \left. \vec{\dot{\theta}_G} \right\vert_{\vec{\theta_G} = \vec{\theta^\star_G}} }{\partial \vec{\theta_D}} \right\vert_{\vec{\theta_D} = \vec{\theta^\star_D}}   \\
&=  - \left. \frac{\partial   }{\partial \vec{\theta_D}}  \int_{\mathcal{X}} f(D_{\vec{\theta_D}}(x)) \nabla_{\vec{\theta_G}} p_{\vec{\theta_G}}(x)  dx  \right\vert_{\vec{\theta_D} = \vec{\theta^\star_D}, \vec{\theta_G} = \vec{\theta^\star_G}}   = - \vec{J}_{DG}^T\\ 
\end{align*}

Furthermore, the lower $n_{G} \times n_{G}$ matrix $\vec{J}_{GG}$ turns out to be zero. Here, we will use an implication of Assumption~\ref{as:same-support}. More specifically, generators $\vec{\theta_G}$  that are within a sufficiently small radius $\epsilon_G$ around the equilibrium have the same support and therefore  i) $D_{\vec{\theta^\star_D}}(x) = 0$ for $x$ in this support. Furthermore for all generators within a radius $\epsilon_G/2$, any perturbation of the generator is not going to change the support, and therefore  ii) $\nabla_{\vec{\theta_G}} p_{\vec{\theta_G}}(x) = 0$  for $x$ that is not in this support. \footnote{ We can consider only $\epsilon_G/2$ perturbations and not $\epsilon_G$ perturbations because for $\vec{\theta_G}$ that is $\epsilon_G$ away from $\vec{\theta^\star_G}$, perturbing it a little further might potentially change its support as a result of which $\nabla_{\vec{\theta_G}} p_{\vec{\theta_G}}(x)$ may not necessarily be zero for all $x \notin {\rm supp} (p_{\vec{\theta_G^\star}})$}

Now, to show that $\vec{J}_{GG}$ is zero, we take any vector $\vec{v}$ that is a perturbation in the generator space and show that $\vec{v}^T \vec{J}_{GG} = 0$. Here, we will use the limit definition of the derivative along a particular direction $\vec{v}$. 

\begin{align*}
 \vec{v}^T \left. \frac{\partial  \vec{\dot{\theta}_G}}{\partial \vec{\theta_G}}  \right\vert_{\vec{\theta_D}  =  \vec{\theta^\star_D}, \vec{\theta_G} =   \vec{\theta^\star_G}}  &= \vec{v}^T\left. \frac{\partial \left. \vec{\dot{\theta}_G} \right\vert_{\vec{\theta_D} =  \vec{\theta^\star_D}} }{\partial \vec{\theta_G}} \right\vert_{\vec{\theta_G} = \vec{\theta^\star_G}}
  =- \lim_{\substack{\vec{\theta_G}- \vec{\theta^\star_G}= \epsilon \vec{v} \\ \epsilon \to 0}} \frac{ \int_{\mathcal{X}} f(-D_{\vec{\theta^\star_D}}(x)) \overbrace{\nabla^T_{\vec{\theta_G}} p_{\vec{\theta_G}}(x)}^{\text{$0$ for $x\notin {\rm supp}(p_{\vec{\theta}^\star_G})$}}  dx  }{ \epsilon}   \\
&=-  \lim_{\substack{\vec{\theta_G}- \vec{\theta^\star_G}= \epsilon \vec{v} \\ \epsilon \to 0}} \frac{ \int_{\rm supp (p_{\vec{\theta^\star_G}})} f(-\overbrace{D_{\vec{\theta^\star_D}}(x)}^{0})  \nabla^T_{\vec{\theta_G}} p_{\vec{\theta_G}}(x)  dx  }{\epsilon} \\
& =  - f(0) \lim_{\substack{\vec{\theta_G}- \vec{\theta^\star_G}= \epsilon \vec{v} \\ \epsilon \to 0}} \frac{  \nabla^T_{\vec{\theta_G}} \int_{\rm supp (p_{\vec{\theta^\star_G}})} p_{\vec{\theta_G}}(x)  dx  }{\epsilon}  \\
& =   - f(0) \lim_{\substack{\vec{\theta_G}- \vec{\theta^\star_G}= \epsilon \vec{v} \\ \epsilon \to 0}}  \frac{  \nabla^T_{\vec{\theta_G}} 1  }{\epsilon}  = 0 \\
\end{align*}

\end{proof}

To prove that the system is stable we will need to show that this matrix is Hurwitz. 
We show later in Lemma~\ref{lem:undamped-bound} that when i) $\vec{J}_{DD} \prec 0$ and furthermore ii) $\vec{J}_{DG}$ is full column rank, then $\vec{J}$ is indeed Hurwitz. However from $f''(0) < 0$, we only have that $\vec{J}_{DD} \preceq 0$. For these two conditions to be met, we will need $\vec{K}_{DD}$ and $\vec{K}_{DG}^T \vec{K}_{DG}$ to be full rank, which  you may recall from our discussion in the main paper below Assumption~\ref{as:convexity}, is met only when there is a unique equilibrim locally. 

 
Now, we show why this is the case -- by establishing a relation between the matrices $\vec{K}_{DD}$ and $\vec{K}_{DG}$ and the curvature of functions in  Assumption~\ref{as:convexity} -- and further show how the null spaces of these matrices correspond to a subspace of equilibria. Then, we show in Lemma~\ref{lem:projection}, how to consider a rotation of the system and project to a space that is orthogonal to this subspace of equilibria. Then from the Theorem~\ref{thm:multiple-equilibria} that we have proved in Appendix~\ref{app:lyapunov}, it is sufficient to show that the Jacobian of the projected system is Hurwitz.



In the following discussion, we will use the term ``equilibrium discriminator'' to denote a discriminator that is identically zero on the support and ``equilibrium generator'' to denote a generator that matches the true distribution, as defined in Assumption~\ref{as:global-gen}. Note that for an equilibrium discriminator, the generator updates are zero and vice versa for an equilibrium generator.

\begin{lemma}
\label{lem:eqspace}
For the dynamical system defined by the GAN objective in
Equation~\ref{eq:generic_gan} and the updates in
Equation~\ref{eq:undamped_updates},
under Assumptions~\ref{as:global-gen} and ~\ref{as:convexity}, there exists $\epsilon_D, \epsilon_G > 0$ such that for all $\epsilon_D' \leq \epsilon_D$ and $\epsilon_G' \leq \epsilon_G$, and for any unit vectors  $\vec{u} \in \N(\vec{K}_{DD}), \vec{v} \in \N(\vec{K}_{DG})$,
 $(\vec{\theta_D^\star} + \epsilon_D' \vec{u}, \vec{\theta_G^\star} + \epsilon_G' \vec{v})$ is an equilibrium point as defined in Assumption~\ref{as:global-gen}.
 \end{lemma}

\begin{proof}

Note that $2\vec{K}_{DD}$ is the Hessian of the function $\mathbb{E}_{p_{\rm data}}[D^2_{\vec{\theta_D}} (x)]$ at equilibrium:

\begin{align*}
\left.\nabla^2_{\vec{\theta_D}}   \mathbb{E}_{p_{\rm data}} [D^2_{\vec{\theta_D}}(x)] \right\vert_{\vec{\theta^\star_D}}& =   2 \left.\der{\vec{\theta_D}}   {\mathbb{E}_{p_{\rm data}} [D_{\vec{\theta_D}}(x) \nabla_{\vec{\theta_D}} D(x)]} \right\vert_{\vec{\theta^\star_D}}  \\
& =   2 \left.\left( \mathbb{E}_{p_{\rm data}} [\nabla_{\vec{\theta_D}} D_{\vec{\theta_D}}(x) \nabla^T_{\vec{\theta_D}} D(x)]+\mathbb{E}_{p_{\rm data}} [\underbrace{D_{\vec{\theta_D}}(x)}_{\text{0 at eqbm}} \nabla^2_{\vec{\theta_D}} D(x)] \right) \right\vert_{\vec{\theta^\star_D}}  \\
& =   2 \left.\left( \mathbb{E}_{p_{\rm data}} [\nabla_{\vec{\theta_D}} {D_{\vec{\theta_D}}(x)} \nabla^T_{\vec{\theta_D}} D(x)] \right) \right\vert_{\vec{\theta^\star_D}}  = 2\vec{K}_{DD}\\
 \end{align*}

Then, by Assumption~\ref{as:convexity}, $\mathbb{E}_{p_{\rm data}}[D^2_{\vec{\theta_D}} (x)]$ is locally constant along any unit vector $\vec{u} \in \N(\vec{K}_{DD})$. That is, 
 for sufficiently small $\epsilon$, if $\vec{\theta_D} = \vec{\theta^\star_D} + \epsilon  \vec{u}$, $\mathbb{E}_{p_{\rm data}}[D^2_{\vec{\theta_D}} (x)]$ equals the value of the function at equilibrium, which is $0$ because $D_{\vec{\theta_D^\star}(x)} = 0$ (according to Assumption~\ref{as:global-gen}). Thus, we can conclude that for all $x$ in the support of $p_{\rm data}$, $D_{\vec{\theta_D}} (x) = 0$. Then, the generator update is zero, because

\[\vec{\dot{\theta}_G} =  - f(0)\int_{\rm{supp}(p_{\rm data})} \nabla_{\vec{\theta_G}} p_{\vec{\theta_G}}(x) dx =  - f(0) \nabla_{\vec{\theta_G}}\int_{\rm{supp}(p_{\rm data})}  p_{\vec{\theta_G}}(x) dx   =  - f(0) \nabla_{\vec{\theta_G}} 1 = 0. \]

 In other words, $\vec{\theta_D}$ is an equilibrium discriminator which when paired with any generator results in zero updates on the generator.

 Similarly, $2\vec{K}_{DG}^T \vec{K}_{DG}$ is the Hessian of the function $\left\| \mathbb{E}_{p_{\rm data}}[ \nabla_{\vec{\theta_D}} D_{\vec{\theta_D}} (x)  ]   -  \mathbb{E}_{p_{\vec{\theta_G}}}[ \nabla_{\vec{\theta_D}} D_{\vec{\theta_D}} (x)  ]   \right\|^2$ at equilibrium:

 \begin{align*}
 &\nabla_{\vec{\theta_G}}   \left\| \mathbb{E}_{p_{\rm data}}[ \nabla_{\vec{\theta_D}} D_{\vec{\theta_D}} (x)  ]   -  \mathbb{E}_{p_{\vec{\theta_G}}}[ \nabla_{\vec{\theta_D}} D_{\vec{\theta_D}} (x)  ]   \right\|^2 \\
 & = -2 \left( \int_{\mathcal{X}} \nabla_{\vec{\theta_G}} p_{\vec{\theta_G}}(x) \nabla_{\vec{\theta_D}} D_{\vec{\theta_D}} (x)  dx\right)  \left(\mathbb{E}_{p_{\rm data}}[ \nabla_{\vec{\theta_D}} D_{\vec{\theta_D}} (x)  ]   -  \mathbb{E}_{p_{\vec{\theta_G}}}[ \nabla_{\vec{\theta_D}} D_{\vec{\theta_D}} (x) ] \right) \\
 \implies &\left.\nabla^2_{\vec{\theta_G}}   \left\| \mathbb{E}_{p_{\rm data}}[ \nabla_{\vec{\theta_D}} D_{\vec{\theta_D}} (x)  ]   -  \mathbb{E}_{p_{\vec{\theta_G}}}[ \nabla_{\vec{\theta_D}} D_{\vec{\theta_D}} (x)  ]   \right\|^2\right\vert_{\vec{\theta^\star_D}, \vec{\theta^\star_G}} \\
 & =  \left. 2\left(\int_{\mathcal{X}} \nabla_{\vec{\theta_G}} p_{\vec{\theta_G}}(x) \nabla_{\vec{\theta_D}} D_{\vec{\theta_D}} (x)  dx  \right)\left(\int_{\mathcal{X}} \nabla_{\vec{\theta_G}} p_{\vec{\theta_G}}(x) \nabla_{\vec{\theta_D}} D_{\vec{\theta_D}} (x)  dx\right)^{T} \right\vert_{\vec{\theta_D^\star},\vec{\theta_G^\star}}\\
 & - 2\left. \left(\underbrace{\mathbb{E}_{p_{\rm data}}[ \nabla_{\vec{\theta_D}} D_{\vec{\theta_D}} (x)  ]   -  \mathbb{E}_{p_{\vec{\theta_G}}}[ \nabla_{\vec{\theta_D}} D_{\vec{\theta_D}} (x) ] }_{\text{0 at eqbm}}\right)^T \int_{\mathcal{X}}  \nabla_{\vec{\theta_D}} D_{\vec{\theta_D}} (x)\nabla^2_{\vec{\theta_G}} p_{\vec{\theta_G}}(x)  dx \right\vert_{\vec{\theta_D^\star},\vec{\theta_G^\star}} \\
 &=2\vec{K}_{DG}^T \vec{K}_{DG}
 \end{align*}

Then, by Assumption~\ref{as:convexity}, $\left\| \mathbb{E}_{p_{\rm data}}[ \nabla_{\vec{\theta_D}} D_{\vec{\theta_D}} (x)  ]   -  \mathbb{E}_{p_{\vec{\theta_G}}}[ \nabla_{\vec{\theta_D}} D_{\vec{\theta_D}} (x)  ]   \right\|^2$  is locally constant along any unit vector $\vec{v} \in \N(\vec{K}_{DG})$. That is, 
 for sufficiently small $\epsilon'$, if $\vec{\theta_G} = \vec{\theta^\star_G} + \epsilon'  \vec{v}$, $\left\| \mathbb{E}_{p_{\rm data}}[ \nabla_{\vec{\theta_D}} D_{\vec{\theta_D}} (x)  ]   -  \mathbb{E}_{p_{\vec{\theta_G}}}[ \nabla_{\vec{\theta_D}} D_{\vec{\theta_D}} (x)  ]   \right\|^2$  equals the value of the function at equilibrium, which is $0$ because $p_{\vec{\theta_G^\star}} = p_{\rm data}$ (according to Assumption~\ref{as:global-gen}).

 Now, we can't immediately conclude that $\vec{\theta}_G$ corresponds to the true distribution. To show that, we first note that that at $(\vec{\theta_D^\star}, \vec{\theta_G})$, the discriminator update, whose magnitude is equal to $|f'(0)| \cdot \left\| \mathbb{E}_{p_{\rm data}}[ \nabla_{\vec{\theta_D}} D_{\vec{\theta_D}} (x)  ]   -  \mathbb{E}_{p_{\vec{\theta_G}}}[ \nabla_{\vec{\theta_D}} D_{\vec{\theta_D}} (x)  ]   \right\|$, is zero. However, as we have seen at $\vec{\theta_D^\star}$ the generator update is zero too.
Therefore, $(\vec{\theta_D^\star}, \vec{\theta_G})$ is an equilibrium point (both updates are zero) and from Assumption~\ref{as:global-gen} we can conclude that $p_{\vec{\theta_G}} = p_{\rm data}$. Thus, $\vec{\theta_G}$ is an equilibrium generator i.e., when paired with any equilibrium discriminator, the discriminator updates are zero.

In summary, for all slight perturbations along $\vec{u} \in \N(\vec{K}_{DD}), \vec{v} \in \N(\vec{K}_{DG})$ we have established that the discriminator and generator individually satisfy the requirements of an equilibrium discriminator and generator pair, and therefore the system is itself in equilibrium for these perturbations.
\end{proof}

Now, we show how to rotate and project the system to get a Hurwitz Jacobian matrix.

\begin{lemma}
\label{lem:projection}
For the dynamical system defined by the GAN objective in
Equation~\ref{eq:generic_gan} and the updates in
Equation~\ref{eq:undamped_updates},
consider the eigenvalue decompositions $\vec{K}_{DD} = \vec{U_D} \vec{\Lambda_D} \vec{U_{D}}^T$ and $\vec{K}_{DG}^T \vec{K}_{DG} = \vec{U_G} \vec{\Lambda_G} \vec{U_{G}}^T$. Let $\vec{U_D} = [\vec{T}_D^T, \vec{T}_D'^T]$ and $\vec{U_G} = [\vec{T}_G^T, \vec{T}_G'^T]$ such that $\C(\vec{T}_{D}'^T) = \N(\vec{K}_{DD})$ and $\C(\vec{T}_{G}'^T) = \N(\vec{K}_{DG})$. Consider the projections, $\vec{\gamma_D} = \vec{T}_D \vec{\theta}_D$ and $\vec{\gamma_G} = \vec{T}_G \vec{\theta}_G$. Then, the block in the Jacobian at equilibrium that corresponds to the projected system has the form:
\[
\vec{J}' = \begin{bmatrix}
\vec{J}_{DD}' & \vec{J}_{DG}' \\
-\vec{J}_{DG}'^T & 0 
\end{bmatrix} = 
\begin{bmatrix}
2f''(0) \vec{T}_D \vec{K}_{DD} \vec{T}_{D}^T &  f'(0) \vec{T}_D \vec{K}_{DG} \vec{T}_{G}^T \\
-f'(0)\vec{T}_{G}\vec{K}_{DG}^T \vec{T}_D^T & 0
\end{bmatrix}
\]
Under Assumption~\ref{as:loss}, we have that $\vec{J}_{DD}' \prec 0$ and $\vec{J}_{DG}'$ is full column rank.
\end{lemma}

\begin{proof}
Note that the columns of $\vec{U}_D$ and $\vec{U}_G$ correspond to eigenvectors, and furthermore, the rows of $\vec{T}_D'$ and $\vec{T}_G'$ are the eigenvectors that correspond to zero eigenvalues. These eigenvectors correspond to a local subspace of equilibria and 
 the above lemma considers a projection of the system to a space orthogonal to this subspace.

We first address a corner case where either $\vec{T}_{D}$ or $\vec{T}_{G}$ (the eigenvectors with non-zero eigenvalues) is empty. In the case that $\vec{T}_{D}$ is empty, it means that all discriminators in a neighborhood of the considered equilibrium are identically zero on the support of the true distribution (as proved in Lemma~\ref{lem:eqspace}). Then, for any generator, the discriminator update would be zero (because moving the discriminator in any direction locally does not result in a change in the objective). At the same time, the generator update would be zero too because these are all equilibrium discriminators. This means that the considered point is surrounded by a neighborhood of equilibria. Then, the system is trivially exponentially stable since any sufficiently close initialization is already at equilibrium. 

Similarly when $\vec{T}_G$ is empty it means that all generators in a small neighborhood have the same distribution, namely the true underlying distribution (as proved in Lemma~\ref{lem:eqspace}). Then, the generator update for any discriminator would be zero (changing the generator slightly in any direction does not change the generated distribution, and hence the objective). Furthermore, since these are equilibrium generators, the discriminator updates would be zero too, for any discriminator. Thus, again we are situated in a neighborhood of equilibria and the system is trivially exponentially stable.

Now we handle the general case. First note that, the Jacobian block of the projected variables must be

\begin{align*}
\left(
\begin{bmatrix}
{\vec{T_D}} \\
{\vec{T_G}} \\
\end{bmatrix}
\vec{J}
\begin{bmatrix}
{\vec{T_D}}^T &
{\vec{T_G}}^T 
\end{bmatrix} \right)
 = \begin{bmatrix}
2f''(0) \vec{T}_D \vec{K}_{DD} \vec{T}^T_{D} &  f'(0) \vec{T}_D \vec{K}_{DG} \vec{T}^T_{G} \\
-f'(0)\vec{T}_{G}\vec{K}^T_{DG} \vec{T}^T_D & 0
\end{bmatrix}
\end{align*}
where $\vec{J}$ is the Jacobian of the original system which we derived in Lemma~\ref{lem:jacobian}. Now note that, $\vec{T}_D \vec{K}_{DD} \vec{T}_{D}^T  =\vec{T}_D \vec{U_D} \vec{\Lambda_D} \vec{U_{D}}^T \vec{T}_{D}^T = \vec{\Lambda}^{(+)}_D $ which is a diagonal matrix with only the positive eigenvalues. Therefore, since $f''(0) < 0$, $\vec{J}_{DD}' \prec 0$.

Next, in a similar manner we can show that $\vec{T}_{G} \vec{K}_{DG}^T\vec{K}_{DG} \vec{T}_{G}^T = \vec{\Lambda}^{(+)}_G$, which is a diagonal matrix with only positive eigenvalues. Thus, $\vec{K}_{DG} \vec{T}_{G}^T $ is full column rank.  The non-trivial step here is to show that the matrix $ \vec{T}_D \vec{K}_{DG} \vec{T}_{G}^T $ which has fewer rows is full column rank too. This will follow if we showed that for any $\vec{u}$ such that $\vec{u}^T \vec{K}_{DD} = 0$, $\vec{u}^T\vec{K}_{DG} = 0$ too. That is, the left null space of $\vec{K}_{DD}$ is a subset of the left null space of $\vec{K}_{DG}$ and therefore projecting to the row span of $\vec{K}_{DD}$ does not hurt the row rank of $\vec{K}_{DG}$. 

To see why this is true, observe that from Lemma~\ref{lem:eqspace} for any small perturbation along such a $\vec{u}$, since we are always at an equilibrium discriminator i.e., $\vec{D}_{\vec{\theta_D}} (x) = 0$ for $x$ in the true support,  it must be that $\vec{u}^T\nabla_{\vec{\theta_D}}\vec{D}_{\vec{\theta_D}} (x) = 0$. Furthermore, recall from our derivation of the Jacobian that $\nabla{\vec{\theta_G}} p_{\vec{\theta_G}}(x) = 0$ for $x$ outside of this support.  Then,

\begin{align*}
\vec{u}^T \vec{K}_{DG} &= \left. \int_{\mathcal{X}} \underbrace{\vec{u}^T   \nabla_{\vec{\theta_D}} D_{\vec{\theta_D}}(x)}_{0 \text{ inside supp}}  \underbrace{\nabla^T_{\vec{\theta_G}} p_{\vec{\theta_G}}(x)}_{0 \text{ outside supp}}  dx\right\vert_{\vec{\theta_D} = \vec{\theta^\star_D}, \vec{\theta_G} = \vec{\theta^\star_G}}=0\\
\end{align*}
Therefore, since $f'(0) \neq 0$, this means $f'(0) \vec{T}_D \vec{K}_{DG} \vec{T}_{G}^T$ is full column rank.

\end{proof}

The main theorem then follows from the above lemmas.

\generalstability*

\begin{proof}
We have from Lemma~\ref{lem:eqspace} that the considered equilibrium point lies in a subspace of equilibria in a small neighborhood. Then, we have from Lemma~\ref{lem:projection} that the Jacobian block corresponding to the subspace orthogonal to this, satsifies properties from Lemma~\ref{lem:undamped-bound} which make it Hurwitz. We can then conclude exponential stability of the system from Theorem~\ref{thm:multiple-equilibria}.  The eigenvalue bounds presented in the theorem follow from Lemma~\ref{lem:undamped-bound}.
\end{proof}

Finally, we show that we can indeed find a Lyapunov function that satisfies LaSalle's principle for the projected linearized system. 

\begin{fact}
For the linearized projected system with the Jacobian $\vec{J}'$, we have that 
$1/2 \| \vec{\gamma}_D - \vec{\gamma^\star}_D \|^2 + 1/2 \| \vec{\gamma}_G - \vec{\gamma^\star}_G \|^2$ is a Lyapunov function such that for all non-equilbrium points, it either always decreases or only instantaneously remains constant.
\end{fact}

\begin{proof}
Note that the Lyapunov function is zero only at the equilibrium of the projected system. Furthermore, it is straightforward to verify that the rate at which this changes is given by $ f''(0)  (\vec{\gamma}_D - \vec{\gamma^\star}_D)^T \vec{T}_D^T \vec{K}_{DD} \vec{T}_D (\vec{\gamma}_D - \vec{\gamma^\star}_D)$. Observe that the generator terms have canceled out. Clearly this is zero only when $\vec{\gamma}_D = \vec{\gamma^\star}_D$  because $\vec{T}_D^T \vec{K}_{DD} \vec{T}_D$ is positive definite; otherwise it is strictly negative. Now, when this rate is indeed zero, we have that $\vec{\dot{\gamma}_D} = f'(0) \vec{T}_D \vec{K}_{DG} \vec{T}_{G}^T ( \vec{\gamma}_G - \vec{\gamma^\star}_G )$ because the other term in the update which is proportional to $\vec{K}_{DD} (\vec{\gamma}_D - \vec{\gamma^\star}_D)$ is zero. Now, again, this term is zero only when  $\vec{\gamma}_G = \vec{\gamma^\star}_G$ because $ \vec{T}_D \vec{K}_{DG} \vec{T}_{G}^T$ is full column rank. Thus, when we are not at equilibrium which means $\vec{\gamma}_G \neq \vec{\gamma^\star_G}$, the update on the discriminator parameters is nonzero i.e., $\vec{\dot{\gamma}_D} \neq 0$. In other words, it does not identically stay in the manifold $\vec{\gamma}_D=0$ on which the energy does not decrease.
\end{proof}

\subsection{Realizable case with a relaxed assumption}
\label{app:realizable-relaxed}

In this section, we will relax Assumption ~\ref{as:same-support} and prove stability under certain conditions. Specifically, recall that originally we required the equilibrium generator to share the same support with any perturbation of the generator. Now, we will allow the generator to have different supports when perturbed, and instead impose conditions on the discriminator. 

Our first condition is that the equilibrium discriminator must be zero not only on the support of $\vec{\theta}_G^\star$ but also on the supports of small perturbations of $\vec{\theta^\star_G}$. If this were not true, $\vec{\theta^\star_G}$ may not be at equilibrium as the slope of the discriminator function $D_{\vec{\theta_D^\star}}(x)$ may be non-zero at the boundaries of ${\rm supp}(p_{\vec{\theta}_G}^\star)$ in $\mathcal{X}$, thus potentially encouraging the generator to push data points away from the true support.
 
To motivate our second condition, recall from  Assumption~\ref{as:convexity}, we have that there could be directions along which we can perturb $\vec{\theta_D^\star}$, while ensuring that the discriminator still outputs zero on ${\rm supp}(p_{\vec{\theta^\star_G}})$. The intention behind allowing this was that these directions could allow other equivalent equilibrium discriminators in the neighborhood of $\vec{\theta_D^\star}$. However, under the relaxation of Assumption ~\ref{as:same-support} that we are now aiming for,  these perturbations will correspond to equilibrium discriminators only if they satsify the above condition i.e., that they are zero on the support of perturbations of $\vec{\theta^\star_G}$ too. We need to explicitly assume that this holds as we describe below. \footnote{Thanks to Lars Mescheder for identifying that such a condition was missing in earlier versions of this paper.}

 To state this assumption using the terminology we've developed so far, recall that imposing Property~\ref{prop:convex} on the function $\mathbb{E}_{p_{\rm data}}[D^2_{\vec{\theta_D}} (x)]$ at $\vec{\theta_D^\star}$ (where it attains its minimum of zero) implied that perturbations of $\vec{\theta^\star_D}$ along the flat directions of the function retains the property that the discriminator is zero on the support of $p_{\rm data}$ (i.e., $p_{\vec{\theta_G^\star}}$).  Extending this, we will assume that this property holds at $\vec{\theta^\star_D}$ for the functions $\mathbb{E}_{p_{\vec{\theta}_G}}[D^2_{\vec{\theta_D}} (x)]$ corresponding to every small perturbation $\vec{{\theta}_G}$ of $\vec{\theta_G^\star}$. Furthermore, the flat directions of all these functions must be identical so that perturbing $\vec{\theta_D^\star}$ along these directions guarantees that all these functions are zero. Then, the output of the perturbed discriminator would be zero on the support of all perturbations of $\vec{\theta^\star_G}$.

Formally, we can state  these  assumptions as follows:

\textbf{Assumption ~\ref{as:same-support}} (\textbf{Relaxed})
$\exists \epsilon_G, \epsilon_D > 0$  such that for all $\vec{\theta_G} \in B_{\epsilon_G}(\vec{\theta^\star_G})$:
\begin{enumerate}
	\item for all $x \in  {\rm supp} (p_{\vec{\theta_G}})$, $ D_{\vec{\theta_D^\star}}(x) = 0$.
	\item at $(\vec{\theta^\star_D}, \vec{\theta_G})$, the function $\mathbb{E}_{p_{\vec{\theta_G}}}[D^2_{\vec{\theta_D}} (x)]$ satisfies Property~\ref{prop:convex} in the discriminator space and furthermore, 
	$\Null\left(\left. \nabla^2_{\vec{\theta}_D} \mathbb{E}_{p_{\rm data}}[D^2_{\vec{\theta_D}} (x)]\right
	\vert_{\vec{\theta_D}=\vec{\theta_D^{\star}}}\right) = \Null\left(\left. \nabla^2_{\vec{\theta}_D} \mathbb{E}_{p_{\vec{\theta_G}}}[D^2_{\vec{\theta_D}} (x)]\right\vert_{\vec{\theta_D}=\vec{\theta_D^{\star}}}\right)$.
	\end{enumerate}


\subparagraph{Examples.} It is useful to illustrate simple examples that satisfy or break the two conditions above, for a clearer picture of what these assumptions imply. First, as an example that satisfies these conditions (and not the original Assumption~\ref{as:same-support}), consider a system where $p_{\rm data}$ is uniform over $[-1,1]$ ($\mathcal{X} = \mathbb{R}$), the generator is a uniform distribution over an interval parametrized as $[-\theta_G, \theta_G]$, and the discriminator is any polynomial, for example, a linear function $\theta_D x$. Note that at equilibrium $\theta_D=0$. Then, it can be verified that for this system the Hessian of $\mathbb{E}_{p_{\rm data}}[D^2_{\vec{\theta_D}} (x)]$  is positive definite at equilibrium, thus trivially satisfying the second assumption. 

As a simple example that breaks these assumptions, specifically condition (2) above\footnote{Thanks to Lars Mescheder for identifying this example.}, consider a system  where $p_{\rm data}$ is just a point mass at $0$ ($\mathcal{X} = \mathbb{R}$), the generator is also a point mass at $\theta_G$ and the discriminator is a linear function $\theta_D x$. Again, at equilibrium $\theta_D = 0$ and $\theta_G = 0$. Surprisingly, even though this is a unique equilibrium, the Hessian of $\mathbb{E}_{p_{\rm data}}[D^2_{\vec{\theta_D}} (x)]$  at equilibrium turns out to be zero. Thus, the null space of $\nabla^2_{\vec{\theta}_D} \mathbb{E}_{p_{\rm data}}[D^2_{\vec{\theta_D}} (x)]$ at equilibrium corresponds to the whole parameter space. On the other hand,  at equilibrium $\nabla^2_{\vec{\theta}_D} \mathbb{E}_{p_{\vec{\theta_G}}}[D^2_{\vec{\theta_D}} (x)] = 2\theta_G^2$, which is non-zero for any $\theta_G$ arbitrarily close to equilibrium. Thus, in the second condition above, while we have a null space for the first Hessian, there is no null space for the second Hessian, thereby breaking the condition. It can be shown that this system which breaks the condition is in fact not locally exponentially stable!\\

We now show that if these conditions hold, local exponentially stability holds too.

\begin{proof}
Most of the original proof holds as it is because all we needed was that the equilibrium discriminator be identically zero on the true support. We will prove only parts of the proof that required more than just this.

First, we extend Lemma~\ref{lem:eqspace} for this assumption. First, observe that any vector $\vec{u} \in \N(\vec{K_{DD}})$, also satisfies $\vec{u} \in \Null\left(\left. \nabla^2_{\vec{\theta}_D} \mathbb{E}_{p_{\vec{\theta_G}}}[D^2_{\vec{\theta_D}} (x)]\right\vert_{\vec{\theta_D}=\vec{\theta_D^{\star}}}\right)$  for all $\vec{\theta_G} \in B_{\epsilon}(\vec{\theta_G^\star})$ by the second condition in Assumption~\ref{as:same-support}. Then for any $\vec{\theta_D} = \vec{\theta_D^\star} + \epsilon \vec{u}$, $D_{\vec{\theta_D}}(x) = 0$ for all $x$ in the support of $p_{\vec{\theta_G}}$ where $\vec{\theta_G} \in B_{\epsilon}(\vec{\theta_G^\star})$. 
Then, we can show that any perturbation of the discriminator within the null space of $\vec{K}_{DD}$ is an `equilibrium discriminator' which when paired with any generator in small neighborhood around $\vec{\theta^\star_G}$, results in zero updates on the generator.  To prove this, recall that $\vec{\dot{\theta}_G}$ consists of two terms integrated over $\mathcal{X}$, $D_{\vec{\theta_D}}(x)$ and $\nabla_{\vec{\theta_G}} p_{\vec{\theta_G}}(x)$. In our previous proof under the original version of Assumption~\ref{as:same-support}, we used an intricate fact about these two terms. In particular, we said that for a generator within a radius of $\epsilon_G/2$ from equilibrium (where $\epsilon_G$ is as defined in the original version of Assumption~\ref{as:same-support}), i) the support of $p_{\vec{\theta_G}}$ is the same as $p_{\rm data}$ and therefore $D_{\vec{\theta}_D}(x) = 0$ for all $x$ in the true support and ii) for all $x$ not in the true support,  and for any generator $\vec{\theta_G} \in  B_{\epsilon_G/2}(\vec{\theta^\star_G})$,  $\nabla_{\vec{\theta_G}} p_{\vec{\theta_G}}(x) = 0$. 

In this case, we only have a weaker guarantee that for a generator within a perturbation of $\epsilon_G/2$ from $\vec{\theta_G^\star}$, the support is contained in the combined support $ \bigcup_{\vec{\theta_G} \in B_{\epsilon_G}(\vec{\theta^\star_G})} {\rm supp} (p_{\vec{\theta_G}})$. But then, i)
 for all $x$ in the combined support we have that  $D_{\vec{\theta_D}}(x) = 0$ and ii) for all $x$ not in the combined support and for any generator $\vec{\theta_G} \in  B_{\epsilon_G/2}(\vec{\theta^\star_G})$, $\nabla_{\vec{\theta_G}} p_{\vec{\theta_G}} (x)=0$. Then, the generator updates are:

\begin{align*}\vec{\dot{\theta}_G} = & - \int_{\mathcal{X}} f(-\underbrace{D_{\vec{\theta_D}}(x)}_{\substack{0 \text{ inside} \\ \text{combined supp}}}) \underbrace{\nabla_{\vec{\theta_G}} p_{\vec{\theta_G}}(x)}_{\substack{0 \text{ outside} \\ \text{combined supp}}} dx =  - f(0) \nabla_{\vec{\theta_G}}\int_{ \bigcup_{\vec{\theta_G} \in B_{\epsilon_G}(\vec{\theta^\star_G})} {\rm supp} (p_{\vec{\theta_G}})}  p_{\vec{\theta_G}}(x) dx \\
&  =  - f(0) \nabla_{\vec{\theta_G}} 1 = 0 \end{align*}

The second part of Lemma~\ref{lem:eqspace} holds similarly.

We need to make a similar argument to prove that the generator's Hessian $\vec{J}_{GG} = 0$ at equilibrium. 


\begin{align*}
 \vec{v}^T \left. \frac{\partial  \vec{\dot{\theta}_G}}{\partial \vec{\theta_G}}  \right\vert_{\vec{\theta_D}  =  \vec{\theta^\star_D}, \vec{\theta_G} =   \vec{\theta^\star_G}}  &= \vec{v}^T \left. \frac{\partial \left. \vec{\dot{\theta}_G} \right\vert_{\vec{\theta_D} = \vec{v}^T \vec{\theta^\star_D}} }{\partial \vec{\theta_G}} \right\vert_{\vec{\theta_G} = \vec{\theta^\star_G}} =-\lim_{\substack{\vec{\theta_G}- \vec{\theta^\star_G}= \epsilon \vec{v} \\ \epsilon \to 0}} \frac{ \int_{\mathcal{X}} 
 f(-\overbrace{D_{\vec{\theta^\star_D}}(x)}^{\substack{0 \text{ inside} \\ \text{combined supp}}})
 \overbrace{\nabla^T_{\vec{\theta_G}} p_{\vec{\theta_G}}(x)  }^{\substack{0 \text{ outside} \\ \text{combined supp}}}
   dx  }{ \epsilon}   \\
&=- f(0) \lim_{\substack{\vec{\theta_G}- \vec{\theta^\star_G}= \epsilon \vec{v} \\ \epsilon \to 0}} \frac{ \int_{ \bigcup_{\vec{\theta_G} \in B_{\epsilon_G}(\vec{\theta^\star_G})} {\rm supp} (p_{\vec{\theta_G}})} \nabla^T_{\vec{\theta_G}} p_{\vec{\theta_G}}(x)  dx  }{\epsilon} \\
& =  - f(0)\lim_{\substack{\vec{\theta_G}- \vec{\theta^\star_G}= \epsilon \vec{v} \\ \epsilon \to 0}} \frac{  \nabla^T_{\vec{\theta_G}} \int_{ \bigcup_{\vec{\theta_G} \in B_{\epsilon_G}(\vec{\theta^\star_G})} {\rm supp} (p_{\vec{\theta_G}})} p_{\vec{\theta_G}}(x)  dx  }{\epsilon}  \\
& =   - f(0) \lim_{\substack{\vec{\theta_G}- \vec{\theta^\star_G}= \epsilon \vec{v} \\ \epsilon \to 0}}  \frac{  \nabla^T_{\vec{\theta_G}} 1  }{\epsilon}  = 0 \\
\end{align*}
A similar modification of the proof can be done for Lemma~\ref{lem:projection} where we show that $\vec{T}_{D} \vec{K}_{DG} \vec{T}_{G}^T$ has the same column rank as $\vec{K}_{DG} \vec{T}_{G}^T$.
The rest of the proof follows as it did.
\end{proof}

\subsection{The non-realizable case}
\label{app:nonrealizable}
In this section, we extend our results about local stability of GANs to the case in which the true distribution can not be represented by any generator in the generator space. 
 While this is a hard problem in general,  we consider a specific case in which the discriminator is linear in its parameters and show that the system is locally stable at any equilibrium and its surrounding equilibria (none of which may correspond to the true distribution). More formally, consider a discriminator of the form:
\[
D_{\vec{\theta}_D}(x) = \vec{\theta}_D^T \vec{\upphi}(x)
\]

where $\vec{\upphi}$ is any feature mapping. For example, $\vec{\upphi}(x)$ could be a polynomial basis or the representation learned by a neural network (which we assume is not trained during the updates near equilibrium). Thus, the objective in this case is:
\[
V(D_{\vec{\theta_D}}, G_{\vec{\theta_G}})  = \mathbb{E}_{p_{\rm data}} [ f(\vec{\theta}_D^T \vec{\upphi}(x))] + \mathbb{E}_{p_{\vec{\theta_G}}} [ f(\vec{\theta}_D^T \vec{\upphi}(x))] 
\]


We consider a generator space that does not necessarily contain the true distribution, but however contains a generator $\vec{\theta^{\star}_G}$ that is an equilibrium point when paired with a discriminator that is zero on the support of the true data and the generated data. It must be noted that $\vec{\theta^\star_D}=\vec{0}$ is not necessarily the only equilibrium discriminator. Especially, if $\vec{\upphi}$ lies in a lower dimensional manifold, there could be a subspace of all-zero discriminators. Now, for such a generator to exist, we need:
\begin{align*}
\nabla_{\vec{\theta_D}} V(D_{\vec{\theta_D}}, G_{\vec{\theta_G}})  \vert_{(\vec{\theta^{\star}_D}, \vec{\theta^{\star}_G})}=  0 \\
\implies \mathbb{E}_{p_{\rm data}} [ \vec{\upphi}(\vec{x})] = \mathbb{E}_{p_{\vec{\theta^{\star}_G}}} [ \vec{\upphi}(\vec{x})]
\end{align*}

In other words, we want the means of the generated distribution and the true distribution in the representation $\vec{\upphi}$ to be identical. For a given generator space, this essentially is a restriction on the representation $\vec{\upphi}$ that has been learned/chosen for the discriminator. If $\vec{\upphi}$ was a richer representation that computes many higher order moments of the data, we may never find an equilibrium generator. 

We now prove Theorem~\ref{thm:general-stability} for the non-realizable case. 
Our main idea is identical to that of the proof in the realizable case. However, we need to be careful in a number of steps.  We first prove a result similar to Lemma~\ref{lem:jacobian} that derives the Jacobian of the system at equilibrium.

\begin{lemma}
\label{lem:jacobian-nonrealizable}
For the dynamical system defined by the GAN objective in
Equation~\ref{eq:generic_gan} and the updates in
Equation~\ref{eq:undamped_updates}, the Jacobian at 
an equilibrium point  
$(\vec{\theta^\star_D},\vec{\theta^\star_G})$, under the Assumptions~\ref{as:global-gen} (for the non-realizable case) and ~\ref{as:same-support} is:
\[
\vec{J}= 
\begin{bmatrix}
\vec{J}_{DD}& \vec{J}_{DG} \\
-\vec{J}_{DG}^T & \vec{J}_{GG} \\
\end{bmatrix} = 
\begin{bmatrix}
2f''(0) \vec{K}_{DD} &f'(0)
\vec{K}_{DG} \\ 
-f'(0) \vec{K}_{DG} ^T & 0 \\
\end{bmatrix} 
\]
where \[ 2\vec{K}_{DD} \triangleq \left.  \mathbb{E}_{p_{\rm data}} [(\nabla_{\vec{\theta_D}}
    D_{\vec{\theta_D}}(x)) (\nabla_{\vec{\theta_D}} D_{\vec{\theta_D}}(x))^T] + \mathbb{E}_{p_{\vec{\theta_G^\star}}} [(\nabla_{\vec{\theta_D}}
    D_{\vec{\theta_D}}(x)) (\nabla_{\vec{\theta_D}} D_{\vec{\theta_D}}(x))^T] \right\vert_{\vec{\theta^\star_D}} \succeq 0\]  and 
 \[\vec{K}_{DG} \triangleq\left. \int_{\mathcal{X}} \nabla_{\vec{\theta_D}} D_{\vec{\theta_D}}(x)  \nabla^T_{\vec{\theta_G}} p_{\vec{\theta_G}}(x)  dx\right\vert_{\vec{\theta_D} = \vec{\theta^\star_D}, \vec{\theta_G} = \vec{\theta^\star_G}}\]
\end{lemma}

\begin{proof}

Recall that, \begin{align*}
V(D_{\vec{\theta_D}},G_{\vec{\theta_G}}) & = \mathbb{E}_{p_{data}}[f(D_{\vec{\theta_D}}(x))] + \mathbb{E}_{p_{\vec{\theta_G}}}[f(-D_{\vec{\theta_D}}(x))] \\
\vec{\dot{\theta}_D} & = \mathbb{E}_{p_{data}}[f'(D_{\vec{\theta_D}}(x)) \nabla_{\vec{\theta_D}} D_{\vec{\theta_D}}(x)] - \mathbb{E}_{p_{\vec{\theta_G}}}[f'(-D_{\vec{\theta_D}}(x)) \nabla_{\vec{\theta_D}} D_{\vec{\theta_D}}(x)] \\
\vec{\dot{\theta}_G} & = - \int_{\mathcal{X}} \nabla_{\vec{\theta_G}}^T p_{\vec{\theta_G}} f(-D_{\vec{\theta_D}}(x)) dx
\end{align*}
First we show that  $\vec{J}_{DD}$ has a similar form which is still negative semi-definite when $f''(0) < 0$:

\begin{align*}
\vec{J}_{DD} &=  \left. \nabla^2_{\vec{\theta_D}} V(G_{\vec{\theta_G}}, D_{\vec{\theta_D}}) \right\vert_{(\vec{\theta^\star_D},\vec{\theta^\star_G})} =  \left. \frac{\partial  \vec{\dot{\theta}_D}}{\partial \vec{\theta_D}}  \right\vert_{\vec{\theta_D} = \vec{\theta^\star_D}, \vec{\theta_G} = \vec{\theta^\star_G}}  = \left. \frac{\partial \left. \vec{\dot{\theta}_D} \right\vert_{\vec{\theta_G} = \vec{\theta^\star_G}} }{\partial \vec{\theta_D}} \right\vert_{\vec{\theta_D} = \vec{\theta^\star_D}}   \\
& = \left. \frac{\partial \left( \mathbb{E}_{p_{data}}[f'(D_{\vec{\theta_D}}(x)) \nabla_{\vec{\theta_D}} D_{\vec{\theta_D}}(x)] - \mathbb{E}_{p_{\vec{\theta^\star_G}}}[f'(-D_{\vec{\theta_D}}(x)) \nabla_{\vec{\theta_D}} D_{\vec{\theta_D}}(x)]\right) }{\partial \vec{\theta_D}} \right\vert_{\vec{\theta_D} = \vec{\theta^\star_D}}     \\  
& =  \left. \left( \mathbb{E}_{p_{data}}\left[f''(D_{\vec{\theta_D}}(x)) \nabla_{\vec{\theta_D}} D_{\vec{\theta_D}}(x)  \nabla^T_{\vec{\theta_D}} D_{\vec{\theta_D}}(x) \right] + \mathbb{E}_{p_{data}}\left[f'(D_{\vec{\theta_D}}(x)) \nabla^2_{\vec{\theta_D}} D_{\vec{\theta_D}}(x)  \right]
 \right)   \right\vert_{\vec{\theta_D} = \vec{\theta^\star_D}} \\
& + \left. \left( \mathbb{E}_{p_{\vec{\theta^\star_G}}}\left[f''(-D_{\vec{\theta_D}}(x)) \nabla_{\vec{\theta_D}} D_{\vec{\theta_D}}(x)  \nabla^T_{\vec{\theta_D}} D_{\vec{\theta_D}}(x) \right] - \mathbb{E}_{p_{\vec{\theta^\star_G}}}\left[f'(-D_{\vec{\theta_D}}(x)) \nabla^2_{\vec{\theta_D}} D_{\vec{\theta_D}}(x)  \right] \right)   \right\vert_{\vec{\theta_D} = \vec{\theta^\star_D}} \\
&  =\left. \left( \mathbb{E}_{p_{\rm data}}\left[f''(0) \nabla_{\vec{\theta_D}} D_{\vec{\theta_D}}(x)  \nabla^T_{\vec{\theta_D}} D_{\vec{\theta_D}}(x) \right] + \mathbb{E}_{p_{\rm data}}\left[f'(0) \underbrace{\nabla^2_{\vec{\theta_D}} D_{\vec{\theta_D}}(x)}_{=0}  \right] \right)   \right\vert_{\vec{\theta_D} = \vec{\theta^\star_D}} \\
& + \left. \left( \mathbb{E}_{p_{\vec{\theta^\star_G}}}\left[f''(0) \nabla_{\vec{\theta_D}} D_{\vec{\theta_D}}(x)  \nabla^T_{\vec{\theta_D}} D_{\vec{\theta_D}}(x) \right] - \mathbb{E}_{p_{\vec{\theta^\star_G}}}\left[f'(0) \underbrace{\nabla^2_{\vec{\theta_D}} D_{\vec{\theta_D}}(x) }_{=0} \right] \right)   \right\vert_{\vec{\theta_D} = \vec{\theta^\star_D}} \\
&  = f''(0) \left. \left( \mathbb{E}_{p_{\rm data}}\left[ \nabla_{\vec{\theta_D}} D_{\vec{\theta_D}}(x)  \nabla^T_{\vec{\theta_D}} D_{\vec{\theta_D}}(x) \right] + \mathbb{E}_{p_{\vec{\theta^\star_D}}}\left[ \nabla_{\vec{\theta_D}} D_{\vec{\theta_D}}(x)  \nabla^T_{\vec{\theta_D}} D_{\vec{\theta_D}}(x) \right]  \right) \right\vert_{\vec{\theta_D} = \vec{\theta^\star_D}}  \\
& = 2 f''(0)\vec{K}_{DD}
\end{align*}

The most crucial step here is that we were able to ignore the terms corresponding to $ \nabla^2_{\vec{\theta_D}} D_{\vec{\theta_D}}(x) $ because the discriminator is linear in its parameters i.e., $\nabla_{\vec{\theta_D}} D_{\vec{\theta_D}}(x) = \vec{\upphi}(x) $ and thus the Hessian is zero. 

All other terms in the Jacobian are identical to the realizable case because we assume that at equilibrium the discriminator must be identically zero.

\end{proof}

Now, we again show that the equilibrium point  in consideration lies in a subspace of equilibria.

\begin{lemma}
\label{lem:eqspace-nonrealizable}  
Under Assumptions~\ref{as:global-gen} (Non-realizable), ~\ref{as:convexity}, and ~\ref{as:same-support} there exists $\epsilon_D, \epsilon_G > 0$ such that for all $\epsilon_D' \leq \epsilon_D$ and $\epsilon_G' \leq \epsilon_G$, and for any unit vectors  $\vec{u} \in \N(\vec{K}_{DD}), \vec{v} \in \N(\vec{K}_{DG})$,
 $(\vec{\theta_D^\star} + \epsilon_D' \vec{u}, \vec{\theta_G^\star} + \epsilon_G' \vec{v})$ is an equilibrium point.
 \end{lemma}

\begin{proof}
Our proof is only slightly different from that of Lemma~\ref{lem:eqspace}.
Note that $4\vec{K}_{DD}$ is the Hessian of the function $\mathbb{E}_{p_{\rm data}}[D^2_{\vec{\theta_D}} (x)] + \mathbb{E}_{p_{\vec{\theta_G^\star}}}[D^2_{\vec{\theta_D}} (x)]$ at equilibrium. 

Since this is the sum of two positive semi-definite matrices, any vector in the null space of $\N(\vec{K}_{DD})$ is also in the null space of the Hessian of $\mathbb{E}_{p_{\rm data}}[D^2_{\vec{\theta_D}} (x)]$. Then, by Assumption~\ref{as:convexity}, $\mathbb{E}_{p_{\rm data}}[D^2_{\vec{\theta_D}} (x)]$ is locally constant along any unit vector $\vec{u} \in \N(\vec{K}_{DD})$. That is, 
 for sufficiently small $\epsilon$, if $\vec{\theta_D} = \vec{\theta^\star_D} + \epsilon  \vec{u}$, $\mathbb{E}_{p_{\rm data}}[D^2_{\vec{\theta_D}} (x)]$ equals the value of the function at equilibrium, which is $0$ because $D_{\vec{\theta}_D^\star}(x) = 0$ (according to Assumption~\ref{as:global-gen}). Thus, we can conclude that for all $x$ in the support of $p_{\rm data}$, $D_{\vec{\theta_D}} (x) = 0$. 
Now from Assumption~\ref{as:same-support}, the support of generators in a small neighborhood is identical to the support of the true distribution, therefore these discriminators are equilibrium discriminators i.e., when paired with any generator, the generator updates are zero.

 Similarly, $2\vec{K}_{DG}^T \vec{K}_{DG}$ is the Hessian of the function $\left\| \mathbb{E}_{p_{\rm data}}[ \nabla_{\vec{\theta_D}} D_{\vec{\theta_D}} (x)  ]   -  \mathbb{E}_{p_{\vec{\theta_G}}}[ \nabla_{\vec{\theta_D}} D_{\vec{\theta_D}} (x)  ]   \right\|^2$ at equilibrium. 
Then, by Assumption~\ref{as:convexity}, $\left\| \mathbb{E}_{p_{\rm data}}[ \nabla_{\vec{\theta_D}} D_{\vec{\theta_D}} (x)  ]   -  \mathbb{E}_{p_{\vec{\theta_G}}}[ \nabla_{\vec{\theta_D}} D_{\vec{\theta_D}} (x)  ]   \right\|^2$  is locally constant along any unit vector $\vec{v} \in \N(\vec{K}_{DG})$. That is, 
 for sufficiently small $\epsilon'$, if $\vec{\theta_G} = \vec{\theta^\star_G} + \epsilon'  \vec{v}$, $\left\| \mathbb{E}_{p_{\rm data}}[ \nabla_{\vec{\theta_D}} D_{\vec{\theta_D}} (x)  ]   -  \mathbb{E}_{p_{\vec{\theta_G}}}[ \nabla_{\vec{\theta_D}} D_{\vec{\theta_D}} (x)  ]   \right\|^2$  equals the value of the function at equilibrium. 
 Now, since this function is proportional to the magnitude of the equilibrium discriminator's update, it equals zero at equilibrium. Now, observe that 
 \[
 \mathbb{E}_{p_{\rm data}}[ \nabla_{\vec{\theta_D}} D_{\vec{\theta_D}} (x)  ]   -  \mathbb{E}_{p_{\vec{\theta_G}}}[ \nabla_{\vec{\theta_D}} D_{\vec{\theta_D}} (x)  ] = \mathbb{E}_{p_{\rm data}}[ \vec{\upphi}(x)  ]   -  \mathbb{E}_{p_{\vec{\theta_G}}}[ \vec{\upphi}(x)   ] 
 \]

is independent of the discriminator variables (Here, we have used the fact that the discriminator is linear in its parameters.)
. This means that for these generators along $\vec{v}$, the discriminator update must be zero. In other words, these generators are equilibrium generators in the non-realizable sense, that their $\vec{\upphi}$ representation matches with the true distribution.


In summary, for all slight perturbations along $\vec{u} \in \N(\vec{K}_{DD}), \vec{v} \in \N(\vec{K}_{DG})$ we have established that the discriminator and generator individually satisfy the requirements of an equilibrium discriminator and generator pair, and therefore the system is itself is in equilibrium for these perturbations.
\end{proof}

It turns out that given these two lemmas, Lemma~\ref{lem:projection} follows as it did earlier, and therefore the main theorem follows too.

\section{Linear Quadratic GAN  -- Gaussian example}
\label{app:lqgan}
In order to illustrate our
assumptions in Theorem~\ref{thm:general-stability}, consider a simple GAN that learns an 
$n$-dimensional Gaussian distribution $\mathcal{N}(\vec{\upmu}, \vec{\Sigma})$, where $\vec{\Sigma} \succ 0$. Let the latent variable be drawn from the
standard normal, $\mathcal{N}(\vec{0},\vec{I}_n)$. Consider a quadratic
discriminator 
$D(\vec{x}) = \vec{x}^T \vec{W}_2 \vec{x} +
\vec{w}_1^T \vec{x}$, 
and a linear generator 
$G(z) = \vec{A}\vec{z} + \vec{b}$.  
We call the resulting system \LQ{} (linear-quadratic).  Let $\vec{\Sigma}^{1/2}$ be the unique
real positive definite matrix such that $\left(\vec{\Sigma}^{1/2}\right)^2 = \vec{\Sigma}$.
Then we have the following:

\begin{restatable}{theorem}{gaussianstability}
\label{thm:gaussian-stability}
In \LQ{}, $\vec{A} =\vec{\Sigma}^{1/2}, \vec{b} = \vec{\upmu}$ and $\vec{W}_2=0, \vec{w}_1 =0$ corresponds to an equilibrium  that is locally exponentially stable provided $f''(0) < 0$ and $f'(0) \neq 0$.
\end{restatable}


\begin{proof}
Since the system consists of parameters arranged in the form of matrices, we will need vectorization calculus \citep{magnus1995matrix} to arrange these parameters as a vector and differentiate them/with respect to them.

To verify that the given point is indeed an equilibrium, let us look at the GAN objective:
\begin{align*}
V(G,D) &=\mathbb{E}_{\vec{x} \sim \mathcal{N}(\vec{\upmu},\vec{\Sigma})} [f(\vec{x}^T \vec{W}_2 \vec{x} + \vec{w}_1^T \vec{x})] \\ &+ \mathbb{E}_{ \vec{z} \sim \mathcal{N}(\vec{0}, \vec{I}_n) } [f(-(\vec{A} \vec{z}+\vec{b})^T \vec{W}_2 (\vec{A} \vec{z}+\vec{b}) - \vec{w}_1^T(\vec{A} \vec{z}+\vec{b}))] \\
\end{align*}

The updates in Equation~\ref{eq:undamped_updates}  for \LQ{} can be written as  :
\begin{align*}
\dot{\vec{W}}_2  =& \mathbb{E}_{\vec{x} \sim \mathcal{N}(\vec{\upmu},\vec{\Sigma})} [\vec{x}\vec{ x}^Tf'(\vec{x}^T \vec{W}_2 \vec{x} + \vec{w}_1^T \vec{x} )] \\ &- \mathbb{E}_{ \vec{z} \sim \mathcal{N}(\vec{0}, \vec{I}_n) } [(\vec{A} \vec{z}+\vec{b})  (\vec{A} \vec{z}+\vec{b})^T f'(-(\vec{A} \vec{z}+\vec{b})^T \vec{W}_2 (\vec{A} \vec{z}+\vec{b}) - \vec{w}_1^T(\vec{A} \vec{z}+\vec{b}) )] \\
\dot{\vec{w}_1 } =& \mathbb{E}_{\vec{x} \sim \mathcal{N}(\vec{\upmu},\vec{\Sigma})} [\vec{x}f'(\vec{x}^T \vec{W}_2 \vec{x} + \vec{w}_1^T \vec{x} )] \\ &- \mathbb{E}_{ \vec{z} \sim \mathcal{N}(\vec{0}, \vec{I}_n) } [(\vec{A} \vec{z}+\vec{b})f'(-(\vec{A} \vec{z}+\vec{b})^T \vec{W}_2 (\vec{A} \vec{z}+\vec{b}) - \vec{w}_1^T(\vec{A} \vec{z}+\vec{b}) )] \\
\dot{\vec{A}} =&  \mathbb{E}_{ \vec{z} \sim \mathcal{N}(\vec{0}, \vec{I}_n) } [ ((\vec{W}_2 + \vec{W}_2^T)\vec{A}\vec{z} \vec{z}^T + (\vec{W}_2 + \vec{W}_2^T)\vec{b} \vec{z}^T + \vec{w}_1 \vec{z}^T) \\
& \; f'(-(\vec{A} \vec{z}+\vec{b})^T \vec{W}_2 (\vec{A} \vec{z}+\vec{b}) - \vec{w}_1^T(\vec{A} \vec{z}+\vec{b}) )] \\
\dot{\vec{b}} = &  \mathbb{E}_{ \vec{z} \sim \mathcal{N}(\vec{0}, \vec{I}_n) } [ ((\vec{W}_2 + \vec{W}_2^T)\vec{A}\vec{z} + (\vec{W}_2 + \vec{W}_2^T)\vec{b} + \vec{w}_1 ) \\
& \; f'(-(\vec{A} \vec{z}+\vec{b})^T \vec{W}_2 (\vec{A} \vec{z}+\vec{b}) - \vec{w}_1^T(\vec{A} \vec{z}+\vec{b}) )] \\
\end{align*}


Clearly, when $\vec{z} \sim \mathcal{N}(\vec{0}, I)$, we have that $\vec{\Sigma}^{1/2} \vec{z} + \vec{\upmu} \sim \mathcal{N}(\vec{\upmu},\vec{\Sigma})$, therefore at $\vec{A} = \vec{\Sigma}^{1/2}, \vec{b} = \vec{\upmu}$ and $\vec{W}_2=0, \vec{w}_1 =0$, all the above updates become zero, implying that it is an equilibrium for which the generator has converged to the true distribution.  To prove that it is locally stable, we need to examine the Jacobian at that point. Note that since the Jacobian is a matrix with one cell for each pair of discriminator-generator parameters, we need to calculate second-order derivatives after vectorizing the parameter matrices $\vec{Q}$ and $\vec{A}$. 
	
We first calculate the derivative of the discriminator updates with respect to the discriminator itself.

	\begin{align*}
\left.  \frac{ \partial vec(\dot{\vec{W}}_2 )}{\partial vec(\vec{W}_2)}   \right\rvert_{ \substack{ \vec{b} = \vec{\upmu},   \vec{W}_2=0,\\ \vec{w}_1=0, \vec{A} =\vec{\Sigma}^{1/2}}} = & \left . \frac{ \partial }{\partial vec(\vec{W}_2)}  \left( \left. vec(\dot{\vec{W}}_2 ) \right\rvert_{\vec{b}= \vec{\upmu},   \vec{A} =\vec{\Sigma}^{1/2}, \vec{w}_1=0}\right) \right\vert_{\vec{W}_2=0}\\
=& \left . \frac{ \partial }{\partial vec(\vec{W}_2) } \mathbb{E}_{\vec{x} \sim \mathcal{N}(\vec{\upmu}, \vec{\Sigma})}[vec(\vec{x}\vec{x}^T) (f'(\vec{x}^T \vec{W}_2 \vec{x})-f'(-\vec{x}^T \vec{W}_2 \vec{x}))]  \right\vert_{\vec{W}_2=0} \\
& = 2f''(0) \mathbb{E}_{\vec{x} \sim \mathcal{N}(\vec{\upmu}, \vec{\Sigma})}[ (\vec{x} \otimes \vec{x}) (\vec{x} \otimes \vec{x})^T] \\
\left.  \frac{ \partial vec(\dot{\vec{W}}_2 )}{\partial \vec{w}_1}   \right\rvert_{ \substack{ \vec{b} = \vec{\upmu},   \vec{W}_2=0,\\\vec{w}_1=0, \vec{A} =\vec{\Sigma}^{1/2}} }= & \left . \frac{ \partial }{\partial \vec{w}_1} \left( \left. vec(\dot{\vec{W}}_2 ) \right\rvert_{\vec{b}= \vec{\upmu},   \vec{A} =\vec{\Sigma}^{1/2}, \vec{W}_2=0}\right)  \right\vert_{\vec{w}_1=0} \\ 
& = \left. \frac{\partial }{\partial \vec{w}_1} \mathbb{E}_{\vec{x} \sim \mathcal{N}(\vec{\upmu}, \vec{\Sigma})}[ vec (\vec{x}\vec{x}^T) (f'(\vec{w}_1^T \vec{x})-f'(-\vec{w}_1^T \vec{x}))]   \right\vert_{\vec{w}_1=0} \\ \
&= 2f''(0) \mathbb{E}_{\vec{x} \sim \mathcal{N}(\vec{\upmu}, \vec{\Sigma})} [ (\vec{x} \otimes \vec{x})\vec{x}^T ]\\
\left.  \frac{ \partial \dot{\vec{w}_1}}{\partial \vec{w}_1}   \right\rvert_{ \substack{\vec{b} = \vec{\upmu},   \vec{W}_2=0,\\ \vec{w}_1=0, \vec{A} =\vec{\Sigma}^{1/2}}} = & \left . \frac{ \partial }{\partial \vec{w}_1} \left( \left. \dot{\vec{w}_1} \right\rvert_{\vec{b}= \vec{\upmu},   \vec{A} =\vec{\Sigma}^{1/2}, \vec{W}_2=0}\right)  \right\vert_{\vec{w}_1=0} \\ 
=&  \left . \frac{ \partial }{\partial \vec{w}_1} \mathbb{E}_{\vec{x} \sim \mathcal{N}(\vec{\upmu}, \vec{\Sigma})}[ \vec{x} (f'(\vec{w}_1^T \vec{x})-f'(-\vec{w}_1^T \vec{x})  )]   \right\vert_{\vec{w}_1=0} \\ 
=& 2f''(0)\mathbb{E}_{\vec{x} \sim \mathcal{N}(\vec{\upmu}, \vec{\Sigma})}[ \vec{x} \vec{x}^T ] \\  
\end{align*}

Then we calculate the derivative of the discriminator updates with respect to the generator parameters. Note that we will be using the constant matrix $\vec{T}_{n,n}$ which is a matrix of zeros and ones defined in vectorization algebra; this matrix is the vectorization equivalent of the transpose operator. That is, for any square matrix $\vec{V} \in \mathbb{R}^n$, $\vec{T}_{n,n} vec(\vec{V}) = vec(\vec{V}^T)$. 

\begin{align*}
	\left . \frac{ \partial vec(\dot{\vec{W}}_2 )}{\partial vec(\vec{A})} \right\rvert_{\substack{\vec{b} = \vec{\upmu},   \vec{W}_2=0,\\\vec{w}_1=0, \vec{A} =\vec{\Sigma}^{1/2}}} = & \left . \frac{ \partial }{\partial vec(\vec{A})}  \left( \left. vec(\dot{\vec{W}}_2 ) \right\rvert_{\vec{b}= \vec{\upmu},   \vec{W}_2=0, \vec{w}_1=0}\right) \right\vert_{\vec{A} =\vec{\Sigma}^{1/2}} \\ 
=& - \left. \frac{\partial}{\partial vec(\vec{A})} vec\left(\mathbb{E}_{ \vec{z} \sim \mathcal{N}(\vec{0}, \vec{I}_n) } [(\vec{A} \vec{z}+\vec{\upmu})  (\vec{A} \vec{z}+\vec{\upmu})^Tf'(0)] \right)  \right\vert_{\vec{A} =\vec{\Sigma}^{1/2}}\\
 = & - \left. \frac{\partial}{\partial vec(\vec{A})} vec\left(\mathbb{E}_{ \vec{z} \sim \mathcal{N}(\vec{0}, \vec{I}_n) } [ (\vec{A}\vec{z} \vec{z}^T \vec{A}^T + \vec{A}\vec{z} \vec{\upmu}^T + \vec{\upmu} \vec{z}^T {\vec{A}^T}) f'(0)] \right) \right\vert_{\vec{A} =\vec{\Sigma}^{1/2}} \\& \\ 
=&  - \left.\frac{\partial}{\partial vec(\vec{A})} vec(\vec{A}\vec{A}^T)f'(0)  \right\vert_{\vec{A} =\vec{\Sigma}^{1/2}} =  -(\vec{I}_{n^2}+\vec{T}_{n,n})(\vec{\Sigma}^{1/2} \times \vec{I}_n) f'(0)\\
 \left. \frac{\partial vec(\dot{\vec{W}}_2) }{\partial \vec{b}}  \right\rvert_{\substack{\vec{b} = \vec{\upmu},   \vec{W}_2=0,\\\vec{w}_1=0, \vec{A} =\vec{\Sigma}^{1/2}}} =&  \left. \frac{\partial  }{\partial \vec{b}} \left( \left. vec(\dot{\vec{W}}_2)  \right\rvert_{\vec{w}_1=0,   \vec{W}_2=0,\vec{A} =\vec{\Sigma}^{1/2}} \right)\right\rvert_{ \vec{b} = \vec{\upmu}}  \\
= & -\frac{\partial }{\partial \vec{b}} vec\left( \mathbb{E}_{ \vec{z} \sim \mathcal{N}(\vec{0}, \vec{I}_n) } [(\vec{\Sigma}^{1/2} \vec{z}+\vec{b})  (\vec{\Sigma}^{1/2} \vec{z}+\vec{b})^T f'(0)] \right)\\
 = &  - f'(0) \left. \frac{\partial }{\partial \vec{b}}  vec(\vec{b}\vec{b}^T) \right\vert_{\vec{b} = \upmu} \\
 =& -f'(0)(\vec{\upmu} \otimes \vec{I}_n + \vec{I}_n \otimes \vec{\upmu}) \\ 
\frac{\partial \dot{\vec{w}_1}}{\partial {vec(\vec{A}})} =& \left. \frac{\partial  }{\partial {vec(\vec{A})}} \left( \left. \dot{\vec{w}_1} \right\rvert_{\vec{w}_1=0,   \vec{W}_2=0, \vec{b} =\vec{\upmu}} \right)\right\rvert_{ \vec{A} = \vec{\Sigma}^{1/2}}  \\ 
&= -\frac{\partial }{\partial vec(\vec{A})}  \mathbb{E}_{ {\vec{z}} \sim \mathcal{N}(\vec{0}, \vec{I}_n) } [(\vec{A} \vec{z}+\vec{b})f'(0)]   =0\\
\frac{\partial \dot{\vec{w}_1}}{\partial {\vec{b}}} = &  \left. \frac{\partial  }{\partial {\vec{b}}} \left( \left. \dot{\vec{w}_1} \right\rvert_{\vec{w}_1=0,   \vec{W}_2=0, \vec{A} =\vec{\Sigma}^{1/2}} \right)\right\rvert_{ b=\vec{\upmu}}  = -\frac{\partial }{\partial \vec{b}}  \mathbb{E}_{ \vec{z} \sim \mathcal{N}(\vec{0}, \vec{I}_n) } [({\vec{\Sigma}}^{1/2} \vec{z}+\vec{b})f'(0)] \\ & = - \vec{I} f'(0)\\	
	\end{align*}
	
Recall that the Jacobian can then be written as:
\[
\begin{bmatrix}
\vec{J}_{DD} & \vec{J}_{DG} \\ -\vec{J}_{DG}^T & 0 
\end{bmatrix}
\]
where
\begin{align*}
\vec{J}_{DD} &= \begin{bmatrix}
\left.  \frac{\partial  vec(\dot{\vec{W}}_2)}{\partial vec(\vec{W}_2)} \right\vert_{\text{eqbm}}& 
\left.  \frac{\partial  vec(\dot{\vec{W}}_2)}{\partial \vec{w}_1} \right\vert_{\text{eqbm}}& \\
\left.  \frac{\partial   \dot{\vec{w}_1}}{\partial vec(\vec{W}_2)} \right\vert_{\text{eqbm}}& 
\left.  \frac{\partial  \dot{\vec{w}_1}}{\partial \vec{w}_1} \right\vert_{\text{eqbm}}& 
\end{bmatrix} = \\
&=
\begin{bmatrix}
  \mathbb{E}_{\vec{x} \sim \mathcal{N}(\vec{\upmu}, \vec{\Sigma})}[ (\vec{x} \otimes \vec{x}) (\vec{x} \otimes \vec{x})^T] & 
  \mathbb{E}_{\vec{x} \sim \mathcal{N}(\vec{\upmu}, \vec{\Sigma})} [ (\vec{x} \otimes \vec{x})\vec{x}^T ]  \\
  \transpose{  \mathbb{E}_{\vec{x} \sim \mathcal{N}(\vec{\upmu}, \vec{\Sigma})} [ (\vec{x} \otimes \vec{x})\vec{x}^T ] }&  
    \mathbb{E}_{\vec{x} \sim \mathcal{N}(\vec{\upmu}, \vec{\Sigma})}[ \vec{x} \vec{x}^T ] 
\end{bmatrix}2f''(0)
\end{align*}
and 
\begin{align*}
\vec{J}_{DG} &= \begin{bmatrix}
\left.  \frac{\partial  vec(\dot{\vec{W}}_2)}{\partial vec(\vec{A})} \right\vert_{\text{eqbm}}& 
\left.  \frac{\partial  vec(\dot{\vec{W}}_2)}{\partial \vec{b}} \right\vert_{\text{eqbm}}& \\
\left.  \frac{\partial   \dot{\vec{w}_1}}{\partial vec(\vec{A})} \right\vert_{\text{eqbm}}& 
\left.  \frac{\partial  \dot{\vec{w}_1}}{\partial \vec{b}} \right\vert_{\text{eqbm}}& 
\end{bmatrix} = \\
&=
-\begin{bmatrix}
  (\vec{I}_{n^2}+\vec{T}_{n,n})(\vec{\Sigma}^{1/2} \otimes \vec{I}_n)& 
  \vec{\upmu} \otimes \vec{I}_n + \vec{I}_n \otimes \vec{\upmu} \\
    0 &  
  \vec{I}_{n}
\end{bmatrix}f'(0)
\end{align*}

We can show that $\vec{J}_{DD}$ is negative definite because it is a moment matrix with a negative multiplicative factor. This is proved in Theorem~\ref{thm:moment-matrix}. Recall that as long as $f''(0) <0$, $f'(0) \neq 0$ and $\vec{J}_{DG}$ is full column rank (in this case full rank because $\vec{J}_{DG}$ is a square matrix), the matrix has eigenvalues whose real components are strictly negative.

To show that $\vec{J}_{DG}$ is full column rank, first  observe that the last few columns corresponding to $\vec{b}$ are linearly independent because, if $\vec{y}$ belongs to its null space, then 
\[
\begin{bmatrix}
 \vec{\upmu} \otimes \vec{I}_n + \vec{I}_n \otimes \vec{\upmu} \\ \vec{I} 
\end{bmatrix} \vec{y} = \begin{bmatrix}
( \vec{\upmu} \otimes \vec{I}_n + \vec{I}_n \otimes \vec{\upmu} ) \vec{y}\\ \vec{y}
\end{bmatrix}  = 0,
\]
which implies that $\vec{y} =0$. 

To verify whether the first few columns corresponding to $\vec{A}$ are linearly independent or not, consider any $\vec{V} \neq 0$. Then, we want to verify whether the following term is always non-zero or not:

\begin{align*}
 (\vec{I}_{n^2}+\vec{T}_{n,n})(\vec{\Sigma}^{1/2} \otimes \vec{I}_n) vec(\vec{V}) &=  (\vec{I}_{n^2}+\vec{T}_{n,n}) vec(\vec{I}_n \vec{V} (\vec{\Sigma}^{1/2})^T) \\
& = vec(\vec{V} (\vec{\Sigma}^{1/2})^T + \vec{\Sigma}^{1/2} \vec{V}^T) ,
\end{align*}
which is equivalent to testing whether $\vec{V} (\vec{\Sigma}^{1/2})^T + \vec{\Sigma}^{1/2} \vec{V}^T $  is non-zero.


Now, 
we will show that if $\vec{V} (\vec{\Sigma}^{1/2})^T + \vec{\Sigma}^{1/2} \vec{V}^T = 0$, then $\vec{V} = 0$. Recall that $\vec{\Sigma}^{1/2} = \vec{U} \vec{\Lambda}^{1/2} \vec{U}^T$. Then,

\begin{align*}
 \vec{V} (\vec{\Sigma}^{1/2})^T  & = - \vec{\Sigma}^{1/2} \vec{V}^T \\
\implies\vec{V} \vec{U} \vec{\Lambda}^{1/2} \vec{U}^T & = -\vec{U} \vec{\Lambda}^{1/2} \vec{U}^T \vec{V}^T \\
\vec{U}^T \vec{V} \vec{U} \vec{\Lambda}^{1/2} \vec{U}^T  \vec{V} \vec{U}& = - \vec{\Lambda}^{1/2} \vec{U}^T \vec{V}^T \vec{V} \vec{U} \\
\end{align*}

Observe that the left hand side is positive semi-definite while the right hand side is negative semi-definite. Therefore these terms must be equal to zero, which would then imply that $\vec{V}^T\vec{V} = 0$ i.e., $\vec{V} = 0$. Thus the Jacobian is indeed Hurwitz.

In summary, this means that  Assumption~\ref{as:convexity} holds trivially because there are no zero eigenvalues for the matrices involved in the Jacobian. This further means that there are no other equilibria in a small neighborhood around the considered equilibrium. Therefore, Assumption~\ref{as:global-gen} is also satisfied. Finally, since the support of the distribution is $\mathbb{R}^n$, Assumption~\ref{as:same-support} is also trivially satisfied. Thus, if Assumption~\ref{as:loss} holds, the system is exponentially stable.

\end{proof}

We now prove that $\vec{J}_{DD}$ is negative definite. 

\begin{theorem}
\label{thm:moment-matrix}
The matrix \[
\begin{bmatrix}
  \mathbb{E}_{\vec{x} \sim \mathcal{N}(\vec{\mu}, \Sigma)}[ (\vec{x} \otimes \vec{x}) (\vec{x} \otimes \vec{x})^T] & 
  \mathbb{E}_{\vec{x} \sim \mathcal{N}(\vec{\mu}, \Sigma)} [ (\vec{x} \otimes \vec{x})\vec{x}^T ]  \\
  \transpose{  \mathbb{E}_{\vec{x} \sim \mathcal{N}(\vec{\mu}, \Sigma)} [ (\vec{x} \otimes \vec{x})\vec{x}^T ]} &  
    \mathbb{E}_{\vec{x} \sim \mathcal{N}(\vec{\mu}, \Sigma)}[ \vec{x} \vec{x}^T ] 
\end{bmatrix}\] is positive definite. 
\end{theorem}

\begin{proof}

Let $\vec{U}$  be any arbitrary matrix and $\vec{v}$ be an arbitrary vector. Then,
\[
\begin{bmatrix}
vec(\vec{U}) \\ \vec{v}
\end{bmatrix}^T
\begin{bmatrix}
  \mathbb{E}_{\vec{x} \sim \mathcal{N}(\vec{\mu}, \Sigma)}[ (\vec{x} \otimes \vec{x}) (\vec{x} \otimes \vec{x})^T] & 
  \mathbb{E}_{\vec{x} \sim \mathcal{N}(\vec{\mu}, \Sigma)} [ (\vec{x} \otimes \vec{x})\vec{x}^T ]  \\
 \transpose{   \mathbb{E}_{\vec{x} \sim \mathcal{N}(\vec{\mu}, \Sigma)} [ (\vec{x} \otimes \vec{x})\vec{x}^T ] }&  
    \mathbb{E}_{\vec{x} \sim \mathcal{N}(\vec{\mu}, \Sigma)}[ \vec{x} \vec{x}^T ] 
\end{bmatrix}\begin{bmatrix}
vec(\vec{U}) \\ \vec{v}
\end{bmatrix}=
\]
\begin{align*} & =  \mathbb{E}_{\vec{x} \sim \mathcal{N}(\vec{\mu}, \Sigma)} \left[ \left\| \begin{bmatrix}
\vec{x} \otimes \vec{x} \\
\vec{x}
\end{bmatrix}^T \begin{bmatrix}
vec(\vec{U}) \\ \vec{v}
\end{bmatrix}\right\|^2 \right]\\
&= \mathbb{E}_{\vec{x} \sim \mathcal{N}(\vec{\mu}, \Sigma)}  \left[  \left(\vec{x}^T \vec{U} \vec{x} + \vec{x}^T\vec{v} \right)^2 \right]\\
\end{align*}

Now,  $\left(\vec{x}^T \vec{U} \vec{x} + \vec{x}^T\vec{v} \right)^2=0$ forms a quadric  $n-1$-dimensional hypersurface in $n$ dimensions, and therefore is of measure zero. For all other points, $\left(\vec{x}^T \vec{U} \vec{x} + \vec{x}^T\vec{v} \right)^2 > 0$ and therefore the above expectation is strictly positive. 
 
\end{proof}

\section{WGANs are not necessarily asymptotically stable}
\label{app:wgan-unstable}

 We consider a specific case
of the \LQ{} WGAN that learns a zero mean gaussian distribution, and show that
there exists points near certain equilibria such that if the system is
initialized to that point, it will periodically come back to that initial point
rather than converge to the equilibrium.
\begin{restatable}{theorem}{wganunstable}
\label{thm:wganunstable}
The \LQ{} WGAN system for learning a zero mean Gaussian distribution
$\mathcal{N}(\vec{0},\vec{\Sigma})$ ($\vec{\Sigma} \succ 0$) is not
asymptotically stable at the equilibrium corresponding to $\vec{A} =\vec{\Sigma}^{1/2}, \vec{b} = \vec{0}$ and $\vec{W}_2=0, \vec{w}_1 =0$.
\end{restatable}


\begin{proof} In order to show that the system is not asymptotically stable, we show that there are initializations of the system that are arbitrarily close to the equilibrium such that the system goes orbits around the equilibrium forever.
For simplicity, we first prove this for the one-dimensional gaussian  $\mathcal{N}(0,\sigma)$ and later extend it to the multi-dimensional case.  Let the quadratic discriminator be $D(x) = w_2^2 x + w_1x$ and the linear generator be $az + b$.  Then the WGAN objective in Equation~\ref{eq:generic_gan} for the \LQ{} system is:
\begin{align*}
V(G,D) &=\mathbb{E}_{x \sim \mathcal{N}(0,\sigma)} [w_2 x^2 + w_1 x] - \mathbb{E}_{ \vec{z} \sim \mathcal{N}(0,1) } [ w_2(az+b)^2 + w_1 (az+b)] \\
& = w_2(\sigma^2) - w_2(a^2 + b^2) - w_1 b
\end{align*}

The updates in Equation~\ref{eq:undamped_updates}  for \LQ{} simplify as follows:
\begin{align*}
\dot{w}_2 = & \sigma^2  -a^2 - b^2 \\
\dot{w}_1 = & - b\\
\dot{a} = &  2w_2a \\
\dot{b}= & 2w_2 b +w_1\\
\end{align*}

The system has two equilibria, $w_2  = 0, w_1 = 0, a=\pm \sigma, b = 0$. We will assume that the system is initialized with $w_1 = b=0$, which means that the system will forever have $w_1 =  b=0$ because the respective updates are zero too. Hence, we only need to focus on the variables  $w_2$ and $a$.

Now, it can be shown that if $a$ is initialized to $a_0 \geq 0$, $a$ never becomes negative (and similarly for $a \leq 0$). Therefore, we will focus on the equilibrium where $a = \sigma$, and assuming $a \geq 0$ examine how the distance from the equilibrium $w_2^2 + (a-\sigma)^2$ changes with time. The rate of change of this quantity is given by $2(w_2 \dot{w}_2 + (a-\sigma)\dot{a} ) = 2w_2(a-\sigma)^2$.
Observe that when $w_2 > 0$, this term is non-negative i.e., the system never gets closer to the equilibrium. Thus, when the system is in the ``bad'' half-space $w_2 > 0$, the only hope for it to converge is to exit this half-space so that $w_2$ becomes negative. However, we show that there exists initializations that are close to the equilibrium such that even if it does exit the bad half-space it eventually re-enters it, going in a perpetual loop.

More specifically, let $(w_2(t), a(t))$ denote the system at time $t$. Let the initialization satisfy $w_2(0) =0$ and $a(0) \in (0,\sigma)$.  We will now analyze the trajectory of this system. First note that $\dot{w}_2(0) > 0$, which means the system enters the bad half-space after immediately $t > 0$. Thus, if the system had to converge to the considered equilibrium, it would have to reach $w_2= 0$ again at some time $T$. First observe that at this time $a(T) > \sigma$ because we need $\dot{w}_2(T) < 0$ at this time. (In fact we can say that $a(T) - \sigma \geq \sigma - a(0)$ because we know that the radius never decreased until time $T$.) Now, we claim that the system simply retraces back its path along $a$ and reaches $a(0)$ at time $2T$. More clearly, we claim that the system at time $T+t$ can be described in terms of what it was at time $T-t$ as  $(w_2(T+t), a(T+t)) = (-w_2(T-t), a(T-t))$. 

To prove this observe that this statement is true for $t=0$ because $w_2(T) = 0$. Then  we only need to show that at any $t$, if $(w_2(T+t), a(T+t)) = (-w_2(T-t), a(T-t))$, then $\dot{w}_2(T+t) = \dot{w}_2(T-t)$ and $\dot{a}(T+t) = -\dot{a}(T-t)$. This is indeed true because $\dot{w}_2(T+t) = \sigma^2  -a^2(T+t) = \sigma^2  -a^2(T-t)=\dot{w}_2(T-t)$ and $\dot{a}_2(T+t) =2w_2(T+t)a(T+t) =  2(-w_2(T-t))a(T-t)=-\dot{a}(T-t)$. Therefore, applying $t=T$, we get  $(w_2(2T), a(2T)) = (-w_2(0), a(0)) = (0,a(0))$ i.e., the system has looped back to its original state by following its old path mirrored across the line $w_2=0$. Since this holds for initializations that are arbitrarily close to the equilibrium (i.e., $a(0)$ can be arbitrarily close to $\sigma$), the system is not asymptotically stable.

We extend this argument to the higher dimensional case as follows. Again, we initialize the system so that $\vec{w}_1= \vec{0}$ and $\vec{b} = \vec{0}$, then we can only focus on the updates on $\vec{W}_2$ and $\vec{A}$:
\begin{align*}
\dot{\vec{W}}_2 &= \vec{\Sigma} - \vec{A}\vec{A}^T \\
\dot{\vec{A}} & = (\vec{W}_2+\vec{W}_2^T) \vec{A}
\end{align*}

As before, we initialize $\vec{W}_2= 0$. We will also consider a more sophisticated initialization compared to $a \in (0, \sigma)$. Since $\vec{\Sigma}$ is positive definite, let $\vec{\Sigma} = \vec{U} \vec{\Lambda} \vec{U}^T$. We initialize $\vec{A} =\vec{U}\vec{\Lambda}_{A}(0) \vec{U}^T$ such that $\vec{\Lambda}_A(0)$ has at least one diagonal element that is positive but strictly less than the corresponding diagonal element in $\vec{\Lambda}^{1/2}$ (where $\vec{\Lambda}^{1/2} \succ 0$).

Now, we first establish that all the updates and the variables in the system remain in the eigenspace defined by $\vec{U}$. That is, at any point in time $t$, the variables can be expressed as $\vec{W}_2(t) =\vec{U} \vec{\Lambda}_W(t) \vec{U}^T $ and $\vec{A}(t) = \vec{U} \vec{\Lambda}_A(t) \vec{U}^T$ for some real diagonal matrices $\vec{\Lambda}_W(t)$ and $\vec{\Lambda}_A(t)$. Clearly, this is true for time $t=0$. Assuming this is true for arbitrary time $t$, observe that  the updates are 
\begin{align*}
\dot{\vec{W}}_2(t) =& \vec{U} (\vec{\Lambda} - \vec{\Lambda}_A^2(t)) \vec{U}^T \\
\dot{\vec{A}}(t) = & 2\vec{U} \vec{\Lambda}_W \vec{U}^T \vec{U} \vec{\Lambda}_A \vec{U}^T  = 2\vec{U} \vec{\Lambda}_W \vec{\Lambda}_A \vec{U}^T  
\end{align*}
Thus this is true for any time $t$. 
Therefore, we can analyze the system in terms of $\vec{\Lambda}_A, \vec{\Lambda}_W$ and the constant $\vec{\Lambda}$ as though there are $n$ independent 1-dimensional Gaussian systems. Then, the orbiting systems from the 1-dimensional updates must manifest here too. More specifically, these cycles would correspond to the diagonal in $\vec{\Lambda}_A$ which was initialized to be less than $\vec{\Lambda}^{1/2}$.
\end{proof}

\section{Gradient-based regularization}
\label{app:damped-updates}

In Section~\ref{app:gradreg-stability}, we prove how our gradient-based regularizer stabilizes the both the GAN and the WGAN system.
Besides this property, in Section~\ref{app:intuition} we provide an alternative mathematical intuition that is based on arg-max differentiation, to motivate our regularization term. Finally, in Section~\ref{app:unrolled}, we discuss how our regularizer addresses mode collapse and 1-unrolled GAN updates. 

 \subsection{Local stability of gradient-regularized GANs}
 \label{app:gradreg-stability}
We first restate our main result below. 
\regularized*

To prove this result, we first present the Jacobian of the system at equilibrium in the presence of the gradient penalty. Recall that the penalty basically adds an  extra $- \nabla_{\vec{\theta_G}} \| \nabla_{\vec{\theta}_D} V(D_{\vec{\theta_D}}, G_{\vec{\theta_G}}) \|^2$ to the generator's update.  
\begin{lemma}
\label{lem:damped-jacobian}

For the dynamical system defined by the GAN objective in
Equation~\ref{eq:generic_gan} and the updates in
Equation~\ref{eq:damped_updates}, the Jacobian at 
an equilibrium point  
$(\vec{\theta^\star_D},\vec{\theta^\star_G})$, under the Assumptions~\ref{as:global-gen} and ~\ref{as:same-support} is:
\begin{align*}
\vec{J} = 
\begin{bmatrix}
\vec{J}_{DD} & \vec{J}_{DG} \\
-\vec{J}_{DG}^T(\vec{I} +2 \eta \vec{J}_{DD}) &  - 2\eta \vec{J}_{DG}^T \vec{J}_{DG}
\end{bmatrix}
\end{align*}
where $\vec{J}_{DD}$ and $\vec{J}_{DG}$ are terms in the Jacobian corresponding to the original updates, as described in Theorem~\ref{thm:general-stability}.
\end{lemma} 

\begin{proof}
Note that the only change to the Jacobian would be in the rows corresponding to the generator parameters. Therefore, we will focus only on the additional terms in these rows. 

The additional term added to $-\vec{J}_{DG}^T$ is:
\begin{align*}
&-\left. \der{\vec{\theta_D}} {\eta \nabla_{\vec{\theta_G}} \| \nabla_{\vec{\theta}_D} V(D_{\vec{\theta_D}}, G_{\vec{\theta_G}}) \|^2} \right\vert_{\vec{\theta^\star_D}, \vec{\theta^\star_G}} =  - \eta
 \left. \transpose{\der{\vec{\theta_G}}{\nabla_{\vec{\theta_D}} \| \nabla_{\vec{\theta}_D} V(D_{\vec{\theta_D}}, G_{\vec{\theta_G}}) \|^2} \right\vert_{\vec{\theta^\star_D}, \vec{\theta^\star_G}}} \\
&=  - \eta \left. \transpose{\der{\vec{\theta_G}}{ \left( 2 \nabla^2_{\vec{\theta_D}} V(D_{\vec{\theta_D}}, G_{\vec{\theta_G}}) \nabla_{\vec{\theta_D}} V(D_{\vec{\theta_D}}, G_{\vec{\theta_G}})  \right)} }\right\vert_{\vec{\theta^\star_D}, \vec{\theta^\star_G}} \\
&=  -2 \eta \left.  \transpose{ \der{\vec{\theta_G}} {\nabla^2_{\vec{\theta_D}} V(D_{\vec{\theta_D}}, G_{\vec{\theta_G}})} \underbrace{\nabla_{\vec{\theta_D}} V(D_{\vec{\theta_D}}, G_{\vec{\theta_G}})}_{0 \text{ at eqbm}} +   
\der{\vec{\theta_D}}{\nabla_{\vec{\theta_G}}V(D_{\vec{\theta_D}}, G_{\vec{\theta_G}} )}
 \nabla^2_{\vec{\theta_D}} V(D_{\vec{\theta_D}}, G_{\vec{\theta_G}})  } \right\vert_{\vec{\theta^\star_D}, \vec{\theta^\star_G}} \\
= & -2\eta \vec{J}_{DG}^T \vec{J}_{DD}
\end{align*}

Now, the additional term added to $\vec{J}_{GG}$ is:

\begin{align*}
&-\left. \der{\vec{\theta_G}}{ \eta \nabla_{\vec{\theta_G}} \| \nabla_{\vec{\theta}_D} V(D_{\vec{\theta_D}}, G_{\vec{\theta_G}}) \|^2 }\right\vert_{\vec{\theta^\star_D}, \vec{\theta^\star_G}} \\
&=  - \eta \left. \flatder{\vec{\theta_G}}{  \left( 2 \der{\vec{\theta_D}}{\nabla_{\vec{\theta_G}} V(D_{\vec{\theta_D}}, G_{\vec{\theta_G}})} \nabla_{\vec{\theta_D}} V(D_{\vec{\theta_D}}, G_{\vec{\theta_G}})  \right)} \right\vert_{\vec{\theta^\star_D}, \vec{\theta^\star_G}} \\
&=  -2 \eta \left. \left( \der{\vec{\theta_D}}{\nabla^2_{\vec{\theta_G}} V(D_{\vec{\theta_D}}, G_{\vec{\theta_G}})} \underbrace{\nabla_{\vec{\theta_D}} V(D_{\vec{\theta_D}}, G_{\vec{\theta_G}})}_{0 \text{ at eqbm}} +  
\transpose{
\der{\vec{\theta_D}}{ \nabla_{\theta_G}V(D_{\vec{\theta_D}}, G_{\vec{\theta_G}} )}
}
\der{\vec{\theta_D}}{ \nabla_{\theta_G}V(D_{\vec{\theta_D}}, G_{\vec{\theta_G}} )}
 \right) \right\vert_{\vec{\theta^\star_D}, \vec{\theta^\star_G}} \\
= & -2\eta \vec{J}_{DG}^T \vec{J}_{DG}
\end{align*}
\end{proof}

Now, we will prove stability of the regularized system for conventional GANs. Observe that  Lemmas~\ref{lem:eqspace} regarding the subspace of equilibria holds in this case too. Again, we can project the system as follows:

\begin{lemma}
\label{lem:reg-projection}
For the dynamical system defined by the GAN objective in
Equation~\ref{eq:generic_gan} and the updates in
Equation~\ref{eq:damped_updates},
consider the eigenvalue decompositions $\vec{K}_{DD} = \vec{U_D} \vec{\Lambda_D} \vec{U_{D}}^T$ and $\vec{K}_{DG}^T \vec{K}_{DG} = \vec{U_G} \vec{\Lambda_G} \vec{U_{G}}^T$. Let $\vec{U_D} = [\vec{T}_D^T, \vec{T}_D'^T]$ and $\vec{U_G} = [\vec{T}_G^T, \vec{T}_G'^T]$ such that $\C(\vec{T}_{D}'^T) = \N(\vec{K}_{DD})$ and $\C(\vec{T}_{G}'^T) = \N(\vec{K}_{DG})$. Consider the projections, $\vec{\gamma_D} = \vec{T}_D \vec{\theta}_D$ and $\vec{\gamma_G} = \vec{T}_G \vec{\theta}_G$. Then, the block in the Jacobian at equilibrium that corresponds to the projected system has the form:
\[
\vec{J}' = \begin{bmatrix}
\vec{J}_{DD}' & \vec{J}_{DG}' \\
-\vec{J}_{DG}'^T & \vec{J}_{GG}' 
\end{bmatrix} = 
\begin{bmatrix}
\vec{T}_D \vec{J}_{DD} \vec{T}_{D}^T &   \vec{T}_D \vec{J}_{DG} \vec{T}_{G}^T \\
-\vec{T}_{G}\vec{J}_{DG}^T (\vec{I}+2\eta\vec{J}_{DD}) \vec{T}_D^T & -2\eta\vec{T}_{G}\vec{J}_{DG}^T\vec{J}_{DG} \vec{T}_{G}^T
\end{bmatrix}
\]
Under Assumption~\ref{as:loss}, we have that $\vec{J}_{DD}' \prec 0$ and $\vec{J}_{DG}'$ is full column rank and $\vec{J}_{GG}' \prec 0$.
\end{lemma}

It is straightforward to extend the proof of Lemma~\ref{lem:projection} to prove this lemma. 
Now, recall from Theorem~\ref{thm:multiple-equilibria} that if we show $\vec{J'}$ is Hurwitz the original system is exponentially stable. In the non-regularized system, we showed this by making use of the structure of the matrix. For this system, we will design a quadratic Lyapunov function that strictly decreases at non-equilibria points.

\begin{lemma}
\label{lem:reg-lyapunov}
For the dynamical system defined by the GAN objective in
Equation~\ref{eq:generic_gan} and the updates in
Equation~\ref{eq:damped_updates}, if $\eta < \frac{1}{2\lambda_{\max}(-\vec{J_{DD}})}$  the linearization of the system projected to a subspace orthogonal to the subspace of equilibria is exponentially stable with the Lyapunov function $\vec{x}^T \vec{P} \vec{x}$ where,
\[
\vec{P} = 
\begin{bmatrix}
\vec{T}_D (\vec{I} + 2\eta \vec{J}_{DD}) \vec{T}_D^T & 0 \\
0 & \vec{I}
\end{bmatrix}
\]
and $\vec{x}^T$ is $[\vec{\gamma_D}^T  \vec{\gamma_G}^T ] - [\vec{\gamma^\star_D}^T  \vec{\gamma^\star_G}^T ]$. The function
strictly decreases with time except at the equilibrium $[\vec{\gamma^\star_D}^T \vec{\gamma^\star_G}^T]^T$.
\end{lemma}

\begin{proof}

Note that when $\eta < \frac{1}{2\lambda_{\max}(-\vec{J_{DD}})}$,  $\vec{P} = \vec{P}^T \succ 0$ therefore the Lyapunov function is indeed positive definite. Furthermore, note that the rate of decrease is given by $\vec{x}^T \vec{Q} \vec{x}$ where $\vec{Q} = (\vec{J}'^T \vec{P} + \vec{P}\vec{J}') $. To show that this is strictly decreasing, we only need to show that $\vec{J}'^T \vec{P} + \vec{P} \vec{J}' \prec 0$.
First of all, note that $\vec{Q} =$

\begin{align*}
\begin{bmatrix}
\vec{T}_D \\
\vec{T}_G
\end{bmatrix}
\Bigg(\begin{bmatrix}
\vec{J}_{DD} (\vec{I} + 2\eta \vec{J}_{DD})&  -(\vec{I} +2 \eta \vec{J}_{DD}) \vec{J_{DG}}  \\
\vec{J}_{DG}^T \vec{T}_D^T \vec{T}_D (\vec{I} + 2\eta \vec{J}_{DD}) &  - 2\eta \vec{J}_{DG}^T \vec{J}_{DG}
\end{bmatrix} + \\
\begin{bmatrix}
 (\vec{I} + 2\eta \vec{J}_{DD}) \vec{J}_{DD}  &  (\vec{I} + 2\eta \vec{J}_{DD})\vec{T}_D^T \vec{T}_D \vec{J}_{DG} \\
-\vec{J}_{DG}^T(\vec{I} +2 \eta \vec{J}_{DD}) &  - 2\eta \vec{J}_{DG}^T \vec{J}_{DG}
\end{bmatrix}\Bigg)\begin{bmatrix}
\vec{T}_D \\
\vec{T}_G
\end{bmatrix}^T
\end{align*}
Here, the off-diagonal terms are $ \vec{T}_D (\vec{I} + 2\eta \vec{J}_{DD}) (\vec{I} - \vec{T}_D^T \vec{T}_D) \vec{J}_{DG} \vec{T}_{G}^T$ and its negative. This can be equated to zero because, 
\begin{align*}
\vec{T}_D (\vec{I} + 2\eta \vec{J}_{DD}) (\vec{I} - \vec{T}_D^T \vec{T}_D) & = \vec{T}_{D} - \overbrace{\vec{T}_{D}\vec{T}_D^T}^{\vec{I}} \vec{T}_D +   2\eta \vec{T}_{D} \vec{J}_{DD} - 2\eta\vec{T}_{D} \vec{J}_{DD} \vec{T}_D^T \vec{T}_D  \\
&= 2\eta  (\vec{T}_{D} \vec{J}_{DD} - \vec{T}_{D} \vec{J}_{DD} \vec{T}_D^T \vec{T}_D  ) \\
&= 2\eta  (\vec{T}_{D} \vec{T}_{D}^T \vec{\Lambda}_{D} \vec{T}_D - \vec{T}_{D} \vec{T}_{D}^T \vec{\Lambda}_{D} \vec{T}_D  \vec{T}_D^T \vec{T}_D  ) \\
& = 2\eta (\vec{\Lambda}_{D} \vec{T}_D  - \vec{\Lambda}_{D} \vec{T}_D  \vec{T}_D^T \vec{T}_D  ) = 0
\end{align*}

Then, the above matrix is equal to the diagonal matrix:
\[
\begin{bmatrix}
\vec{T_D}\left[\vec{J}_{DD} (\vec{I} + 2\eta \vec{J}_{DD}) + (\vec{I} + 2\eta \vec{J}_{DD}) \vec{J}_{DD} \right]\vec{T_D}^T & 0\\
0 &  - 4\eta \vec{T}_G\vec{J}_{DG}^T \vec{J}_{DG} \vec{T}_G^T \\
\end{bmatrix}  
\]

Note that by our choice of $\eta$, $(\vec{I} + 2\eta \vec{J}_{DD}) \succ 0$. 
Therefore, $\vec{J}_{DD}$ and $\vec{I} + 2\eta \vec{J}_{DD}$ share the same set of eigenvectors. Thus, the null space of $\vec{J}_{DD} $ and the term $\vec{J}_{DD} (\vec{I} + 2\eta \vec{J}_{DD}) + (\vec{I} + 2\eta \vec{J}_{DD}) \vec{J}_{DD}$ are the same, specifically orthogonal to $\vec{T}_D$. In other words, the top-left block above is a diagonal matrix with strictly negative eigenvalues.  Similarly, we also know that $-2\eta \vec{T_G}\vec{J}_{DG}^T \vec{J}_{DG} \vec{T_G}^T$ is a diagonal matrix with negative values. Hence, the above matrix is negative definite.
\end{proof}

\subsubsection{Exponential stability of gradient-regularized WGAN} 

We now proceed to the Wasserstain GAN scenario. First we lay down equivalent assumptions for the WGAN under which we can guarantee exponential stability in the regularized case. Note that even under these conditions, the unregularized update does not ensure asymptotic stability.

First, we note that due to the linearity of the loss function, it is not necessary that the discriminator be only identically zero on the support for the system to be at equilibrium --- it could also be constant on the support. 
Thus, we relax Assumption~\ref{as:global-gen} for this case to accommodate this.

\textbf{Assumption} \textbf{~\ref{as:global-gen}.} (\textbf{WGAN, Realizable}) 
$p_{\vec{\theta_G^\star}} = p_{\rm data}$ and 
$D_{\vec{\theta^\star_D}}(x) = c$, $\forall \;
x \in {\rm supp}(p_{\rm data})$ for some $c \in \mathbb{R}$.

Next, we state an assumption equivalent to Assumption~\ref{as:convexity}.  Recall that earlier we wanted $\mathbb{E}_{p_{\rm data}}[D^2_{\vec{\theta_D}}(x)]$ to satisfy Property~\ref{prop:convex} in the discriminator space.  Instead  of this function, we will now require that the magnitude of the generator updates satisfy Property~\ref{prop:convex} in the discriminator space. Note that the Hessian of this function at equilibrium is $\vec{K}_{DG}\vec{K}_{DG}^T$.

\textbf{Assumption}  \textbf{~\ref{as:convexity}.} (\textbf{WGAN}) At an equilibrium $(\vec{\theta^\star_D}, \vec{\theta^\star_G})$, the functions
$\left. \left\| \int_{\mathcal{X}} \nabla_{\vec{\theta_G}} p_{\vec{\theta_G}}(x) D_{\vec{\theta_D}}(x)  \right\|^2 \right\vert_{\vec{\theta_G}=\vec{\theta_G^\star}}$ and $\left. \left\| \mathbb{E}_{p_{\rm data}}[ \nabla_{\vec{\theta_D}} D_{\vec{\theta_D}} (x)  ]   -  \mathbb{E}_{p_{\vec{\theta_G}}}[ \nabla_{\vec{\theta_D}} D_{\vec{\theta_D}} (x)  ]   \right\|^2 \right\vert_{\vec{\theta_D} = \vec{\theta_D^\star}}$  
must satisfy Property~\ref{prop:convex} in the discriminator and generator space respectively.\\

Note that in effect, we get rid of the assumption on the other function and introduce a different function here; in either case, the original system is not asymptotically stable due to zero diagonal blocks in its Jacobian. Next, we retain Assumption~\ref{as:same-support} as it is. These are the only three assumptions we will need. 

We will now begin with a lemma similar to Lemma~\ref{lem:eqspace}

\begin{lemma}
\label{lem:wgan-eqspace}
For the dynamical system defined by the WGAN objective in
Equation~\ref{eq:generic_gan} and the updates in
Equation~\ref{eq:damped_updates},
under Assumptions~\ref{as:global-gen} and ~\ref{as:convexity} under the WGAN case, there exists $\epsilon_D, \epsilon_G > 0$ such that for all $\epsilon_D' \leq \epsilon_D$ and $\epsilon_G' \leq \epsilon_G$, and for any unit vectors  $\vec{u} \in \N(\vec{K}_{DG}^T), \vec{v} \in \N(\vec{K}_{DG})$,
 $(\vec{\theta_D^\star} + \epsilon_D' \vec{u}, \vec{\theta_G^\star} + \epsilon_G' \vec{v})$ is an equilibrium point.
 \end{lemma}

\begin{proof}

Note that $2\vec{K}_{DG}\vec{K}_{DG}^T$ is the Hessian of the function $\left\| \int_{\mathcal{X}} \nabla_{\vec{\theta_G}} p_{\vec{\theta_G}}(x) D_{\vec{\theta_D}}(x)  \right\|^2$ at equilibrium, namely the magnitude of the generator update. Then, by Assumption~\ref{as:convexity},  this function is locally constant  along any unit vector $\vec{u} \in \N(\vec{K}_{DG}^T)$. That is, 
 for sufficiently small $\epsilon$, if $\vec{\theta_D} = \vec{\theta^\star_D} + \epsilon  \vec{u}$, the function value is equal to the value at equilibrium which is zero, because by definition at equilibrium the generator update is zero. Now at $(\vec{\theta_D}, \vec{\theta_G^\star})$, the discriminator update is zero too since the generator matches the true distribution. Then by Assumption~\ref{as:global-gen}, it means that  $D_{\vec{\theta_D}}$ is identical over the true support. Then, it is an equilibrium discriminator such that the update for any generator would be zero.

 Similarly, as we saw, $2\vec{K}_{DG}^T \vec{K}_{DG}$ is the Hessian of the function $\left\| \mathbb{E}_{p_{\rm data}}[ \nabla_{\vec{\theta_D}} D_{\vec{\theta_D}} (x)  ]   -  \mathbb{E}_{p_{\vec{\theta_G}}}[ \nabla_{\vec{\theta_D}} D_{\vec{\theta_D}} (x)  ]   \right\|^2$ at equilibrium, namely the magnitude of the discriminator update. We also saw that 
 for sufficiently small $\epsilon'$, if $\vec{\theta_G} = \vec{\theta^\star_G} + \epsilon'  \vec{v}$, this function is zero. Thus,  at $(\vec{\theta_D^\star}, \vec{\theta_G})$, the discriminator update is zero. Furthermore, the generator update is zero too because the discriminator is constant throughout the support. Thus,  $(\vec{\theta_D^\star}, \vec{\theta_G})$ is an equilibrium point and from Assumption~\ref{as:global-gen} we can conclude that $p_{\vec{\theta_G}} = p_{\rm data}$. Thus, it is an equilibrium generator such that the update for any discriminator would be zero.

In summary, for all slight perturbations along $\vec{u} \in \N(\vec{K}_{DG}^T), \vec{v} \in \N(\vec{K}_{DG})$ we have established that the discriminator and generator individually satisfy the requirements of an equilibrium discriminator and generator pair, and therefore the system is itself is in equilibrium for these perturbations.
\end{proof}

Now, we show that this system can again be projected to a subspace orthogonal the equilibrium subspace such that the resulting Jacobian of the reduced system is Hurwitz. While earlier we chose $\vec{T}_{D}$ based on the matrix $\vec{K}_{DD}$ now we will choose it based on $\vec{K}_{DG}^T$.

\begin{lemma}
\label{lem:wgan-projection}
For the dynamical system defined by the GAN objective in
Equation~\ref{eq:generic_gan} and the updates in
Equation~\ref{eq:damped_updates},
consider the eigenvalue decompositions $\vec{K}_{DG}\vec{K}_{DG}^T = \vec{U_D} \vec{\Lambda_D} \vec{U_{D}}^T$ and $\vec{K}_{DG}^T \vec{K}_{DG} = \vec{U_G} \vec{\Lambda_G} \vec{U_{G}}^T$. Let $\vec{U_D} = [\vec{T}_D^T, \vec{T}_D'^T]$ and $\vec{U_G} = [\vec{T}_G^T, \vec{T}_G'^T]$ such that $\C(\vec{T}_{D}'^T) = \N(\vec{K}_{DD})$ and $\C(\vec{T}_{G}'^T) = \N(\vec{K}_{DG})$. Consider the projections, $\vec{\gamma_D} = \vec{T}_D \vec{\theta}_D$ and $\vec{\gamma_G} = \vec{T}_G \vec{\theta}_G$. Then, the block in the Jacobian at equilibrium that corresponds to the projected system has the form:
\[
\vec{J}' = \begin{bmatrix}
\vec{J}_{DD}' & \vec{J}_{DG}' \\
-\vec{J}_{DG}'^T & \vec{J}_{GG}' 
\end{bmatrix} = 
\begin{bmatrix}
0 &   \vec{T}_D \vec{J}_{DG} \vec{T}_{G}^T \\
-\vec{T}_{G}\vec{J}_{DG}^T \vec{T}_D^T & -2\eta\vec{T}_{G}\vec{J}_{DG}^T\vec{J}_{DG} \vec{T}_{G}^T
\end{bmatrix}
\]
Furthermore $\vec{J}_{GG}' \prec 0$ and $\vec{J}_{DG}'^T$ is full column rank.
\end{lemma}

\begin{proof}
Observe that the form of $\vec{J}'$ follows from Lemma~\ref{lem:damped-jacobian} by substituting $\vec{J}_{DD} = 0$. Furthermore, like we have seen before, observe that $\vec{T}_{G}\vec{J}_{DG}^T\vec{J}_{DG} \vec{T}_{G}^T$ is a diagonal matrix with positive eigenvalues and therefore $\vec{J'}_{GG} \prec 0$. Similarly, $\vec{J}_{DG}^T \vec{T}_D$ is a full column rank matrix because we have projected it to the subspace orthogonal to its null space. However, we need to show that $\vec{T}_G^T \vec{J}_{DG}^T \vec{T}_D$ which may have fewer rows, did not reduce in its rank. This is indeed true, since this is effectively a projection onto the subspace orthogonal to its left null space. 
\end{proof}

We now compile the above lemmas to prove our main result in Theorem~\ref{thm:regularized}.


\begin{proof}
The first part of the theorem statement for the conventional GAN follows from Lemma~\ref{lem:damped-jacobian}, ~\ref{lem:reg-projection}, ~\ref{lem:reg-lyapunov}.

To prove the second part it is sufficient to show that the projected Jacobian of the linearized system in Lemma~\ref{lem:wgan-projection} is Hurwitz, from which exponential stability of the original system follows from Theorem~\ref{thm:multiple-equilibria}. The fact that this is Hurwitz follows as usual from Lemma~\ref{lem:undamped-bound} after we flip the discriminator and generator variables:

\[
\begin{bmatrix}
\vec{J}_{GG}'  & -\vec{J}_{DG}'^T  \\
 \vec{J}_{DG}'  & 0
\end{bmatrix}.
\]
The Jacobian is thus Hurwitz because $\vec{J}_{GG}'$ is negative definite and $-\vec{J}_{DG}'^T$ is full column rank. Now, for the eigenvalue bounds we have from Lemma~\ref{lem:undamped-bound} that:
\begin{itemize}
\item If $\Im(\lambda) = 0$, then $\Re(\lambda) \leq  - \frac{2 f'^2(0) \eta \lambda_{\min}^{(+)}(\vec{K}_{DG} \vec{K}_{DG}^T) \lambda_{\min}^{(+)}(\vec{K}_{DG}^T \vec{K}_{DG})}{4  f'^2(0)\eta^2 \lambda_{\max}(\vec{K}_{DG} \vec{K}_{DG}^T)\lambda_{\min}^{(+)}(\vec{K}_{DG} \vec{K}_{DG}^T)  +  \lambda_{\min}^{(+)}(\vec{K}_{DG}^T \vec{K}_{DG})} $
\item If $\Im(\lambda) \neq 0$, then $\Re(\lambda) \leq -  \eta f'^2(0) {\lambda_{\min}^{(+)}(\vec{K}_{DG} \vec{K}_{DG}^T)} $
 \end{itemize}

However, this can be further simplified to arrive at the given bound by noting that all the non-zero eigenvalues of any matrix $\vec{A} \vec{B}$ is also equal to the non-zero eigenvalues of the matrix $\vec{B} \vec{A}$. Therefore, we can replace every occurrence of $\vec{K}_{DG} \vec{K}_{DG}^T$ with $\vec{K}_{DG}^T \vec{K}_{DG}$ in the above inequality.

\end{proof}

Additionally, we show that we can find a Lyapunov function that satisfies LaSalle's principle for the projected linearized system. 

\begin{fact}
For the linearized projected system with the Jacobian $\vec{J}'$, we have that 
$1/2 \| \vec{\gamma}_D - \vec{\gamma^\star}_D \|^2 + 1/2 \| \vec{\gamma}_G - \vec{\gamma^\star}_G \|^2$ is a Lyapunov function such that for all non-equilbrium points, it either always decreases or only instantaneously remains constant.
\end{fact}

\begin{proof}
Note that the Lyapunov function is zero only at the equilibrium of the projected system. Furthermore, it is straightforward to verify that the rate at which this changes is given by 
$- 2\eta (\vec{\gamma_G} - \vec{\gamma_G^\star})^T
\vec{T}_G \vec{K}_{DG}^T \vec{K}_{DG} \vec{T}_G^T (\vec{\gamma_G} - \vec{\gamma_G^\star})$ which is non-positive. Clearly this is zero only when $\vec{\gamma_G} = \vec{\gamma_G^\star}$ because $\vec{T}_G \vec{K}_{DG}^T \vec{K}_{DG} \vec{T}_G^T $ is positive definite. 
When this rate is indeed zero, we have that for the linearized system,  
  $\vec{\dot{\gamma}_G} = \vec{T}_G \vec{K}_{DG}^T \vec{T}_{D}^T ( \vec{\gamma}_D - \vec{\gamma^\star}_D )$ because the other term becomes zero. For the system to identically stay on the manifold $\vec{\gamma_G} = \vec{\gamma_G^\star}$ we need $\vec{\dot{\gamma}_G} = 0$, which happens only when $\vec{\gamma}_D = \vec{\gamma^\star}_D$ because $\vec{T}_G \vec{K}_{DG}^T \vec{T}_{D}^T$ is full column rank. When that is the case, we are at equilibrium.

\end{proof}


 \subsection{Intuition based on arg-max differentiation}
\label{app:intuition}

 In an ideal world, an optimizer would hope
to have access to  a function $\vec{\theta^\star_D}(\vec{\theta_G}) =
\arg\max_{\vec{\theta_D}} V (D_{\vec{\theta_D}}, G_{\vec{\theta_G}})$,
which is basically the optimal discriminator as a function of the generator;
given this, the optimizer should be able to update the generator with respect to
that. Then, the update can be shown to be the following (for clarity we use the
superscript $t$ and $t+1$ to denote the current and the updated parameters):  

\begin{align*}
    \vec{\theta}^{(t+1)}_G := &  \vec{\theta}^{(t)}_G - \alpha \underbrace{ \nabla_{\vec{\theta_G}}  V(D_{\vec{\theta^\star_D}(\vec{\theta}^{(t)}_G)}, G_{\vec{\theta_G}})}_{\text{conventional update}}  - \\ & \alpha   \left.  
   \left(\frac{\partial \vec{\theta^\star_D}(\vec{\theta_G})  }{\partial \vec{\theta_G}}     \right)^T
 \right\vert_{\vec{\theta_G} = \vec{\theta}^{(t)}_G}     
    \left. \nabla_{\vec{\theta_D}}  V(D_{\vec{\theta_D}}, G_{\vec{\theta}^{(t)}_G})  \right\vert_{\vec{\theta_D} =  \vec{\theta^\star_D}(\vec{\theta}^{(t)}_G) }\label{eq:intuition} \\ 
\end{align*}

Observe that the last term is zero because, for the optimal discriminator
$\nabla_{\vec{\theta_D}} V(D_{\vec{\theta_D}}, G_{\vec{\theta}^{(t)}_G}) =
0$. However, in practice, we would not be at the optimal discriminator and
therefore this term may be non-zero. Our hypothesis is that, instead of ignoring
this term like it is done for the conventional updates, retaining this term may
prove to be useful. To do so, we simply plug in the current discriminator for
$\vec{\theta}^\star_D(\vec{\theta}^\star_G)$ while computing this term, and furthermore, estimate the value of
$\nabla_{\vec{\theta_G}} \vec{\theta^\star_D}(\vec{\theta_G})$ using the
following equation: 

\begin{equation}
\begin{split}
0 = & \left. \nabla_{\vec{\theta_D}} V(\vec{\theta_D}, \vec{\theta_G^{(t)}}) \right\vert_{\vec{\theta_D} = \vec{\theta^\star_D}(\vec{\theta_G}^{(t)})} \implies \\ 
 0 =& \left. \flatder{\vec{\theta_G}} { \left( \left.
\nabla_{\vec{\theta_D}} V(D_{\vec{\theta_D}}, G_{\vec{\theta_G}}) \right\vert_{\vec{\theta_D} = \vec{\theta^\star_D}(\vec{\theta_G}^{(t)})} \right)}
\right\vert_{ \vec{\theta_G} = \vec{\theta_G}^{(t)}}  \\
= & \left. 
\der{\vec{\theta_G}}{\nabla_{\vec{\theta_D}} V(D_{\vec{\theta_D}}, G_{\vec{\theta_G}}) }
\right\vert_{\substack{\vec{\theta_D} = \vec{\theta^\star_D}(\vec{\theta_G}), \\ \vec{\theta_G} = \vec{\theta_G}^{(t)}} } +  \left. \nabla^2_{\vec{\theta_D}} V(D_{\vec{\theta_D}}, D_{\vec{\theta_G}^{(t)}}) \right\vert_{\vec{\theta_D} = \vec{\theta^\star_D}(\vec{\theta_G})} \left.
{\der{\vec{\theta_G}}  { \vec{\theta^\star_D}(\vec{\theta_G}) } }
\right\vert_{\vec{\theta_G} = \vec{\theta_G^{(t)}}}  
\end{split}
\end{equation}

In the second step above, we apply the chain rule. Rearranging, we get:
\begin{equation}
\begin{split}
\left.
{\der{\vec{\theta_G}}  { \vec{\theta^\star_D}(\vec{\theta_G}) } }
\right\vert_{\vec{\theta_G} = \vec{\theta_G^{(t)}}}  = - \left(  \left. \nabla^2_{\vec{\theta_D}} V(D_{\vec{\theta_D}}, D_{\vec{\theta_G}^{(t)}}) \right\vert_{\vec{\theta_D} = \vec{\theta^\star_D}(\vec{\theta_G})}   \right)^{-1} 
\left. 
\der{\vec{\theta_G}}{\nabla_{\vec{\theta_D}} V(D_{\vec{\theta_D}}, G_{\vec{\theta_G}}) }
\right\vert_{\substack{\vec{\theta_D} = \vec{\theta^\star_D}(\vec{\theta_G}), \\ \vec{\theta_G} = \vec{\theta_G}^{(t)}} }
\end{split}
\end{equation}

Since we hope the objective to be concave in the discriminator parameters, we
can approximate the Hessian as  $ \nabla^2_{\vec{\theta_D}}
V(D_{\vec{\theta_D}}, D_{\vec{\theta_G}^{(t)}}) = -\vec{I}/\eta$. Plugging
this into the update equation of $\vec{\theta}^{(t+1)}_G$ and also replacing  the optimal discriminator
with the current discriminator, 
we get the following  update rule which is equivalent to the original one presented in Equation~\ref{eq:damped_updates}:

\begin{equation} 
\label{eq:alternate_damped_updates}
    \vec{\theta}^{(t+1)}_G  :=  \vec{\theta}^{(t)}_G - \alpha \left. \nabla_{\vec{\theta_G}}  V(D_{\vec{\theta_D}}, G_{\vec{\theta_G}})\right\vert_{\vec{\theta_G} = \vec{\theta_G^{(t)}}}  -  \alpha    \eta  \left.
\transpose{\der{\vec{\theta_G}}{\nabla_{\vec{\theta_D}}    V(D_{\vec{\theta_D}}, G_{\vec{\theta_G}})   }    }
   \right\vert_{ \vec{\theta_G}=\vec{\theta_G}^{(t)}} 
    \nabla_{\vec{\theta_D}} V(D_{\vec{\theta}}, G_{\vec{\theta}^{(t)}_G})  \\
\end{equation}

\subsection{Mode Collapse and Relation to 1-unrolled updates}
\label{app:unrolled}

Our regularization term also has natural and intuitive
connections to an important issue that arises in GAN optimization called mode collapse.  Mode collapse is a situation where a GAN may enter an irrecoverable failure state 
where the generator incorrectly assigns all its probability mass to a
small region in space. This arises because a globally optimal strategy for the generator 
is to push all its mass towards the single point that the discriminator is the most confident
about being a real data point. To overcome this the generator needs more ``foresight'' -- it must
know that when it collapses all the mass, the discriminator will subsequently label
the collapsed point as fake data. Our penalty indeed encodes this foresight, because 
the  discriminator's ability to outdo the generator is quantified by the magnitude of the
discriminator's gradient. More clearly, our generator seeks a state where it can spread data out
enough, to make sure the discriminator has no obvious countermeasure (i.e., no
big gradients). 

In fact, we can show how our penalty term and 1-unrolled
GANs have very similar structure because intuitively both provide a one-step
lookahead to the generator. More precisely, we can arrive at 
1-unrolled updates if we simplify our updates further
and replace $\vec{\theta_D}$ by an ``unrolled''
$\vec{\theta_D} + \eta\nabla_{\vec{\theta_D}} \hat{V}(D_{\vec{\theta_D}},G_{\vec{\theta_G}})$.

 We begin by simplifying the 1-unrolled updates.
The key idea of a
1-unrolled update is to allow the generator to {\em explicitly} foresee how the
discriminator would react to its update, and optimize accordingly:
\begin{align*}
 & \vec{\theta_G}^{(t+1)}  := \vec{\theta_G}^{(t)} - \left.\alpha 
 \nabla_{\vec{\theta_G}}  
 V(
 D_{\vec{\theta_D} \underbrace{+ \eta\nabla_{\vec{\theta_D}} V(D_{\vec{\theta_D}},G_{\vec{\theta_G}})}_{\text{unrolling}}}, 
 G_{\vec{\theta_G}})
 \right\vert_{\vec{\theta_G} = \vec{\theta_G}^{(t)}}  \\
  & = \vec{\theta_G}^{(t)} - \left.\alpha \nabla_{\vec{\theta_G}}  V(D_{\vec{\theta_D} + \eta\nabla_{\vec{\theta_D}} V(D_{\vec{\theta_D}},G_{\vec{\theta_G}^{(t)}})}, G_{\vec{\theta_G}})\right\vert_{\vec{\theta_G} = \vec{\theta_G}^{(t)}}  \\
  &   - \left.\alpha \nabla_{\vec{\theta_G}}  V(D_{\vec{\theta_D} + \eta\nabla_{\vec{\theta_D}} V(D_{\vec{\theta_D}},G_{\vec{\theta_G}})}, G_{\vec{\theta_G}^{(t)}})\right\vert_{\vec{\theta_G} = \vec{\theta_G}^{(t)}}  \\
    & = \vec{\theta_G}^{(t)} - \left.\alpha \nabla_{\vec{\theta_G}}  V(D_{\vec{\theta_D} + \eta\nabla_{\vec{\theta_D}} V(D_{\vec{\theta_D}},G_{\vec{\theta_G}^{(t)}})}, G_{\vec{\theta_G}})\right\vert_{\vec{\theta_G} = \vec{\theta_G}^{(t)}}\\
  &   - \alpha \eta 
 \left.  
\transpose{\der{\vec{\theta_G}}{\nabla_{\vec{\theta_D}} V(D_{\vec{\theta_D}}, G_{\vec{\theta_G}}) }}  \right\vert_{\vec{\theta_G} = \vec{\theta_G}^{(t)}}  
\left.\nabla_{\vec{\theta_D'}} V(D_{\vec{\theta_D'}}, G_{\vec{\theta_G}^{(t)}}) \right\vert_{\vec{\theta_D'}=\vec{\theta_D} + \eta\nabla_{\vec{\theta_D}} V(D_{\vec{\theta_D}},G_{ \vec{\theta_G}^{(t)}})}\\
\end{align*}
In the first step, we compute gradient with respect to $\vec{\theta_G}$ as the sum of the gradients with respect to the two instances of $\vec{\theta_G}$ that occur in $V(
 D_{\vec{\theta_D} + \eta\nabla_{\vec{\theta_D}} V(D_{\vec{\theta_D}},G_{\vec{\theta_G}})}, 
 G_{\vec{\theta_G}})$, the first one that occurs as the second argument to $V(\cdot, \cdot)$, and the second one that occurs in the unrolled update of the first argument.
In the second step, we apply the chain rule on the second gradient.

We can compare our updates in Equation~\ref{eq:alternate_damped_updates} with the above to 
show how our updates are more flexible in terms of using the lookahead.
While both have two similar terms, a crucial difference is that in the latter,
every occurrence of the discriminator parameters (except one) has an additional unrolled update, namely
$\eta\nabla_{\vec{\theta_D}} V(D_{\vec{\theta_D}},G_{\vec{\theta_G}})$.
Clearly, this should provide more power to the latter; however in practice, we
observe that our technique can be more powerful than $1$-unrolled or even
$10$-unrolled updates (which are in fact much slower to run). The reason is that
the unrolled updates constrain $\eta$ to be small, typically of the order
$10^{-4}$ which is the step size. It would not be possible to increase $\eta$ to
greater magnitudes as it would be equivalent to a coarse step size in the
unrolling. Our method on the other hand, allows for larger $\eta$ because the
discriminator is retained as it is; in some sense, our penalty
provides a way of extracting and leveraging the unrolled update more flexibly.

\section{Eigenvalue bounds}
\label{app:bounds}

In this section, we prove one of the most useful lemmas that we used in our proofs, that matrices of the form $\begin{bmatrix} -\vec{Q} & \vec{P}; & -\vec{P}^T & 0\end{bmatrix}$
are Hurwitz when $Q \succ 0$ and $P$ is full column rank. We also prove eigenvalue bounds for such a matrix. To do so, we begin with a simple fact: 

\begin{lemma}
\label{lem:sq}
For $\vec{Q} \succeq 0$ be a real symmetric matrix.  If $\vec{a}^T \vec{Q} \vec{a} = c$, then $\vec{a}^T \vec{Q}^T \vec{Q} \vec{a} \in [\lambda_{\min}(\vec{Q}) c, \lambda_{\max}(\vec{Q}) c,]$.
\end{lemma}
\begin{proof}
Let $\vec{Q} = \vec{U}\vec{\Lambda}\vec{U}^T$ be the eigenvalue decomposition of $\vec{Q}$. Let $\vec{x} = \vec{U}\vec{a}$. Then, $c  = \vec{x} \vec{\Lambda} \vec{x}$ or in other words, $c = \sum x^2_i \lambda_i$. Similarly, $\vec{a}^T \vec{Q}^T \vec{Q} \vec{a}  = \sum x_i^2 \lambda^2_i$ which differs from $c$ by a multiplicative factor within $[\lambda_{\min}(\vec{Q}), \lambda_{\max}(\vec{Q})]$.
\end{proof}

We now prove our main result. 

\begin{lemma}
\label{lem:undamped-bound}
Let \[\vec{J} =\begin{bmatrix}  -\vec{Q}  & \vec{P} \\ -\vec{P}^T & 0 \end{bmatrix}, \] where $\vec{Q}  $ is a symmetric real positive definite matrix and $\vec{P}$ is a full column rank matrix.  Then, $\Re(\lambda) < 0$ for every eigenvalue $\lambda$ of $\vec{J}$. In fact, 
\begin{itemize}
\item When $\Im(\lambda) = 0$, 
\[
\Re(\lambda) \leq - \frac{\lambda_{\min}(\vec{Q} )\lambda_{\min}(\vec{P}^T\vec{P}) }{ \lambda_{\max} (\vec{Q} ) \lambda_{\min}(\vec{Q} )    +  \lambda_{\min}(\vec{P}^T\vec{P})}\\
\]
\item 
When $\Im(\lambda) \neq 0$,  \[\Re(\lambda) \leq - \frac{\lambda_{\min}(\vec{Q} )}{2}\]
\end{itemize}
\end{lemma}

\begin{proof}

We consider a generic eigenvector equation and equate the real and complex parts together so as to arrive at our bounds. 
Consider the following eigenvector equation:

\[
\begin{bmatrix}  -\vec{Q}  & \vec{P} \\ -\vec{P}^T & 0 \end{bmatrix} \begin{bmatrix}
\vec{a}_1 + i \vec{a}_2 \\\vec{b}_1 + i \vec{b}_2 
\end{bmatrix} = (\lambda_1 + i \lambda_2) \begin{bmatrix}
\vec{a}_1 + i \vec{a}_2 \\\vec{b}_1 + i \vec{b}_2 
\end{bmatrix},
\]

where $\vec{a}_i, \vec{b}_i, \lambda_i$ are all real-valued.   We assume that the vector is normalized i.e., $\vec{a}_1^2 + \vec{a}_2^2 + \vec{b}_1^2 + \vec{b}_2^2=1$. So, in case $\lambda_2 = 0$, we assume that $\vec{a}_1^2+\vec{b}_1^2=1$. We want to show that $\lambda_1 < 0$. Let us first rewrite the above equation as follows:
\[
\begin{bmatrix}  -\vec{Q} \vec{a}_1 + \vec{P}\vec{b}_1 + i (-\vec{Q} \vec{a}_2 + \vec{P}\vec{b}_2) \\ -\vec{P}^T\vec{a}_1 + i (-\vec{P}^T\vec{a}_2) \end{bmatrix}  =  \begin{bmatrix} \lambda_1 \vec{a}_1 - \lambda_2 \vec{a}_2 +i(\lambda_1 \vec{a}_2 + \lambda_2 \vec{a}_1) \\ 
\lambda_1 \vec{b}_1 - \lambda_2 \vec{b}_2 +i(\lambda_1 \vec{b}_2 + \lambda_2 \vec{b}_1) \\ 
\end{bmatrix}
\]

We can then equate the real and imaginary parts. 
\begin{align}
  -\vec{Q} \vec{a}_1 + \vec{P}\vec{b}_1  =& \lambda_1 \vec{a}_1 - \lambda_2 \vec{a}_2  \label{eq:1}\\
  -\vec{P}^T\vec{a}_1 = & \lambda_1 \vec{b}_1 - \lambda_2 \vec{b}_2 \label{eq:2} \\
  -\vec{Q} \vec{a}_2 + \vec{P}\vec{b}_2=& \lambda_1 \vec{a}_2 + \lambda_2 \vec{a}_1 \label{eq:3}\\
  -\vec{P}^T\vec{a}_2 =& \lambda_1 \vec{b}_2 + \lambda_2 \vec{b}_1\label{eq:4}
\end{align}

We now multiply the above equations by $\vec{a}_1^T, \vec{b}_1^T, \vec{a}_2^T, \vec{b}_2^T$ respectively and add them:
\[
\begin{array}{rcl}
\vec{a}_1^T(-\vec{Q} \vec{a}_1 + \vec{P}\vec{b}_1)  -\vec{b}_1^T\vec{P}^T\vec{a}_1& =&  \vec{a}_1^T( \lambda_1 \vec{a}_1 - \lambda_2 \vec{a}_2 ) + \vec{b}_1^T( \lambda_1 \vec{b}_1 - \lambda_2 \vec{b}_2) \\
+ \vec{a}_2^T(-\vec{Q} \vec{a}_2 + \vec{P}\vec{b}_2) -\vec{b}_2^T\vec{P}^T\vec{a}_2 & & + \vec{a}_2^T(\lambda_1 \vec{a}_2 + \lambda_2 \vec{a}_1) + \vec{b}_2^T( \lambda_1 \vec{b}_2 + \lambda_2 \vec{b}_1)
\end{array}
\]

As a result, only square terms and $\lambda_1$ terms remain:
\[
-\vec{a}_1^T\vec{Q} \vec{a}_1 -  \vec{a}_2^T\vec{Q} \vec{a}_2 = \lambda_1(\vec{a}_1^T \vec{a}_1 + \vec{a}_2^T\vec{a}_2 + \vec{b}_1^T \vec{b}_1 + \vec{b}_2^T \vec{b}_2) = \lambda_1
\]

\subparagraph{Proof for $\lambda_1< 0$.}
Now observe that $-\vec{a}_1^T\vec{Q} \vec{a}_1 -  \vec{a}_2^T\vec{Q} \vec{a}_2 \leq 0$ because $\vec{Q}  \succ 0$. If $-\vec{a}_1^T\vec{Q} \vec{a}_1 -  \vec{a}_2^T\vec{Q} \vec{a}_2 < 0$, it would immediately imply that $\lambda_1 < 0$. 

However this may not be true, and that would happen only when $\vec{a}_1=0$ and $\vec{a}_2 = 0$ because $\vec{Q}_1, \vec{Q}_2 \prec 0$. We will show that this case would not occur. First of all, this would force $\lambda_1=0$ to ensure the above equality.  By applying the Equations~\ref{eq:2} and ~\ref{eq:4}, we can conclude that $\lambda_2\vec{b}_2 = 0$ and $\lambda_2 \vec{b}_1 = 0$. Since one of $\vec{b}_1, \vec{b}_2 \neq 0$ this implies that $\lambda_2 = 0$ too. Now, by applying Equation~\ref{eq:1} and ~\ref{eq:3}, we have that $\vec{P}\vec{b}_1=0$ and $\vec{P}\vec{b}_2= 0$. Since one of $\vec{b}_1, \vec{b}_2 \neq 0$ (if they were both zero, our eigenvector would itself be zero), this implies that $\vec{P}$ is not a full column rank matrix, which is a contradiction of our assumption. Therefore, it cannot be the case that both $\vec{a}_1=0$ and $\vec{a}_2 = 0$.

\subparagraph{Stricter bound.}

Now, we prove our bounds on $\lambda_1$. (Note that an easy lower bound follows as $\lambda_1 \geq  - \lambda_{\max}(\vec{Q}) (\| \vec{a}_1\|^2 + \| \vec{a}_2\|^2) \geq - \lambda_{\max}(\vec{Q} ) $ but we are interested in an upper bound).  In order to prove the upper bound, we multiply  Equations \ref{eq:1} and Equations \ref{eq:3} by $-\vec{a}_2^T$ and $\vec{a}_1^T$ respectively and sum them up, and Equations \ref{eq:2} and Equations \ref{eq:4} by $-\vec{b}_2^T$ and $\vec{b}_1^T$ respectively and sum them up.

\begin{align*}
\vec{a}_2^T \vec{Q} \vec{a}_1 - \vec{a}_2^T \vec{P}\vec{b}_1 - \vec{a}_1^T\vec{Q} \vec{a}_2 + \vec{a}_1^T\vec{P}\vec{b}_2 =& -\vec{a}_2^T \lambda_1 \vec{a}_1 + \vec{a}_2^T \lambda_2 \vec{a}_2 + \vec{a}_1^T \lambda_1 \vec{a}_2 + \vec{a}_1^T \lambda_2 \vec{a}_1 \\
\implies - \vec{a}_2^T \vec{P}\vec{b}_1 + \vec{a}_1^T\vec{P}\vec{b}_2 =& \lambda_2(\|\vec{a}_2 \|^2 + \| \vec{a}_1^2\|)  \\
  \vec{b}_2^T \vec{P}^T\vec{a}_1  -\vec{b}_1^T\vec{P}^T\vec{a}_2= &-\vec{b}_2^T \lambda_1 \vec{b}_1  +\vec{b}_2^T\lambda_2 \vec{b}_2 +\vec{b}_1^T \lambda_1 \vec{b}_2 +\vec{b}_1^T \lambda_2 \vec{b}_1 \\
\implies  \vec{b}_2^T \vec{P}^T\vec{a}_1  -\vec{b}_1^T\vec{P}^T\vec{a}_2= &  \lambda_2(\|\vec{b}_2 \|^2 + \|\vec{b}_1 \|^2) \\
\end{align*}
As a consequence,
\[
 \lambda_2(\|\vec{a}_2 \|^2 + \| \vec{a}_1^2\|) = \lambda_2(\|\vec{b}_2 \|^2 + \|\vec{b}_1 \|^2)
\]

From the above we have that either $\lambda_2 = 0$ or $\|\vec{b}_2 \|^2 + \|\vec{b}_1\|^2  =  \|\vec{a}_2 \|^2 + \| \vec{a}_1^2\|=1/2$. Now, if $\lambda_2 \neq 0$, since  $-\vec{a}_1^T\vec{Q} \vec{a}_1 -  \vec{a}_2^T\vec{Q} \vec{a}_2 =\lambda_1$, we immediately get a bound $\lambda_1 \leq -\lambda_{\min}(\vec{Q} )/2$.

 In the former case, since the imaginary part of the eigenvalue is zero i.e., $\lambda_2=0$, the imaginary part of the eigenvector must be zero too i.e., $\vec{a}_2 = \vec{b}_2 =0$.  Then,  we have the equations:

\begin{align*}
  -\vec{Q} \vec{a}_1 + \vec{P}\vec{b}_1  =& \lambda_1 \vec{a}_1  \\
  -\vec{P}^T\vec{a}_1 = & \lambda_1 \vec{b}_1\\
\end{align*}

Rearranging and squaring the first equation we get:
\begin{align*}
\vec{b}_1^T \vec{P}^T \vec{P} \vec{b}_1 = & \vec{a}_1^T(\lambda_1 \vec{I} +\vec{Q} )^T(\lambda_1 \vec{I} +\vec{Q} )\vec{a}_1\\
=& \vec{a}_1^T(\lambda_1^2 I + 2\lambda_1 \vec{Q}  + \vec{Q} ^T \vec{Q} )\vec{a}_1 \\
\end{align*}

Then,
\begin{align*}
\implies \lambda_{\min}(\vec{P}^T\vec{P})\| \vec{b}_1^2\| \leq &   \lambda_1^2 \|\vec{a}_1 \|^2 - 2\lambda_1^2 + \vec{a}_1^T \vec{Q} ^T \vec{Q}  \vec{a}  \\
 \lambda_{\min}(\vec{P}^T\vec{P}) \leq &  ( \lambda_1^2 + \lambda_{\min}(\vec{P}^T\vec{P}) ) \|\vec{a}_1 \|^2 - 2\lambda_1^2 + \vec{a}_1^T \vec{Q} ^T \vec{Q}  \vec{a}  \\
\leq &  ( -\lambda_1^3  - \lambda_{\min}(\vec{P}^T\vec{P}) \lambda_1 )\frac{1}{\lambda_{\min}(\vec{Q} )} - 2\lambda_1^2 - \lambda_1 \lambda_{\max}(\vec{Q} ) \\
 \leq & -  \frac{\lambda_1}{\lambda_{\min}(\vec{Q} )}  \left( \lambda_1^2 + 2\lambda_{\min}(\vec{Q} )    \lambda_1 + \lambda_{\max} (\vec{Q} ) \lambda_{\min}(\vec{Q} )    +  \lambda_{\min}(\vec{P}^T\vec{P})\right) .
\end{align*}
In the first step,  we make use of the fact that  $\lambda_1 = -\vec{a}_1^T \vec{Q}  \vec{a}_1$. In the second step, we use $\|\vec{a}_1\|^2 + \|\vec{b}_1\|^2=1$. In the third step, we use Lemma~\ref{lem:sq} i.e.,  $\vec{a}_1^T \vec{Q} ^T \vec{Q}  \vec{a}_1 \leq -\lambda_1 \lambda_{\max}(\vec{Q} )$. We also use the fact that since $\lambda_1 = -\vec{a}_1^T \vec{Q}  \vec{a}_1$, $\| \vec{a}_1\|^2 \leq \frac{-\lambda_1}{\lambda_{\min}(\vec{Q})}$.

How do we upper bound $\lambda_1$ using this inequality?

Let us examine the quadratic in $\lambda_1$ in the above expression. Since the discriminant of this quadratic is $4\lambda_{\min}^2(\vec{Q}) - 4\lambda_{\max}(\vec{Q})\lambda_{\min}(\vec{Q}) - 4 \lambda_{\min}(\vec{P}^T \vec{P}) \leq - 4 \lambda_{\min}(\vec{P}^T \vec{P})  < 0$, the quadratic always takes the same sign, specifically positive. Next, note that the quadratic reaches its minimum at $ -\lambda_{\min}(\vec{Q} )$.

Now, $\lambda_1$ can either satisfy $\lambda_1 \leq -\lambda_{\min}(\vec{Q})$ or $0 \geq \lambda_1 > -\lambda_{\min}(\vec{Q})$. Since the former is already an upper bound, we will derive an upper bound in the latter case. Now, in the interval $(-\lambda_{\min}(Q), 0]$, the quadratic in $\lambda_1$ increases, and therefore for this interval, the above inequality can be rewritten by plugging in $\lambda_1 = 0$ inside the quadratic. On plugging it, we will get:


\begin{align*}
 \lambda_{\min}(\vec{P}^T\vec{P}) \leq & -  \frac{\lambda_1}{\lambda_{\min}(\vec{Q} )}  \left( \lambda_{\max} (\vec{Q} ) \lambda_{\min}(\vec{Q} )    +  \lambda_{\min}(\vec{P}^T\vec{P})\right)   \\
\lambda_1 \leq & -\lambda_{\min}(\vec{Q} ) \frac{\lambda_{\min}(\vec{P}^T\vec{P}) }{ \lambda_{\max} (\vec{Q} ) \lambda_{\min}(\vec{Q} )    +  \lambda_{\min}(\vec{P}^T\vec{P})}\\
\end{align*}

Observe  that the term on the right here lies in $(-\lambda_{\min}(\vec{Q}),0)$. This is because the fraction that is besides $-\lambda_{\min}(\vec{Q})$ in this term lies in $(0,1)$.  Thus, we will use this term as our bound on $\lambda_1$.
\end{proof}

 Now, we provide a similar upper bound result, though only partially, for eigenvalues of matrices that have the same structural properties as the Jacobian of our regularized system. Note that we have upper bounds only for eigenvalues that are complex (we have not used them anywhere in the main paper though).

\begin{theorem}
\label{thm:damped-bound}
Let \[\vec{J} =\begin{bmatrix}  -\vec{Q} & \vec{P} \\ -\vec{P}^T (\vec{I}-\eta \vec{Q}) & -2\eta \vec{P}^T \vec{P} \end{bmatrix}, \] where $\vec{Q} $ is a real symmetric positive definite matrix and $\vec{P}$ is a full column rank matrix.  Let $\eta < \frac{1}{\lambda_{\max}(\vec{Q})}$.

Then, if $\Im(\lambda) \neq 0$ for any eigenvalue $\lambda$ of $\vec{J}$, 

\[ 
\Re(\lambda) \leq 
-\frac{1}{2}  \frac{1-\eta \lambda_{\max}(\vec{Q})}{1-\eta \lambda_{\min}(\vec{Q})} \left(\lambda_{\min}(\vec{Q}) + \eta \lambda_{\min}(\vec{P}^T\vec{P}) \right)\]
\end{theorem}

\begin{proof}

Consider the following eigenvector equation:

\[
\begin{bmatrix}  -\vec{Q} & \vec{P} \\ -\vec{P}^T (\vec{I}-\eta \vec{Q})& -2\eta  \vec{P}^T \vec{P}  \end{bmatrix} \begin{bmatrix}
\vec{a}_1 + i \vec{a}_2 \\\vec{b}_1 + i \vec{b}_2 
\end{bmatrix} = (\lambda_1 + i \lambda_2) \begin{bmatrix}
\vec{a}_1 + i \vec{a}_2 \\\vec{b}_1 + i \vec{b}_2 
\end{bmatrix},
\]

where $u_i, v_i, \lambda_i$ are all real-valued. 
We want to show that $\lambda_1 < 0$. Let us first rewrite the above equation as follows:
\begin{align*}
\begin{bmatrix}  
-\vec{Q}\vec{a}_1 + \vec{P}\vec{b}_1 + i (-\vec{Q}\vec{a}_2 + \vec{P}\vec{b}_2) \\ 
 -\vec{P}^T (\vec{I}-\eta \vec{Q}) \vec{a}_1 - \eta \vec{P}^T\vec{P}\vec{b}_1 + i ( -\vec{P}^T(\vec{I}-\eta \vec{Q}) \vec{a}_2- 2\eta \vec{P}^T\vec{P}\vec{b}_2) 
 \end{bmatrix}  =  \\
\begin{bmatrix} \lambda_1 \vec{a}_1 - \lambda_2 \vec{a}_2 +i(\lambda_1 \vec{a}_2 + \lambda_2 \vec{a}_1) \\ 
\lambda_1 \vec{b}_1 - \lambda_2 \vec{b}_2 +i(\lambda_1 \vec{b}_2 + \lambda_2 \vec{b}_1) \\ 
\end{bmatrix}
\end{align*}

We can then equate the real and imaginary parts. 
\begin{align}
  -\vec{Q}\vec{a}_1 + \vec{P}\vec{b}_1  =& \lambda_1 \vec{a}_1 - \lambda_2 \vec{a}_2 \\
 -\vec{P}^T(\vec{I}-\eta \vec{Q}) \vec{a}_1 -2 \eta \vec{P}^T\vec{P}\vec{b}_1  = & \lambda_1 \vec{b}_1 - \lambda_2 \vec{b}_2\\
  -\vec{Q}\vec{a}_2 + \vec{P}\vec{b}_2=& \lambda_1 \vec{a}_2 + \lambda_2 \vec{a}_1 \\
-\vec{P}^T (\vec{I}-\eta \vec{Q}) \vec{a}_2- 2\eta \vec{P}^T\vec{P}\vec{b}_2=& \lambda_1 \vec{b}_2 + \lambda_2 \vec{b}_1\\
\end{align}

We now multiply the above equations by $\vec{a}_1^T, \vec{b}_1^T, \vec{a}_2^T, \vec{b}_2^T$ respectively and add them:
\[
\begin{array}{rcl}
\vec{a}_1^T(-\vec{Q}\vec{a}_1 + \vec{P}\vec{b}_1)  -\vec{b}_1^T\vec{P}^T\vec{a}_1 -2\eta \vec{b}_1\vec{P}^T\vec{P}\vec{b}_1& =&  \vec{a}_1^T( \lambda_1 \vec{a}_1 - \lambda_2 \vec{a}_2 ) + \vec{b}_1^T( \lambda_1 \vec{b}_1 - \lambda_2 \vec{b}_2) \\
+ \vec{a}_2^T(-\vec{Q}\vec{a}_2 + \vec{P}\vec{b}_2) -\vec{b}_2^T\vec{P}^T\vec{a}_2 -2\eta\vec{b}_2\vec{P}^T\vec{P}\vec{b}_2 & & + \vec{a}_2^T(\lambda_1 \vec{a}_2 + \lambda_2 \vec{a}_1) + \vec{b}_2^T( \lambda_1 \vec{b}_2 + \lambda_2 \vec{b}_1) \\
+\eta \vec{b}_1^T \vec{P}^T \vec{Q} \vec{a}_1 +\eta \vec{b}_2^T \vec{P}^T \vec{Q} \vec{a}_2 &
\end{array}
\]

As a result,  we get:
\[
-\vec{a}_1^T\vec{Q}\vec{a}_1 - \vec{a}_2^T\vec{Q}\vec{a}_2 -2\eta \vec{b}_1\vec{P}^T\vec{P}\vec{b}_1-2\eta \vec{b}_2\vec{P}^T\vec{P}\vec{b}_2   +\eta \vec{b}_1^T \vec{P}^T \vec{Q} \vec{a}_1 +\eta \vec{b}_2^T \vec{P}^T\vec{Q} \vec{a}_2 = \lambda_1 
\]

Above, we can substitute for $\vec{P}\vec{b}_1$ and $\vec{P}\vec{b}_2$ in $ \eta \vec{b}_1^T \vec{P}^T \vec{Q} \vec{a}_1 +\eta \vec{b}_2^T \vec{P}^T\vec{Q} \vec{a}_2$  using the previous equations.

\begin{align*}
-\vec{a}_1^T\vec{Q}\vec{a}_1 - \vec{a}_2^T\vec{Q}\vec{a}_2 -2\eta \vec{b}_1\vec{P}^T\vec{P}\vec{b}_1 - 2\eta \vec{b}_2\vec{P}^T\vec{P}\vec{b}_2  & \\
  +\eta ( \lambda_1 \vec{a}_1^T\vec{Q} \vec{a}_1  -\lambda_2 \vec{a}_1^T \vec{Q} \vec{a}_2 + \vec{a}_1^T \vec{Q}^T\vec{Q}\vec{a}_1) & \\
  +\eta (\lambda_1 \vec{a}_2^T\vec{Q} \vec{a}_2 + \lambda_2 \vec{a}_2^T \vec{Q} \vec{a}_1  +\vec{a}_2^T \vec{Q}^T\vec{Q}\vec{a}_2) & = \lambda_1\\
\frac{-\vec{a}_1^T\vec{Q}\vec{a}_1 - \vec{a}_2^T\vec{Q}\vec{a}_2 -\eta \vec{b}_1\vec{P}^T\vec{P}\vec{b}_1 - \eta \vec{b}_2\vec{P}^T\vec{P}\vec{b}_2 + \eta  \vec{a}_1^T \vec{Q}^T\vec{Q}\vec{a}_1 +\eta \vec{a}_2^T \vec{Q}^T\vec{Q}\vec{a}_2 }{1-\eta \left( \vec{a}_1^T\vec{Q} \vec{a}_1+ \vec{a}_2^T\vec{Q} \vec{a}_2  \right)}  & = \lambda_1 \\
\end{align*}
 
 We could do the above only because  $\eta < \frac{1}{\lambda_{\max}(\vec{Q})}$ and therefore the denominator $1-\eta ( \vec{a}_1^T\vec{Q} \vec{a}_1+ \vec{a}_2^T\vec{Q} \vec{a}_2 ) \neq 0 $.

 In order to prove our upper bound, we first note the following inequality:
 \begin{align}
 \label{eq:generic-inequality}
 |\lambda_1| \geq \frac{ (1-\eta \lambda_{\max}(\vec{Q})) \lambda_{\min}(\vec{Q})  (\| \vec{a}_1\|^2 + \| \vec{a}_2\|^2 )     + \eta \lambda_{\min}(\vec{P}^T\vec{P}) (\|\vec{b}_1 \| + \|\vec{b}_2 \|^2)
 }{1-\eta \lambda_{\min}(\vec{Q}) (\|\vec{a}_1 \|^2 + \| \vec{a}_2 \|^2)}
 \end{align}

 We now multiply  the first and third equations by $-\vec{a}_2^T$ and $\vec{a}_1^T$ respectively and sum them up, and second and fourth by $-\vec{b}_2^T$ and $\vec{b}_1^T$ respectively and sum them up. Then, we get:

\begin{align*}
- \vec{a}_2^T \vec{P}\vec{b}_1 + \vec{a}_1^T\vec{P}\vec{b}_2 =& \lambda_2(\|\vec{a}_2 \|^2 + \| \vec{a}_1^2\|)  \\
  \vec{b}_2^T \vec{P}^T(\vec{I}-\eta \vec{Q})\vec{a}_1  -\vec{b}_1^T\vec{P}^T(\vec{I}-\eta \vec{Q})\vec{a}_2= &  \lambda_2(\|\vec{b}_2 \|^2 + \|\vec{b}_1 \|^2) \\
\end{align*}

Using the above,
\begin{align*}
 \lambda_2(\|\vec{b}_2 \|^2 + \|\vec{b}_1 \|^2) -\lambda_2(\|\vec{a}_2 \|^2 + \| \vec{a}_1\|^2) &= -\eta \vec{b}_2^T \vec{P}^T \vec{Q}\vec{a}_1  +\eta \vec{b}_1^T\vec{P}^T\vec{Q} \vec{a}_2 \\
 &= -\eta \vec{a}_1^T\vec{Q}^T(\lambda_1 \vec{a}_2 + \lambda_2 \vec{a}_1+ \vec{Q}\vec{a}_2 )  + \\
& \eta \vec{a}_2^T\vec{Q}^T ( \lambda_1 \vec{a}_1 - \lambda_2 \vec{a}_2+\vec{Q}\vec{a}_1  ) \\
 &=  -\eta \lambda_2 \vec{a}_1^T\vec{Q}^T \vec{a}_1 - \eta \lambda_2 \vec{a}_2^T\vec{Q}^T \vec{a}_2
\end{align*}

Then, either $\lambda_2=0$ or when $\lambda_2 \neq 0$, we have $\|\vec{b}_2 \|^2 + \|\vec{b}_1 \|^2 = \|\vec{a}_2 \|^2 + \| \vec{a}_1\|^2  -\eta \lambda_2 \vec{a}_1^T\vec{Q}^T \vec{a}_1 - \eta \lambda_2 \vec{a}_2^T\vec{Q}^T \vec{a}_2$. This translates to the inequality:

\[
(1-\eta \lambda_{\max}(\vec{Q}))(\|\vec{a}_2 \|^2 + \| \vec{a}_1\|^2) \leq \|\vec{b}_2 \|^2 + \|\vec{b}_1 \|^2 \leq (1-\eta \lambda_{\min}(\vec{Q}))(\|\vec{a}_2 \|^2 + \| \vec{a}_1^2\|),
\]

By adding $\|\vec{a}_2 \|^2 + \| \vec{a}_1\|^2$ everywhere and using the fact that $\|\vec{a}_2 \|^2 + \| \vec{a}_1\|^2 + \|\vec{b}_2 \|^2 + \|\vec{b}_1 \|^2 = 1$, the above inequality becomes:

\begin{align*}
\frac{1}{2-\eta \lambda_{\min}(\vec{Q})} & \leq \|\vec{a}_2 \|^2 + \| \vec{a}_1\|^2 \leq  \frac{1}{2-\eta \lambda_{\max}(\vec{Q})}\\
\frac{1-\eta \lambda_{\max}(\vec{Q})}{2-\eta \lambda_{\max}(\vec{Q})} & \leq  \|\vec{b}_2 \|^2 + \|\vec{b}_1 \|^2\leq\frac{1-\eta \lambda_{\min}(\vec{Q})}{2-\eta \lambda_{\min}(\vec{Q})}  \\
\end{align*}

The above inequalities yield an immediate lower bound on the magnitude in the latter case by plugging them in Equation~\ref{eq:generic-inequality}:

\begin{align*}
|\lambda_1| &\geq  \frac{ (1-\eta \lambda_{\max}(\vec{Q})) \lambda_{\min}(\vec{Q}) 
 }{\frac{1}{ (\| \vec{a}_1\|^2 + \| \vec{a}_2\|^2 )   }-\eta \lambda_{\min}(\vec{Q})}  +
 \frac{      \eta \lambda_{\min}(\vec{P}^T\vec{P}) (\|\vec{b}_1 \| + \|\vec{b}_2 \|^2)
 }{1-\eta \lambda_{\min}(\vec{Q}) (\|\vec{a}_1 \|^2 + \| \vec{a}_2 \|^2)}  \\
 & \geq  \frac{ (1-\eta \lambda_{\max}(\vec{Q})) \lambda_{\min}(\vec{Q}) 
 }{ 2(1-\eta \lambda_{\min}(\vec{Q}))}  +
 \frac{      \eta \lambda_{\min}(\vec{P}^T\vec{P}) (1-\eta \lambda_{\max}(\vec{Q}))
 }{2(1-\eta \lambda_{\min}(\vec{Q})) }  \\ 
 & \geq  \frac{1}{2}  \frac{1-\eta \lambda_{\max}(\vec{Q})}{1-\eta \lambda_{\min}(\vec{Q})} \left(\lambda_{\min}(\vec{Q}) + \eta \lambda_{\min}(\vec{P}^T\vec{P}) \right)
\end{align*}
Since, we know $\lambda_1$ is negative, this implies an upper bound on $\lambda_1$.

\end{proof}

\end{document}